\documentclass[aos]{imsart}
\RequirePackage{amsthm,amsmath,amssymb}
\RequirePackage{natbib}
\RequirePackage[colorlinks,citecolor=blue,urlcolor=blue]{hyperref}

\usepackage{color,soul}
\usepackage{accents,float,graphicx,subfig,lscape}
\usepackage[letterpaper]{geometry}

\pubyear{2006}
\volume{0}
\issue{0}
\firstpage{1}
\lastpage{8}

\startlocaldefs
\numberwithin{equation}{section}

\theoremstyle{plain}

\newtheorem{definition}{\protect\definitionname}
\providecommand{\definitionname}{Definition}
\numberwithin{definition}{section}

\newtheorem{theorem}{\protect\theoremname}
\providecommand{\theoremname}{Theorem}
\numberwithin{theorem}{section}

\newtheorem{lemma}{\protect\lemmaname}
\providecommand{\lemmaname}{Lemma}
\numberwithin{lemma}{section}

\newcommand\thickbar[1]{\accentset{\rule{.4em}{.8pt}}{#1}}
\endlocaldefs

\begin{document}

\begin{frontmatter}
\title{Rademacher upper bounds for cross-validation errors with an application to the lasso\thanksref{T1}}
\runtitle{Upper bounds for cross-validation errors}
\thankstext{T1}{We would like to thank Pierre Del Moral and Peter Hall for valuable comments on the draft. We also would like to acknowledge seminar participants at UNSW and University of Melbourne for useful questions, comments and suggestions.}

\begin{aug}
\author{\fnms{Ning\snm{Xu,}\thanksref{t1}\ead[label=e1]{ning.xu@sydney.edu.au}}}
\author{\fnms{Jian\snm{Hong}\ead[label=e2]{jian.hong@sydney.edu.au}}}
\and
\author{\fnms{Timothy C.G.}\snm{Fisher}\thanksref{t3}\ead[label=e3]{tim.fisher@sydney.edu.au}}

\address{School of Economics, University of Sydney\\ NSW 2006 Australia\\
\printead{e1,e2,e3}}

\thankstext{t1}{Xu would like to thank Google Australia and NICTA for hardware and programming support in package optimization and development.}
\thankstext{t3}{Fisher would like to acknowledge the financial support of the Australian Research Council grant DP0663477.}
\runauthor{N. Xu et al.}

\affiliation{University of Sydney}
\end{aug}

\begin{abstract} 
We establish a general upper bound for $K$-fold cross-validation ($K$-CV) errors that can be adapted to many $K$-CV-based estimators and learning algorithms. Based on Rademacher complexity of the model and the Orlicz-$\Psi_{\nu}$ norm of the error process, the CV error upper bound applies to both light-tail and heavy-tail error distributions. We also extend the CV error upper bound to $\beta$-mixing data using the technique of independent blocking. We provide a Python package (\texttt{CVbound}, \url{https://github.com/isaac2math}) for computing the CV error upper bound in $K$-CV-based algorithms. Using the lasso as an example, we demonstrate in simulations that the upper bounds are tight and stable across different parameter settings and random seeds. As well as accurately bounding the CV errors for the lasso, the minimizer of the new upper bounds can be used as a criterion for variable selection. Compared with the CV-error minimizer, simulations show that tuning the lasso penalty parameter according to the minimizer of the upper bound yields a more sparse and more stable model that retains all of the relevant variables.
\end{abstract}



\end{frontmatter}

\section{Introduction}

The out-of-sample performance of a model is important for model selection and is often quantified using $K$-fold cross-validation ($K$-CV) \citep{stone74,stone77}. The cross-validation error (CV error), also referred to as the average prediction error, indicates empirically the out-of-sample performance of the model. Algorithms like the lasso are often computed using $K$-CV and minimizing the CV error is frequently used as a criterion for model selection. Thus, intuitively, the reliability of the CV error is the very essence of model selection and evaluation. In particular, the absence of a reliable CV error may result in models that fail to deliver valid interpretation or prediction. In this paper, we offer a general framework to analyze the reliability of the CV error. The framework allows us to derive an upper bound for the CV error that quantifies and improves the reliability of $K$-CV.

\subsection{$K$-fold cross-validation}

Before elaborating the issues, we fix the terminology around the $K$-CV procedure following \citet{friedman2001elements}. $K$-CV evaluates the out-of-sample performance of a model (typically defined to be a function class or a model class) by repeatedly splitting the original sample into a \emph{training set} to train (or estimate) the model, and a \emph{validation set} to evaluate it. The original sample is randomly split into $K$ equal-size folds. In each round of the training-validation split, a single fold is retained as the validation set, returning the \emph{prediction error}, while the remaining $K-1$ folds are used as the training set, returning the \emph{training error}. The procedure is repeated $K$ times, with each of the $K$ folds used as the validation set exactly once, to produce the \emph{average training error} and the \emph{CV error}, the average of the $K$ prediction errors. Ideally, the original sample would be used exclusively for training and validation would be carried out using samples (of the same size) selected repeatedly from the population (called \emph{test sets}). However, scarcity of data typically forces the researcher to split the sample into training and validation sets, with the latter used as `pseudo' test sets. Thus, the CV error is the analog of the \emph{average test error}, the average prediction error of the model on different test sets.

\subsection{Relevant literature}

Researchers in statistics, machine learning and biostatistics have uncovered several problems with $K$-CV. One problem is the high variance of the CV error of a given model. \citet{efron1983estimating} and \citet{breiman1996heuristics} show that, while $K$-CV gives a nearly unbiased estimate of the CV error, it often comes at the cost of an unacceptably high variance, leading to unreliable estimates of the model. In finite sample analysis, \citet{cawley2010over} argue that while unbiasedness \emph{per se} is relatively unimportant for the purposes of model selection, high variance, which affects the reliability of model selection, remains a fundamental concern. Given the model, \citet{cawley2010over} verify the high variance problem with $K$-CV in large-scale simulations. \citet{friedman2001elements} explain that autocorrelation is a major cause of the unreliability of the CV error. In $K$-CV, training sets in any two rounds contain a proportion $(K-2)/(K-1)$ of common data points. Hence, given a model and sample size, training errors will be autocorrelated across rounds with autocorrelation being more severe the larger is $K$. In this case, the estimated models will be similar across training sets. As a result, given the sample size, the CV error will exhibit a higher variance for larger $K$, referred to as the high variance problem \citep[242-244]{friedman2001elements}. Due to the high variance problem, the point estimate of the CV error may not be very reliable. Further studies \citep{kohavi1995study,kearns1999algorithmic,blum1999beating,varma2006bias} confirm the high variance problem with $K$-CV for a range of learning algorithms, such as linear regression, $K$-nearest-neighbors, support vector machines, classification trees, regression trees and shrunken centroids.

An important consequence of the high variance problem is that $K$-CV may not be relied upon for model selection. $K$-CV is widely used for model and variable selection: the CV error for each model (or combination of variables) is computed and the model with minimal CV error, referred to as the \emph{CV-error minimizer}, is selected. However, owing to the high variance in the CV error for each model, the order of the CV errors for different models will vary across samples and the CV-error minimizer will be unstable across samples. \citet{nan2014variable,lim2016estimation} refer to this problem as a `stability issue of model selection' and point out that it becomes more severe in high-dimensional spaces. \citet{cawley2010over} confirm that the model selection stability problem is a fundamental concern for $K$-CV. \citet{lim2016estimation} shows that variable selection algorithms like the lasso may lead to models that are unstable in high dimensions and consequently ill-suited for interpretation and generalization.

Thus, to improve the stability and accuracy of model selection, it is essential to quantify the variability of the CV error of the model. To do that, the most direct way is to estimate the variance of the CV error of the given model. However, computing the CV error variance represents a theoretical and empirical challenge, as detailed by \citet{kohavi1995study}, \citet{dietterich1998approximate}, and \citet{nadeau2000inference}. \citet{dietterich1998approximate} shows that accurately evaluating a model based on CV errors, requires the assumption that the performance of the model changes smoothly with the size of the training set. While smoothness can be checked experimentally, it is hard to check empirically and in any case it is often violated in practice \citep{haussler1996rigorous}. \citet{kohavi1995study} shows that variance estimation of the CV error is biased for any $K$ and, using large-scale simulations, that the standard deviation of the CV error is $\sqrt{50}$ times larger than the population standard deviation of the average test error. \citet{nadeau2000inference} criticize traditional methods more generally on the grounds that research on CV errors does not typically take into account autocorrelation arising from the choice of training set, that methods to compute the variance of the CV error are biased or unstable, and that the expected value for the variance estimate of the CV error exceeds the actual variance. Lastly, \citet{bengio2004no} proves that, given a model in $K$-CV, there is no universal (valid under all distributions) unbiased estimator for the variance of the CV error, even if the variance exists.

Given the difficulties estimating the CV error variance directly, an alternative is to construct an upper bound to quantify the variation of the CV error along the lines of \citet{vc68} or \citet{bartlett2002rademacher}. However, existing bounds for the CV error \citep{devroye2013probabilistic} are specific to locally-defined classifications (such as nearest-neighbor), which are either heuristic or derived under relatively restrictive assumptions. Another way to construct an upper bound for the CV error is to estimate the percentile of the corresponding error distribution via the bootstrap. However, bootstrapping in $K$-CV is computationally expensive, especially for high-dimensional, sparse modelling.

\subsection{Contribution}

In this paper, we use concentration inequalities to establish a general upper bound for the CV error that can be adapted to many $K$-CV-based estimators and learning algorithms. Based on Rademacher complexity of the model and the Orlicz-$\Psi_v$ norm of the error process, the CV error upper bound applies to both light-tail and heavy-tail error distributions. We also extend the CV error upper bound to $\beta$-mixing data using the technique of independent blocking.

We implement the theoretical results in simulations using a Python package (\texttt{CVbound}) that computes the CV error upper bound for a $K$-CV-based algorithm. Using the lasso as an example, we demonstrate that the behavior of the upper bounds in simulations is tight and stable across different parameter settings and random seeds.

As well as effectively bounding the CV errors and average test errors for the lasso, the new upper bounds may also be used as a criterion for model or variable selection. Compared with the CV-error minimizer, our simulations show that tuning the lasso penalty parameter according to the minimizer of the upper bound yields a more sparse and stable model that retains all of the relevant variables.

The paper is organized as follows. In section~2, we list our definitions and assumptions. In section~3, we derive the upper bounds of the CV errors under traditional i.i.d.\ settings. In section~4, using the independent blocking technique, we construct upper bounds for the CV error under a $\beta$-mixing scenario. In section~5, using the lasso as an application, we construct upper bounds for the CV errors, illustrating the shape and tightness of the new bound. The simulations also reveal that the upper-bound minimizer improves the accuracy and sparsity of variable selection in the lasso. Appendix~A contains the relevant proofs and Appendix~B contains additional plots of the simulations.


\section{Definitions and assumptions}

\begin{definition}
[Subsamples, errors and the Orlicz-$\Psi_\nu$ norm]\label{def:notation}
\end{definition}
  \begin{enumerate}
    \item   Let $(y_i, \mathbf{x}_i)$ denote a sample point from $F(y, \mathbf{x})$, the joint distribution of $(y,\mathbf{x})$. Given a sample $(Y,X)$ of size $n$, the \textit{training set} is $(Y_{t},\,X_{t})\in\mathbb{R}^{n_t \times \left( 1 + p \right) }$ and the \textit{validation set} is $(Y_{s},\,X_{s})\in\mathbb{R}^{n_s \times \left( 1 + p \right) }$ where $n_t$ is the size of the training set and $n_s$ is the size of the validation set. For $K$-CV, $n_s = n - n_t = n/K$.
    \item   Let there be $L$ models in $K$-CV, each of which is defined as a different model class $\Lambda_l$, $ 1 \leqslant l \leqslant L $. For $b \in \Lambda_l$, the \textit{loss function} is $Q \left( b, y_i,\mathbf{x}_i \right),\,i=1,\ldots,n$, the \textit{population error} is
        \[ \mathcal{R} \left( b, Y, X \right) = \int Q \left( b, y,\mathbf{x} \right) \mathrm{d} F \left(y ,\mathbf{x} \right)\]
        and the \textit{empirical error} is
        \[\mathcal{R}_{n} \left( b, Y, X \right) = \frac{1}{n}\;\sum_{i=1}^n \; Q \left( b, y_i, \mathbf{x}_i \right).\]
    \item  For $b \in \Lambda_l$, the \textit{training error} is $\mathcal{R}_{n_t} \left( b, Y_{t}, X_{t} \right)$ and the \textit{prediction error} is $\mathcal{R}_{n_s} \left( b, Y_{s}, X_{s} \right)$.
    \item For round $q$ in $K$-CV, the training set is $\left( Y_t^q, X_t^q \right)$ and the validation set is $ \left( Y_s^q, X_s^q \right)$. Thus, for round $q$, the training error on the training set is $\mathcal{R}_{n_t}\left(b, Y_t^q, X_t^q \right)$ and the prediction error on the validation set is $\mathcal{R}_{n_s}\left( b, Y_s^q, X_s^q \right)$.
    \item Define the Orlicz-$\Psi_\nu$ norm for an empirical process $Z$ to be
      \[
        \left\Vert Z \right\Vert_{\Psi_\nu} :=
        \inf_{u > 0} \left\{ u \; \left\vert \; \mathbb{E}
        \left[ \exp \left\{ \frac{\left\vert Z \right\vert^\nu } { u^\nu } \right\} \right]
        < 2 \right\}. \right.
      \]
      Also define the empirical process, $\forall b \in \Lambda_l, \; \forall q \in \left[1,K\right]$,
      \begin{align}
        U_q & := \sup_{b \in \Lambda_l} \; \left\vert
          \mathcal{R}_{n_s} \left( b, Y_{s}^q, X_{s}^q \right) - \mathcal{R}_{n_t} \left( b, Y_{t}^q, X_{t}^q \right)
          \right\vert \notag \\
        T_q & := U_q - \mathbb{E} \left[ U_q \right]. \notag
      \end{align}
  \end{enumerate}

A key step in our analysis is to use the Orlicz-$\Psi_\nu$ norm as opposed to the Lebesgue-$p$ ($L^p$) norm. The Orlicz-$\Psi_\nu$ norm encodes the heaviness of the tail for a random variable. Moreover, the Orlicz-$\Psi_\nu$ norm offers a method to derive exponential concentration inequalities for random variables for which the $L^2$ norm does not exist.

\subsection{Complexity measures for a model class}
The traditional method for bounding out-of-sample performance yields the complexity measures of a model class, developed by \citet{vc68}, \citet{bartlett2002rademacher} and others. We show that on average the CV error of a model may be bounded by the sum of a complexity measure of a given model class, the training error and a measure of sample variation. The key step to establishing the bound for the CV error is to use the Rademacher measure of model complexity as defined by \citet{bartlett2002rademacher}.
%
%

\begin{definition}[Rademacher complexity]
\label{defn:rad_complexity}
  Let $ \left\{ \mathbf{x}_1, \ldots, \mathbf{x}_n \right\}$ denote $n$ observations sampled from the distribution of $X$. Let $\left\{ \omega_1, \ldots, \omega_n \right\}$ denote $n$ independent observations sampled from the Rademacher distribution, which are also independent from $\left\{ \mathbf{x}_1, \ldots, \mathbf{x}_n \right\}$. Let $\Lambda_l$ be the model class and, for $b \in \Lambda_l$, let $b\left( \mathbf{x}_i \right)$ denote the predicted value for $\omega_i$ based on $\mathbf{x}_i$. The \emph{empirical} Rademacher complexity is defined to be
  \begin{equation}
    \widehat{ \mathrm{ RC } } \left( \Lambda_l \right)
      =
      \mathbb{E}_{\omega} \left[ \sup_{b \in \Lambda_l} \left\vert \frac{2}{n} \sum_{i=1}^{n} \omega_i b \left( \mathbf{x}_i \right) \right\vert \; \vert \; \mathbf{x}_1, \ldots, \mathbf{x}_n \right]
  \end{equation}
  The \emph{Rademacher complexity} of $\Lambda_l$ is
  \begin{equation}
    \mathrm{RC}_{n} \left( \Lambda_l \right) =
      \mathbb{E}_{\mathbf{x}} \left[ \widehat{ \mathrm{RC}}_{n} \left( \Lambda_l \right) \right].
  \end{equation}
\end{definition}

To compute the Rademacher complexity of $\Lambda_l$, we use $ \left\{ \mathbf{x}_1, \ldots, \mathbf{x}_n \right\}$ and the (independent) $\left\{ \omega_1, \ldots, \omega_n \right\}$. First, we choose $b$ from $\Lambda_l$ to maximize the correlation between the actual value of $\omega_i$ and its predicted value $b \left( \mathbf{x}_i \right)$. To remove any sample variation in $\left\{ \mathbf{x}_i , \omega_i \right\}_{i=1}^{n}$, we integrate out $\omega_i$ and the $\mathbf{x}_i$. Rademacher complexity measures the maximum extent the model in class $\Lambda_l$ overfits the data: given the model $\Lambda_l$ and data $\left\{ \mathbf{x}_i , \omega_i\right\}_{i=1}^{n}$, the larger Rademacher complexity, the more likely the model overfits the training data. Rademacher complexity can also be used to measure by how much (at most) the sample error deviates from the population error. Empirically implementing Rademacher complexity requires solving several computational problems, which we return to in the simulations section below.

To construct an upper bound for the out-of-sample performance of $K$-CV, we must also define one-round Rademacher complexity.

\begin{definition}[One-round Rademacher complexity for $K$-CV]
\label{defn:one-round_RC}
  Given a sample size $n$ and the number of folds $K$, we define the \emph{one-round Rademacher complexity} of the model class $\Lambda_l$, $\mathrm{RC} \left( \Lambda_l, n, K \right)$ to be
  \[
    \mathrm{RC} \left( \Lambda_l, n, K \right) =
        \tfrac{1}{2} \mathrm{RC}_{n_s} \left( \Lambda_l \right)
      + \tfrac{1}{2} \mathrm{RC}_{n_t} \left( \Lambda_l \right)
  \]
\end{definition}
In this paper, we use Rademacher complexity, the average training error of the model and a measure of sample variation to derive a bound for the CV error. The main assumptions for our analysis are as follows.

\subsection{Main assumptions}

\begin{enumerate}
    \item[\textbf{A1.}]   Let $\left( \Omega, \mathcal{F}, P \right)$ be a probability space, where $\Omega$ is the sample space for $X$ and $Y$, $\mathcal{F}$ is its sigma algebra and $P$ is a probaility measure on $\mathcal{F}$. We assume the loss function $Q$: $\Omega \rightarrow \mathbb{R}^+$ is $\mathcal{F}$-measurable, $\forall b \in \Lambda_l$. We assume the population error $\mathcal{R} \left( b, Y, X \right)$ is well defined for any $b \in \Lambda_l$. Specifically, we assume that the Orlicz-$\Psi_\nu$ norms of all the loss processes are well-defined for all $\nu \geqslant 1$.
    \item[\textbf{A2.}]   The data $ \left( Y, X \right)$ are independently sampled from the same  population. All points in each fold of $K$-CV are randomly partitioned.

    \item[\textbf{A3.}]   \citep{vapnik1998statistical} Each of the $L$ models in $K$-CV belongs to a different model class $\Lambda_l$, $1 \leqslant l \leqslant L$. A linear order exists within the set $ \left\{ \mathrm{RC}_{n/K} \left( \Lambda_l \right) \; \vert \; l = 1, \ldots, L \right\}$.
    \item[\textbf{A4.}]   \citep{vapnik1998statistical} For all the empirical processes in this paper, the VC entropy in model class $\Lambda_l$, $H^{\Lambda_l} \left( \epsilon, n \right)$, satisfies $\lim_{n \rightarrow \infty} H^{\Lambda_l} \left( \epsilon, n \right) / n = 0,\; \forall \epsilon > 0$.
    \item[\textbf{A5.}] Given the model class $\Lambda_l$,
      \begin{align}
          \frac{ \mathbb{E} \left[ \sup_{b \in \Lambda_l} \left( \mathcal{R}_n \left( b, Y, X \right) - \mathcal{R} \left( b, Y, X \right) \right)^2 \right] }
          { \sup_{b \in \Lambda_l} \left\{ \mathbb{E} \left[ \left( \mathcal{R}_n \left( b, Y, X \right) - \mathcal{R} \left( b, Y, X \right) \right)^2 \right] \right\} }
          = \theta \in \mathbb{R}^+.
          \label{a5:theta}
      \end{align}
\end{enumerate}

Several remarks apply to the assumptions. \textbf{A1} defines the class of loss distributions for our analysis. For $\nu \geqslant 2$, if the Orlicz-$\Psi_\nu$ norm is well-defined, the loss distribution is in the subgaussian family and large values of the loss do not occur very often. For $\nu \geqslant 1$, if the Orlicz-$\Psi_\nu$ norm is well-defined, the loss distribution is in the subexponential family and the tails are heavier. If $\nu \in \left(0 , 1\right)$, the Orlicz-$\Psi_\nu$ norm is actually a pseudo norm and some of our results cannot be generalized. \textbf{A2}~applies only for section~3; we modify the assumption in section~4 to allow for $\beta$-mixing in the data generating process. \textbf{A3} originates in the work of \citet{vapnik1998statistical} and rules out cases where the Rademacher complexities of different models cannot be pairwise compared. As shown in \citet{vapnik1998statistical}, \textbf{A4}~ensures the empirical error of any $b \in \Lambda_l$ converges to the population error as $n \rightarrow \infty$. Lastly, \textbf{A5}~is known to be a general condition for the partial interchangeability of $\sup \left[ \cdot \right]$ and $\mathbb{E} \left[ \cdot \right]$, which is useful in the proofs below.

\section{Upper bounds for CV errors with i.i.d.\ data}

Before deriving the upper bound for the CV errors of any $b \in \Lambda_l$, we need to construct an upper bound for the round-$q$ prediction error of $b$, $\forall q$. Here we show that the prediction error in each round for $b \in \Lambda_l$ is bounded by eq.~(\ref{eq:VC_bound}).
%
%
\begin{theorem}[Upper bounds for the round-$q$ prediction error]
  Under \textbf{A1}-\textbf{A5}, in round $q$ of $K$-CV, the following upper bound for the prediction error holds with probability at least $ 1 - \varpi \in \left( 0 , 1 \right]$, $\forall b \in \Lambda_l$.
  \begin{equation}
    \mathcal{R}_{n_s} \left( b, Y_s^q, X_s^q \right) \leqslant \mathcal{R}_{n_t}
      \left( b, Y_t^q, X_t^q \right) + 2 \cdot \mathrm{RC}_{n/K} \left( \Lambda_l \right) + \varsigma,
    \label{eq:VC_bound}
  \end{equation}
  \noindent
  where
  \[
    \varsigma =
      \left\{
        \begin{array}{ll}
          2 \cdot M \cdot \sqrt{ \theta \cdot \frac{ \log \left( 1 / \varpi \right) } { n / K } },
          & \mbox{ if } \sup_{ b \in \Lambda_l } \left( Q \right) \leqslant M \\[8pt]
          2 \cdot B \cdot \sqrt{ \theta \cdot \frac{ \log \left( 1 / \varpi \right) } { n / K } }
          & \mbox{ if } \rho_j \mbox{ is subgaussian } \\[8pt]
          \left\Vert \rho_j \right\Vert_{\Psi_1} \cdot \log \sqrt[c] {\left( 2 / \varpi \right) } ,
          & \mbox{ if } \rho_j \mbox{ is subexponential and } ~ 2 e^{-2c} > \varpi  >  0  \\
          \left\Vert \rho_j \right\Vert_{\Psi_1} \cdot \left( 2 \cdot \log \sqrt[c] { \left( 2 / \varpi \right) } \right)^{\frac{1}{2}} ,
          & \mbox{ if } \rho_j \mbox{ is subexponential and } ~ 2 e^{-2c} \leqslant \varpi \leqslant 1
        \end{array}\right.
  \]
    and, for the i\textsuperscript{th} point in the j\textsuperscript{th} fold, $\left( y^{j,i}, \mathrm{x}^{j,i} \right)$,
  \[
  \rho_j = \sup_{b \in \Lambda_l} \left\vert \sum_{i = 1}^{ n/K } Q \left( b, y^{j,i}, \mathrm{x}^{j,i} \right) / ( n/K ) - \mathcal{R} \left( b, Y, X \right) \right\vert,
  \]
  $B^2$ is $\sup_{b \in \Lambda_l} \left\{ \mathrm{var} \left[ Q \left( b, y^{j,i}, \mathrm{x}^{j,i} \right) \right] \right\}$ if $\rho_j$ is subgaussian, and $c$ is an absolute constant in the exponential inequality \citep{Lecue09tool}.
\label{thm:one_round_VC_bound}
\end{theorem}
%

Eq.~(\ref{eq:VC_bound}) quantifies the variability of the prediction error for the model trained on the round-$q$ training set.  Since eq.~(\ref{eq:VC_bound}) holds with probability at least $1 - \varpi$, the right-hand-side (RHS) may be interpreted as the upper bound for the $(1 - \varpi) \times 100$\textsuperscript{th} percentile of the round-$q$ prediction error distribution. Since the prediction error is always non-negative, the lower bound for the distribution of the prediction error is $0$. Thus, given the model, the interval between $0$ and the RHS of eq.~(\ref{eq:VC_bound}) forms a confidence interval for the empirical prediction error with sample size $n/K$.

The first two terms on the RHS of eq.~(\ref{eq:VC_bound}) are the training error $\mathcal{R}_{n_t}$ and the measure of model complexity $\mathrm{RC}_{n/K} \left( \Lambda_l \right)$, respectively. Given $n$ and $K$, a more complex model is more likely to overfit the data, resulting in a training error that is smaller than the prediction error, on average. On the other hand, a more complex model results in a higher value for $\mathrm{RC}_{n/K} \left( \Lambda_l \right)$, ceteris paribus. Thus, the bound in eq.~(\ref{eq:VC_bound}) captures the trade-off between model complexity and overfitting and determines, ceteris paribus, how the prediction error of a model changes across samples.

The third term on the RHS of eq.~(\ref{eq:VC_bound}), $\varsigma$, is related to the size of each fold and the tail heaviness of the error distribution. The larger the size of each fold, the smaller $\varsigma$. Further, if $\rho_j$ is subgaussian, the larger the variance of $\rho_j$, the larger $\varsigma$ and the higher the upper bound, ceteris paribus. If $\rho_j$ is subexponential, the heavier the tail of the error distribution, the larger the Orlicz-$\Psi_1$ norm of $\rho_j$, the larger $\varsigma$ and the higher the upper bound.

$\varpi$ also affects the location of the upper bound. $1 - \varpi$ is the probability that the bound holds. Since the RHS of eq.~(\ref{eq:VC_bound}) bounds the $(1 - \varpi) \times 100$\textsuperscript{th} percentile of the round-$q$ prediction error distribution, intuitively a change in $\varpi$ should change the location of the bound. For example, if $\varpi$ increases from $0.10$ to $0.15$, the $(1 - \varpi) \times 100$\textsuperscript{th} percentile changes from the $90$\textsuperscript{th} percentile to the $85$\textsuperscript{th} percentile, which shifts the bound downwards. Such intuition is reflected in eq.~(\ref{eq:VC_bound}): ceteris paribus, a larger $\varpi$ reduces the magnitude of the last term in eq.~(\ref{eq:VC_bound}), shifting the bound downwards. In the simulations, we set $1 - \varpi = 0.9$ (see the discussion in section~5).

The next step is to convolute the prediction error and training error in each round and establish the upper bound for the CV error. If the empirical processes $\left\{ \mathcal{R}_{n_t} - \mathcal{R}_{n_s} \right\}$ in all rounds of $K$-CV are independent, we can directly apply concentration inequalities, such as the Hoeffding or Bernstein inequalities, and approximate the probability of the upper bounds. However, with $K>2$, $\left\{ \mathcal{R}_{n_t} - \mathcal{R}_{n_s} \right\}$ is autocorrelated, and the straightforward i.i.d.\ concentration inequalities may not apply. Thus, to establish the bound for the CV error, we first present the Chebyshev inequality for a generic dependent process $\left\{ W_i \right\}$.

%
%
\begin{lemma}
Assume $\left\{ W_i \right\}_{i=1}^{n}$ is sampled from a stationery process and that the autocovariance function $\mathrm{cov} \left(W_{i+l},\; W_{i} \right) := \gamma_l < \infty,\, \forall l \in \mathbb{R}$. The following inequality holds for any $\varpi \in \left[0,1\right)$,

  \begin{equation}
    \mathrm{Pr} \left( \left\vert \overline{W} - \mathbb{E} \left( W \right) \right\vert \leqslant \epsilon \right)
    \geqslant 1 - \frac{ \gamma_0 }{ \epsilon^2 n} \cdot \left( 1 + 2 V_n \left[ W \right] \right),
    \label{Cheby_ineq}
  \end{equation}

\noindent
where $V_n \left[ W \right] := \sum_{l=1}^{n-1} \left\vert \gamma_l\right\vert$ and $\varpi = \gamma_0 / ( n \cdot \epsilon^2 ) \cdot \left( 1 + 2 \cdot V_n \left[ W \right] \right)$.
\label{lem:Cheby_ineq}
\end{lemma}

Based on Definition~\ref{defn:one-round_RC} and using Lemma~\ref{lem:Cheby_ineq} and Theorem~\ref{thm:one_round_VC_bound}, an upper bound for the CV error of the $l$\textsuperscript{th} model in $\Lambda_l$ is now established.
%
%
\begin{theorem}[Upper bounds for the CV error]
Under \textbf{A1}-\textbf{A5} and Lemma~\ref{lem:Cheby_ineq}:

  \begin{enumerate}
  \item If \;$\left\Vert \rho_j \right\Vert_{\Psi_2} \leqslant \infty$, $\forall b \in \Lambda_l$, the following upper bound holds with probability at least $ \left( 1 - \kappa \right)^+$,

    \begin{align}
      \frac{ 1 }{ K } \sum^{ K }_{ q = 1 } \mathcal{R}_{n_s} \left( b, Y_s^q, X_s^q \right)
        \leqslant \frac{ 1 } { K } \sum^{ K }_{ q = 1 }
        \mathcal{R}_{n_t} \left( b, Y_t^q, X_t^q \right)
          + 2 \cdot \mathrm{RC} \left( \Lambda_l, n, K \right)
          + \varsigma
      \label{eq1:convoluted_VC_bound}
    \end{align}

    \noindent
    where
    \begin{equation}
      \kappa = \frac{ \left(\theta - 1 \right)/\theta + 1  }
        { \left( 2 \cdot K \right) \cdot  \log \left( 1 / \varpi \right)}
        \cdot \left( 1 + 2 V_K \left[ T_q \right] \right),
      \label{eq:kappa}
    \end{equation}
    and $\varsigma$ is defined in Theorem~\ref{thm:one_round_VC_bound}.

  \medskip
  \item If \;$\left\Vert \rho_j \right\Vert_{\Psi_1} \leqslant \infty$ and the process $T_q$ has a finite variance, the following upper bound holds with probability at least $ \left( 1 - \kappa \right)^+ $,

    \begin{align}
      \frac{ 1 }{ K } \sum^{ K }_{ q = 1 } \mathcal{R}_{n_s} \left( b, Y_s^q, X_s^q \right)
        \leqslant \frac{ 1 } { K } \sum^{ K }_{ q = 1 }
        \mathcal{R}_{n_t} \left( b, Y_t^q, X_t^q \right)
          + 2 \cdot \mathrm{RC} \left( \Lambda_l, n, K \right)
          + \varsigma
      \label{eq2:convoluted_VC_bound}
    \end{align}
    where

    \[
    \kappa =
      \left\{
        \begin{array}{ll}
          \frac{ 4 \cdot \left[ \frac { \theta - 1 } { \theta } + 1 \right] }
          { K \cdot \log \sqrt[ c ] { 2 / \varpi } }
          \cdot \left( 1 + 2 V_K \left[ T_q \right] \right) ,
          & \mbox{ if } \varpi \in \left[ 2 \exp \left\{ -2c \right\}, 1 \right]. \\
          \frac{ 8 \cdot \left[ \frac { \theta - 1 } { \theta } + 1 \right] }
          { K \cdot \left( \log \sqrt[ c ] { 2 / \varpi } \right)^2 }
          \cdot \left( 1 + 2 V_K \left[ T_q \right] \right) ,
          & \mbox{ if } \varpi \in \left(0 , 2 \exp \left\{ -2c \right\} \right).
        \end{array}\right.
    \]
\label{thm:convoluted_VC_bound}
\end{enumerate}
\end{theorem}

As mentioned in the introduction, ideally the average test error for a model would be derived from validating the model on one or more samples (test sets) that are not used for $K$-CV. The CV error is the in-sample analog of the average test error. Due to the resampling inherent in $K$-CV, the CV error is neither reliable nor stable. Eq.~(\ref{eq2:convoluted_VC_bound}) places a bound on the variation of the CV error. As with eq.~(\ref{eq:VC_bound}), eq.~(\ref{eq2:convoluted_VC_bound}) may be interpreted as the $\left( 1 - \kappa \right) \times 100$\textsuperscript{th} percentile of the CV error distribution. Put another way, given a model in $K$-CV, the interval between $0$ and the RHS of eq.~(\ref{eq2:convoluted_VC_bound}) forms a confidence interval for the CV error. Naturally, for each model in a different model class, eq.~(\ref{eq2:convoluted_VC_bound}) returns a different value for the upper bound of the confidence interval. By graphing these upper bounds, we get the upper bound for the CV errors of all models.

Obviously, eq.~(\ref{eq2:convoluted_VC_bound}) inherits the structure and parameters of eq.~(\ref{eq:VC_bound}). Similar to eq.~(\ref{eq:VC_bound}), $\varsigma$ affects the location of the CV error bound. Intuitively, the more variable the prediction error in each round, the more variable the CV error. In eq.~(\ref{eq2:convoluted_VC_bound}), the more volatile the prediction error in each round (due to a change in $n/K$, variance or anything else), the larger $\varsigma$, which shifts the upper bound of the CV error upwards.

The new parameter $\kappa$ in Theorem~\ref{thm:convoluted_VC_bound} is worthwhile discussing. Similar to the comment for $\varpi$ in eq.~(\ref{eq:VC_bound}), the value of $1 - \kappa$ determines the percentile of the average prediction error distribution in the upper bound. The upper bound holds with probability $(1 - \kappa)^+$, the value of which is determined by $V_K \left[ T_q \right]$, $\varpi$ and other parameters. From eq.~(\ref{eq:VC_bound}), the smaller $\varpi$, the larger the value of $\varsigma$. Ceteris paribus, a larger $\varsigma$ shifts the RHS of eq.~(\ref{eq:VC_bound}) upwards, the upper bound for the prediction error in each round. Since the upper bound for the CV error is a convolution of the upper bounds of prediction errors in $K$ rounds, a shift in the upper bound in each round causes a shift in the upper bound of the CV error. Thus, in eq.~(\ref{eq2:convoluted_VC_bound}), a smaller $\varpi$ increases the value of $\varsigma$, which shifts the bound for the CV error upwards. Since the upper bound for the CV error shifts upwards, the bound tolerates more variation in the CV error. Consequently, the value of $1 - \kappa$ increases and eq.~(\ref{eq2:convoluted_VC_bound}) holds with a larger probability.

\section{Upper bounds for CV errors with $\beta$-mixing data}

In this section, we show that the upper bounds for i.i.d.\ data extend to the non-i.i.d.\ case. We consider a $\beta$-mixing stationary processes for $\left( Y, X \right)$. Thus, assumption \textbf{A2} is modified.

\begin{enumerate}
    \item [\textbf{A2$^\prime$.}]  The data points of $ \left( Y, X \right)$ are sampled from the same stationery $\beta$-mixing data generating process.
\end{enumerate}

Ideally, to ensure that the model chosen with the training set is evaluated accurately on the validation set, the validation and training sets should be independent and drawn from the same population. However, if the data-generating process is dependent across time, the sample may be temporally correlated. As a result, independence between the validation and training sets fails, and the results in section~3 cannot be directly generalized to non-i.i.d.\ data. To address dependence between training and validation sets for non-i.i.d.\ data, we employ the `independent blocks' technique \citep{bernstein1927extension, yu1994rates}.

\begin{figure}[ht]
  \centering
  \subfloat[\label{fig:IB1} original data]
  {\includegraphics[width=0.42\paperwidth]{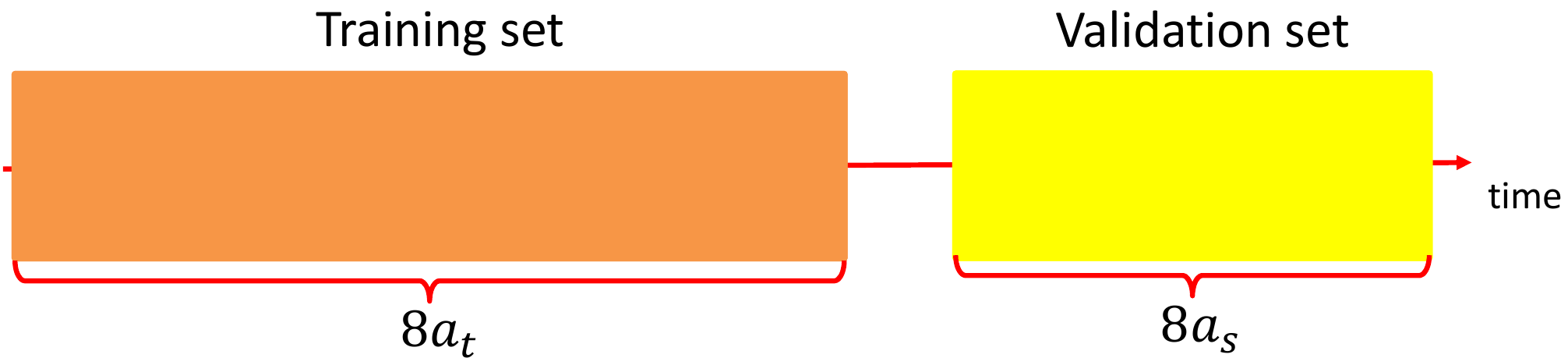}}

  \subfloat[\label{fig:IB2} blocking on original data]
  {\includegraphics[width=0.42\paperwidth]{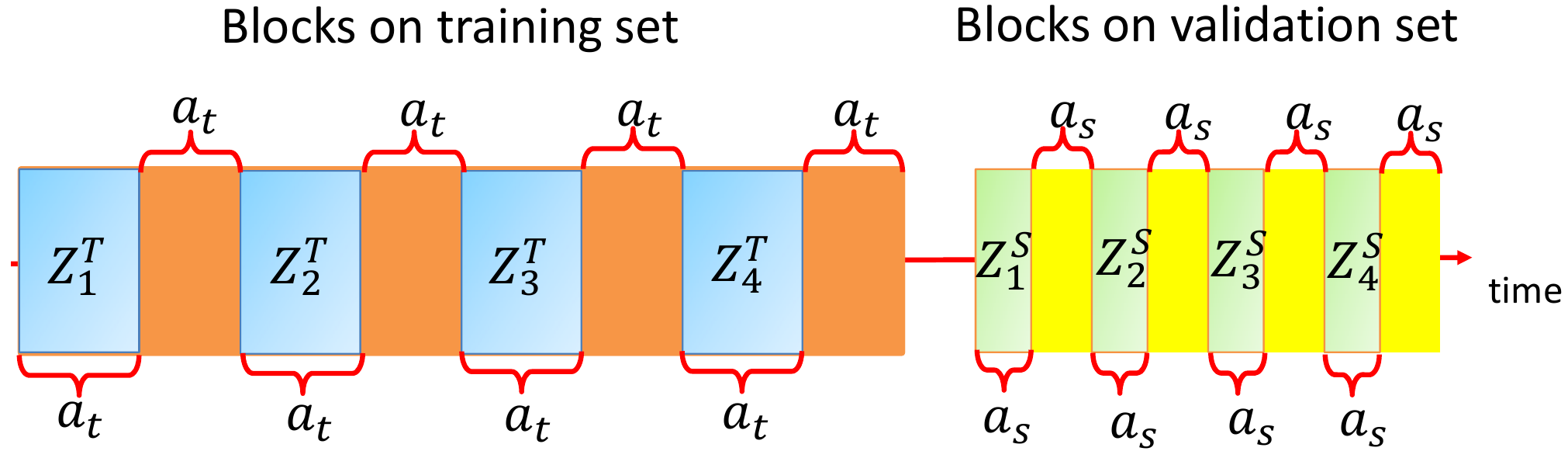}}

  \subfloat[\label{fig:IB3} independent blocks]
  {\includegraphics[width=0.42\paperwidth]{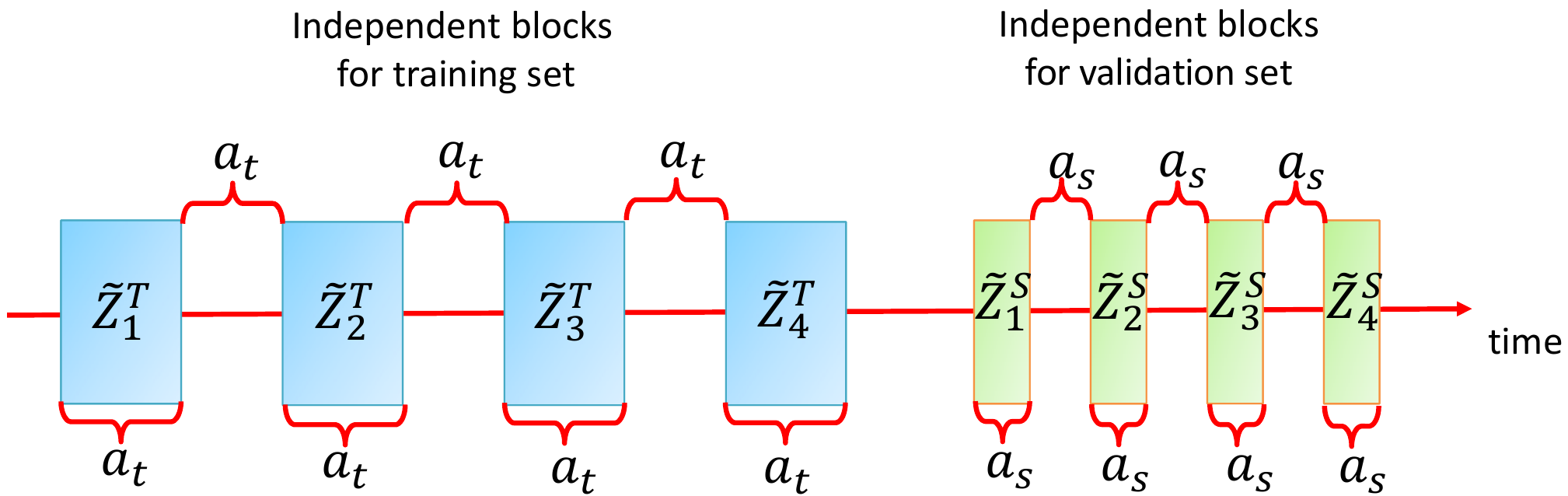}}

  \caption{Illustration of independent blocks with $\mu = 4$}
  \label{fig:IB}
\end{figure}

\subsection{Independent blocks}

Figure~\ref{fig:IB} illustrates the idea of independent blocks. The technique consists of first splitting a sequence of training data, say $S^T$, into two subsequences $S^T_0$ and $S^T_1$, each comprising $\mu$ blocks of $a_t$ consecutive points. Likewise, the validation data $S^S$ are split into $S^S_0$ and $S^S_1$, each comprising $\mu$ blocks of $a_s$ consecutive points. If $z_i = (y_i,\mathbf{x}_i)$, then given a sequence $S^T = \left( z^T_1, z^T_2, \ldots, z^T_{n_t}\right)$ with $n_t = 2 a_t \mu$ and $S^S = \left(z^S_1, z^S_2, \ldots, z^S_{n_s}\right)$ with $n_s = 2 a_s \mu$, $S^T_0$ and $S^S_0$ are constructed as follows
\noindent
\begin{eqnarray}
  S^T_0 = & \left( Z^T_1, Z^T_2, \ldots, Z^T_{\mu}\right), & \mbox{ where } Z^T_i = \left( z^T_{2(i-1)a_t+1}, \ldots , z^T_{(2i-1)a_t}\right) \\
  S^S_0 = & \left( Z^S_1, Z^S_2, \ldots, Z^S_{\mu}\right), & \mbox{ where } Z^S_i = \left( z^S_{2(i-1)a_s+1}, \ldots , z^S_{(2i-1)a_s}\right)
\end{eqnarray}
\noindent
If $\left\{z_i\right\}$ is $\beta$-mixing and the mixing coefficient decays rapidly enough, large $a_t$ and $a_s$ make each block in $S^T_0$ and $S^S_0$ `almost' independent while each block is still drawn from the same distribution due to stationarity.

Next, we create two new sequences of blocks: $\widetilde{S}^T_0 = ( \widetilde{Z}^T_1, \ldots, \widetilde{Z}^T_\mu )$ and $\widetilde{S}^S_0 = ( \widetilde{Z}^S_1, \ldots, \widetilde{Z}^S_\mu )$. The key feature of $\widetilde{S}^T_0$ is that, unlike in $S^T_0$,  $\widetilde{Z}^T_i$ and $\widetilde{Z}^T_j$ are independent for any $i,j$. Also, $\widetilde{Z}^T_i$ and $Z^T_i$ are identically distributed for any $i$, which implies that the `information' encoded in $\widetilde{Z}^T_i$ on average should be identical to the `information' encoded in $Z^T_i$. Hence, the technique creates a sequence of independent, equally-sized blocks, to which standard i.i.d.\ techniques may be applied. Likewise, all blocks in $\widetilde{S}^S_0$ are independent while $\widetilde{Z}^S_i$ and $Z^S_i$ are identically distributed for any $i$. Also, $\widetilde{Z}^T_i$ is independent from $\widetilde{Z}^S_j$, for any $i$ and $j$. This ensures the independent blocks within the training data and within the validation data are also mutually independent. Since the original blocks, $S^T_0$ and $S^S_0$, are `very similar' to their corresponding independent blocks, $\widetilde{S}^T_0$ and $\widetilde{S}^S_0$, we can approximate the performance of $K$-CV on $\left(S^T_0\;,\;S^S_0\right)$ by its performance on $(\widetilde{S}^T_0\;,\;\widetilde{S}^S_0)$. Eventually, the $K$-CV performance on $\left( \widetilde{S}^T_0\;,\; \widetilde{S}^S_0\right)$ may be used as an approximation of its performance on the original sample.

To avoid any ambiguity, we also modify \textbf{A1} slightly. Since we study the out-of-sample performance of any $b \in \Lambda_l$ on a $\beta$-mixing process with independent blocks, we assume measurability in the product probability space $\left(\Pi_{i=1}^{\mu} \Omega_i, \Pi_{i=1}^{\mu} \mathcal{F}_i, \Pi_{i=1}^{\mu} P_i \right)$.

\begin{enumerate}
  \item [\textbf{A1$^\prime$.}]  Let $\left( \Pi_{i=1}^{\mu} \Omega_i, \Pi_{i=1}^{\mu} \mathcal{F}_i, \Pi_{i=1}^{\mu} P_i \right)$ be a probability space. We assume the loss function $Q$ : $\Pi_{i=1}^{\mu} \Omega_i \rightarrow \left[ 0, M \right]$ is $\mathcal{F}$-measurable, $\forall b \in \Lambda_l$. We assume the population error is well defined for any $b \in \Lambda_l$. Specifically, we assume that the Orlicz-$\Psi_\nu$ norms of all the loss processes are well-defined for all $\nu \geqslant 1$.
\end{enumerate}

The following theorem by \citet{yu1994rates} illustrates that, given $a_t$ and $a_s$ large enough, the mean of the bounded and measurable function on the independent blocks are very similar to those of the original blocks.


\begin{theorem}[\citet{yu1994rates}]

Let $\mu > 1$ and assume that $h$ (a generic function) is real-valued, bounded by $\widetilde{M} \geqslant 0$ and measurable in the product probability space $\left(\Pi_{i=1}^{\mu} \Omega_i, \Pi_{i=1}^{\mu} \mathcal{F}_i \right)$. Then, for any $S^T_0$ and $S^S_0$ drawn from a stationery $\beta$-mixing process,
  \begin{eqnarray}
    \left\vert \mathbb{E}_{S^T_0} \left[ h \right] - \mathbb{E}_{\widetilde{S}^T_0} \left[ h \right]\right\vert
    & \leqslant & \left( \mu - 1 \right) \widetilde{M} \beta_{a_t} \\
    \left\vert \mathbb{E}_{S^S_0} \left[ h \right] - \mathbb{E}_{\widetilde{S}^S_0} \left[ h \right]\right\vert
    & \leqslant & \left( \mu - 1 \right) \widetilde{M} \beta_{a_s}
  \end{eqnarray}
where, respectively, $\mathbb{E}_{S^T_0}$ and $\mathbb{E}_{S^S_0}$ are the expectation w.r.t.\ $S^T_0$ and $S^S_0$; $\mathbb{E}_{\widetilde{S}^T}$ and $\mathbb{E}_{\widetilde{S}^S}$ is the expectation w.r.t.\ $\widetilde{S}^T$ and $\widetilde{S}^S$. $\beta_{a_t}$ is the mixing coefficient on each block with $a_t$ consective points across time and $\beta_{a_s}$ is the mixing coefficient on each block with $a_s$ consective points across time.
\label{thm:Yu94}
\end{theorem}
Theorem~\ref{thm:Yu94} allows us to approximate the performance of $K$-CV on $\beta$-mixing data.

\subsection{Upper bounds with $\beta$-mixing data}

Using Theorem~\ref{thm:Yu94} and the McDiarmid inequality, we derive the upper bound for the round-$q$ prediction error in $K$-CV as follows.


\begin{theorem}
Under \textbf{A1$^\prime$},\textbf{A2$^\prime$} and \textbf{A3}-\textbf{A5}, if $\left( \mu - 1 \right) \left[ \beta_{a_t} + \beta_{a_s} \right] < 1$ and the loss function $Q$ is bounded by $M > 0$, the following bound holds, $\forall \varpi \in \left( \; \left( \mu - 1 \right) \left[ \beta_{a_t} + \beta_{a_s} \right] \; , \; 1 \; \right]$, with probability at least $ 1 - \varpi $,
  \begin{align}
    \mathcal{R}_{n_s} \left( b, X_s, Y_s \right)
    \leqslant
    \mathcal{R}_{n_t} \left( b, X_t, Y_t \right)
    + 2\,\mathrm{RC}_{S^S_0} \left( \Lambda_l \right)
    +  M \sqrt{ \frac{ \log \left( 4 / \varpi' \right) } { 2 \mu } }
  \label{eqn:one-round_Rademacher_CV}
  \end{align}
where $\varpi' = \varpi - \left( \mu - 1 \right) \left[ \beta_{a_t} + \beta_{a_s} \right]$ and  $\mathrm{RC}_{\widetilde{S}^S_0} \left( \Lambda_l \right)$ is the upper bound of the model class $\Lambda_l$ on the block $\widetilde{S}^S_0$.
\label{thm:one_round_RC_bound}
\end{theorem}

Theorem~\ref{thm:one_round_RC_bound} shows how to construct the upper bounds for the prediction error in each round of $K$-CV with $\beta$-mixing data. Given the size of each block, the larger the sample size, the larger $\mu$ and the smaller the last two terms of eq.~(\ref{eqn:one-round_Rademacher_CV}). The magnitude of the $\beta$-mixing coefficient together with the autocorrelation of $K$-CV implies the CV error may be more volatile than for the i.i.d.\ bounds. The effect of the additional variability due to $\beta$-mixing is apparent in the following derivation of the upper bound for the CV error.


\begin{theorem}
Based on Lemma~\ref{lem:Cheby_ineq} and Theorem~\ref{thm:one_round_RC_bound}, with probability at least
$1 - 2\mu \left( 1 + 2V_k \left[ T_q \right] \right) / K \log \left( 4/ \varpi' \right) $, the following bound holds
  \begin{align}
    \frac{1}{K} \sum_{q = 1}^{K} \mathcal{R}_{n_s} \left( b, X_s^q, Y_s^q \right)
    \leqslant
    \frac{1}{K} \sum_{q = 1}^{K} \mathcal{R}_{n_t} \left( b, X_t^q, Y_t^q \right)
    + 2\, \mathrm{RC}_{ S^S_0 } \left( \Lambda_l \right)
    + M \sqrt{ \frac{ \log \left( 4 / \varpi' \right) } { 2\mu } }
  \label{eqn:CV_RC_bound}
  \end{align}
\noindent
where
  \begin{align}
    \varpi' \in
     \left(0, 4 \exp \left\{ - \frac{ 2\mu \left( 1 + 2 V_K \left[ T_q \right] \right) } { K } \right\} \right]
  \end{align}
to ensure $1 - 2\mu \left( 1 + 2V_k \left[ T_q \right] \right) / K \log \left( 4/ \varpi' \right) \in \left( 0 , 1 \right]$.
\label{thm:RC_bound_CV}
\end{theorem}

While the approach to deriving the non-i.i.d.\ upper bounds for the CV error is similar to the i.i.d.\ case, the mathematical challenge is different. To ensure $ 1 - \kappa > 0$ for i.i.d.\ data, we need only to consider the magnitude of the autocorrelation of training errors across rounds (encoded in $V_k \left[ T_q \right]$). In contrast, for non-i.i.d.\ data, we need to consider the magnitude of the autocorrelation of training errors across rounds \emph{and} the correlation inherent in the $\beta$-mixing data-generating process (encoded in $\varpi'$). Hence, everything else equal, $\varpi$ is typically larger in non-i.i.d.\ cases. Clearly, stability of the CV error is weaker in non-i.i.d.\ cases.


\section{Application to the lasso}

In sections~3 and~4, we obtain the theoretical results for the upper bound of the CV error (referred to simply as `upper bound' below) under different assumptions. In this section, we demonstrate that the upper bound is empirically computable using a dedicated Python package \texttt{CVbounds}(documented in a supplement to the paper). We apply the upper bound to the lasso on i.i.d. data and show through simulations that the upper bound quantifies the variability of the CV error and improves the sparsity of variable selection while retaining all the relevant variables.

\subsection{Computation of the upper bound with the lasso}

To compute the upper bound of the CV error for the lasso with given $\lambda$, we need to solve several issues including, in particular, the definition of the model class. First, we define the lasso:
\begin{equation}
  \min_{\beta} \frac{1}{n_t} \left\Vert Y - X\beta \right\Vert_{2}^{2} + \lambda \left\Vert \beta \right\Vert_{1}
  \label{lasso}
\end{equation}
where $\lambda$ is the penalty parameter, $\left\Vert \cdot \right\Vert_{1}$ and $\left\Vert \cdot \right\Vert_{2}$ are the $L^1$ and $L^2$ norms, respectively.

The definition of a model class in the lasso, for a given value of $\lambda$, is illustrated in Figure~\ref{fig:model_class}. Given the data-generating process (DGP) of $Y$ and $X$, a typical OLS class is
\begin{align}
  \{
     Y = & X\beta_{\mbox{ols}} + \epsilon
     \;\vert\;
     \beta_{\mbox{ols}} \in \mathbb{R}^{p \times 1} \mbox{ is the L2 error minimizer on }  
     \left(Y, X\right), \notag \\ 
     & \forall \left(Y, X\right) \in \mathbb{R}^{n \times (p + 1)}
  \}.
\end{align}
We avoid the more general linear model class definition $\{ Y = X \beta + \epsilon \vert \beta \in \mathbb{R}^{p \times 1} \}$ that represents all regression methods (logit, LAD regression, etc.) because we are interested only in the properties of OLS estimators. Further, to construct a tight upper bound requires tailoring the analysis to a given DGP. Thus, we define the OLS class in terms of the DGP.

By contrast, a lasso class depends on the penalty parameter $\lambda$ in eq.~(\ref{lasso}). Thus, given the DGP of $Y$ and $X$,
\begin{align}
  \{
     Y = & X \beta_{\mbox{lasso}} + \epsilon
     \;\vert\;
     \beta_{\mbox{lasso}} \mbox{ is the minimizer of eq.~(\ref{lasso}) on } \left(Y, X\right) \notag \\
      & \mbox{ and } \left\Vert \beta_{\mbox{lasso}} \right\Vert_1 = \alpha(\lambda), \forall \left(Y, X\right) \in \mathbb{R}^{n \times (p + 1)}, \mbox{ given } \lambda 
  \},\notag
\end{align}
where $\alpha$ is determined by the value of $\lambda$. Since, given $\lambda > 0$, the estimated lasso regression coefficients are always on the boundary of the corresponding $L^1$ constraint, the class of the lasso given the DGP of $Y$ and $X$ is a subset of the boundary points of the feasible area (the orange rectangle in Figure~\ref{fig:mc1}). However, not all points on the boundary are in the lasso class. In sparse modeling and variable selection, especially in high-dimensional spaces, variables are often eliminated due to the $L^1$ penalty. Thus, only a subset (typically a proper subset) of the boundary points is relevant to the sample estimate of the lasso for $Y$ and $X$. The lasso class, given $\lambda$, is typically distributed around one (or more) corner(s) of the feasible area (the red segment in Figure~\ref{fig:mc1}). It is this subset that is relevant for computing the upper bound of the CV error of the lasso.

\begin{figure}[h]
  \centering
  \subfloat[\label{fig:mc1}lasso class]
  {\includegraphics[width=0.22\paperwidth]{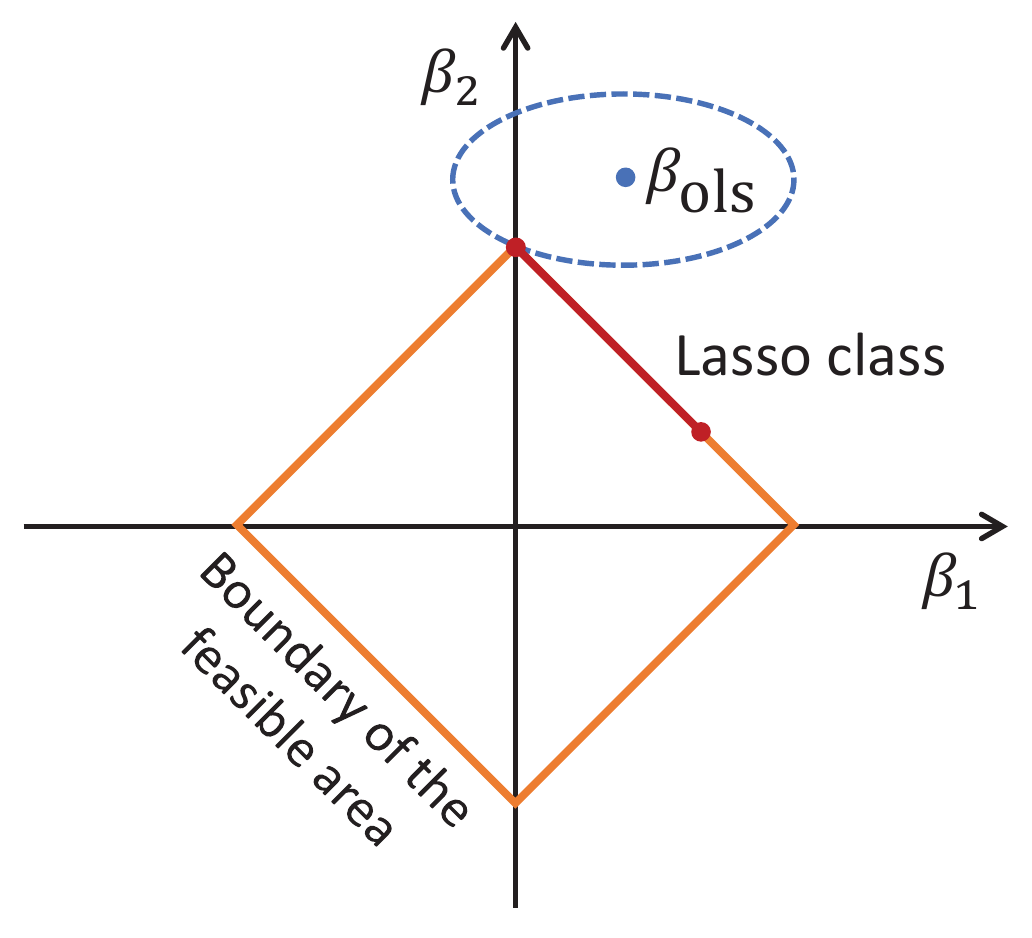}}
  \hfill
  \subfloat[\label{fig:mc2}bootstrap estimate]
  {\includegraphics[width=0.22\paperwidth]{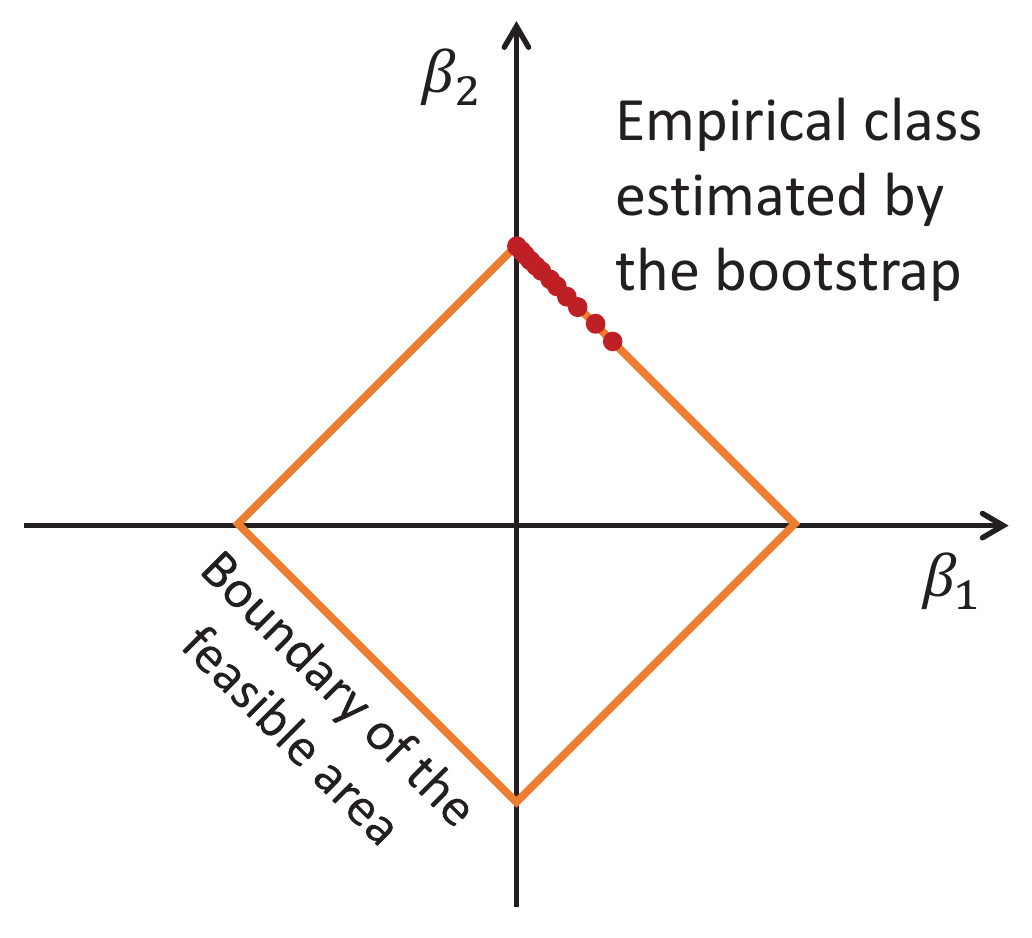}}
  \hfill
  \subfloat[\label{fig:mc3}$K$-CV estimate ($K=5$)]
  {\includegraphics[width=0.22\paperwidth]{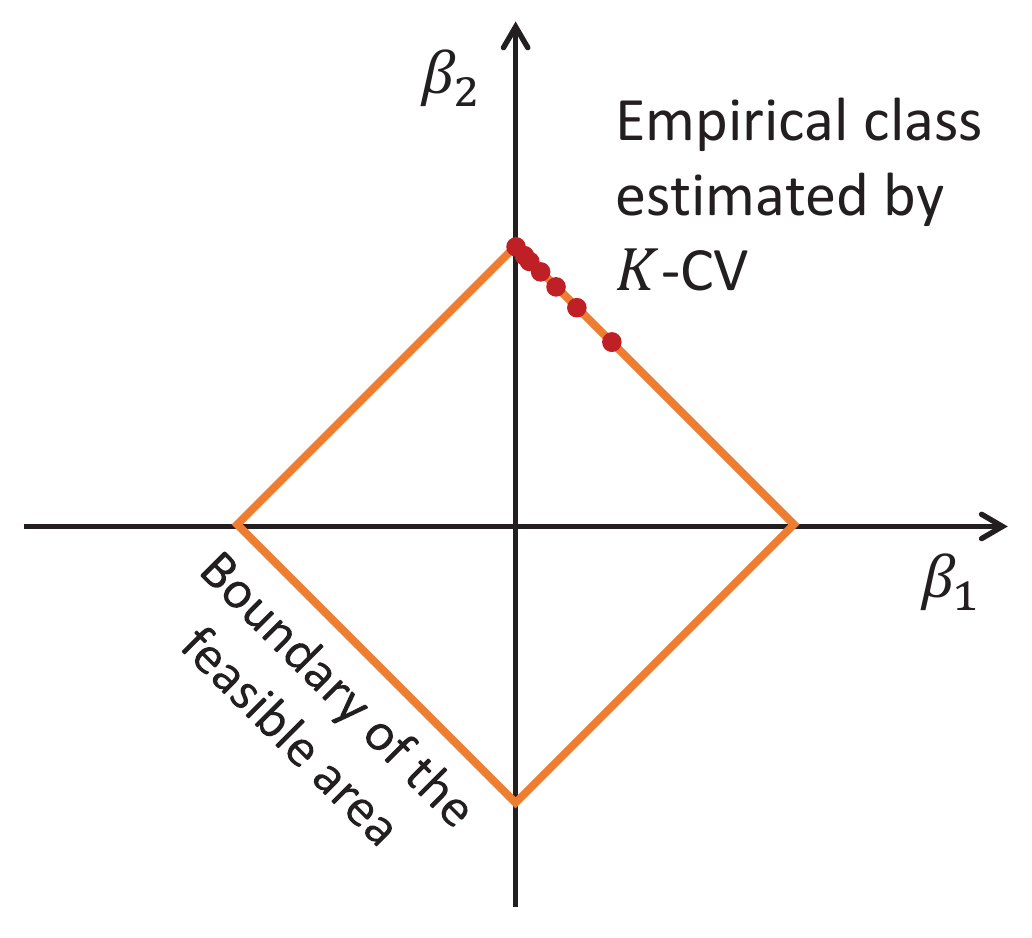}}

  \caption{Illustration of a lasso class and empirical estimates}
  \label{fig:model_class}
\end{figure}

Unfortunately, in empirical applications it is difficult to know a priori the location of the lasso class. One solution is to estimate the lasso class by resampling. In principle, bootstrapped samples could be used to train the lasso for a given $\lambda$ and the estimated regression coefficients used as the empirical class, as in Figure~\ref{fig:mc2}. Since the random sample in simulations is i.i.d., we expect the lasso estimates for a given $\lambda$ on each bootstrapped sample to be similar on average. Alternatively, as illustrated in Figure~\ref{fig:mc3}, the $K$ estimates of the lasso from $K$-CV could be used as the empirical class for the given value of $\lambda$. In simulations, the random sample is i.i.d. and the training sets in any two rounds have a number of points in common. As a result, given $\lambda$, we expect the lasso estimates from training sets in any two rounds to be similar. The larger $K$, the more similar the model estimates from the training sets in any two rounds. Thus, given the data, we expect each element in the empirical class to return similar loss distributions for the bootstrap and $K$-CV, on average.

A working definition of the lasso class has two direct implications. First, it ensures the upper bound is tailored to the DGP of $Y$ and $X$, resulting in a tight upper bound. Moreover, under i.i.d. settings we expect that on average $\theta \approx 1$ for the empirical class. Specifically, if there is only one element in the empirical class (or if all the elements are the same), then the numerator and denominator in eq.~(\ref{a5:theta}) will be equal for the empirical class and $\theta = 1$. Correspondingly, if the elements in the empirical class return very similar loss distributions on given data, the numerator and denominator in eq.~(\ref{a5:theta}) for the empirical class will be close to one another and $\theta \approx 1$. The lasso simulations below support both tight upper bounds and $\theta \approx 1$.

To compute the upper bound, we also need to choose the value of $K$. In terms of the reliability of model selection in $K$-CV, \citet{cawley2010over} and \citet{nadeau2000inference} argue that reducing the variance of the CV error is more important than unbiasedness. We also know that, given $n$, a larger $K$ leads to a higher variance. Thus, we choose $K=2$ and use 2-CV to compute the CV error. With 2-CV, the round~1 training sets do not share any common points with the round~2 training sets, implying the training errors are not autocorrelated across rounds and $V_K \left[ T_q \right] = 0$. Lastly, to set $\kappa$ we need to specify $\varpi$. We choose $\varpi=0.10$, implying from eq.~(\ref{eq:kappa}) that $\kappa\approx 0.10$ for the subgaussian case, which in turn implies a 90\% upper bound for the CV error. 

\subsection{Simulations}

We compute the upper bounds for the lasso CV errors in three settings. In all three settings, the number of variables $p = 100$. The number of observations $n = 100$ in the first setting, $n = 200$ in the second setting and $n = 400$ in the third setting. All data points are identically and independently distributed. The outcome variable $Y \in \mathbb{R}^{n \times 1}$ is generated by
\begin{align}
                     Y & = X_1 \beta_{nonzero} + X_0 \beta_{zero} + e   \notag \\
  \mbox{where }\beta_{nonzero} & = \left[ 3, 4, 5, 6, 7\right]^{T} \notag \\
  \mbox{and }  \beta_{zero} & = \left[ 0, \ldots, 0\right]^{T}  \notag
\end{align}
where $X_1 \in \mathbb{R}^{n \times 5}$, $ X_0 \in \mathbb{R}^{n \times 95}$ and $\left[ X_1, X_0 \right]$ are generated from a zero-means multivariate Gaussian distribution with covariance matrix consisting of 1s~on the main diagonal and 0.5 for the remaining elements. Each variable in $\left[ X_1, X_0 \right]$ is independent from the noise term $e$, which is Gaussian with mean $0$ and variance $1$. Because in lasso regression analysis the mean square error for each validation set is distributed $\chi^2$ with $n/K$ degrees of freedom, we choose the subgaussian bound, which converges to the population error at the rate $1/\sqrt{n}$ \citep{bartlett2005local}.

Ideally, the upper bound for the lasso should demonstrate two key properties. First, since the purpose of variable selection with the lasso is to avoid underfitting and overfitting on the training set, the upper bound should be able to distinguish between underfitting and overfitting for all models. Second, since the upper bound quantifies the worst-case out-of-sample performance of each combination of variables, the upper-bound minimizer in a sense performs `better' than the other models. In the context of the lasso, `better' means a variable selection that is more sparse, accurate and stable. In the following simulations, we demonstrate the shape, location and variable selection properties of the upper bound.

\subsubsection{Simulation 1: shape and location of the upper bound}

In the first simulation, we investigate the shape and location of the upper bound for the lasso under different settings. As noted in section~1.1, the ideal measure of a model's out-of-sample performance is the average test error. Simulation, of course, allows us to obtain the distribution of the average test error for any given model. The simulation is conducted as follows. To derive detailed error plots, we choose a sequence of values for the lasso penalty parameter: $\lambda = \left\{ 0.10, 0.15, 0.20, \dots, 0.45 \right\}$. For each $\lambda$, we estimate the lasso on $\left( Y, X \right)$, resulting in 8~different lasso regression models, average training errors and CV errors. For each model, we simulate 20,000 new samples from the same DGP, each with the same size as the validation set in $2$-CV, producing 20,000 test errors generated outside $2$-CV for each estimated model. Next, we combine the 20,000 test errors into 10,000 pairs, obtaining 10,000 average test errors for each model. Lastly, we plot for each $\lambda$ the 90\% upper bound and the empirical 90\textsuperscript{th} percentile of the average test error distribution (used to approximate the population 90\textsuperscript{th} percentile of the average test error distribution). For the 90\% upper bound to be empirically useful, it must be a tight upper bound, i.e., one that is close to the 90\textsuperscript{th} percentile of the average test error distribution.

The lasso simulation results are plotted in Figure~\ref{fig:sim3}: the average training error is shown in blue, the CV error in red, the empirical 90\textsuperscript{th} percentile in black and the 90\% upper bound in green. The figures illustrate the location, shape and convergence tendency of the upper bound using the same scale. In all three figures, the average training error is the lowest, indicating that the models tend to overfit the training data.

\begin{figure}
  \subfloat[\label{fig:sim31} $n/K = 50$]
  {\includegraphics[width=0.22\paperwidth]{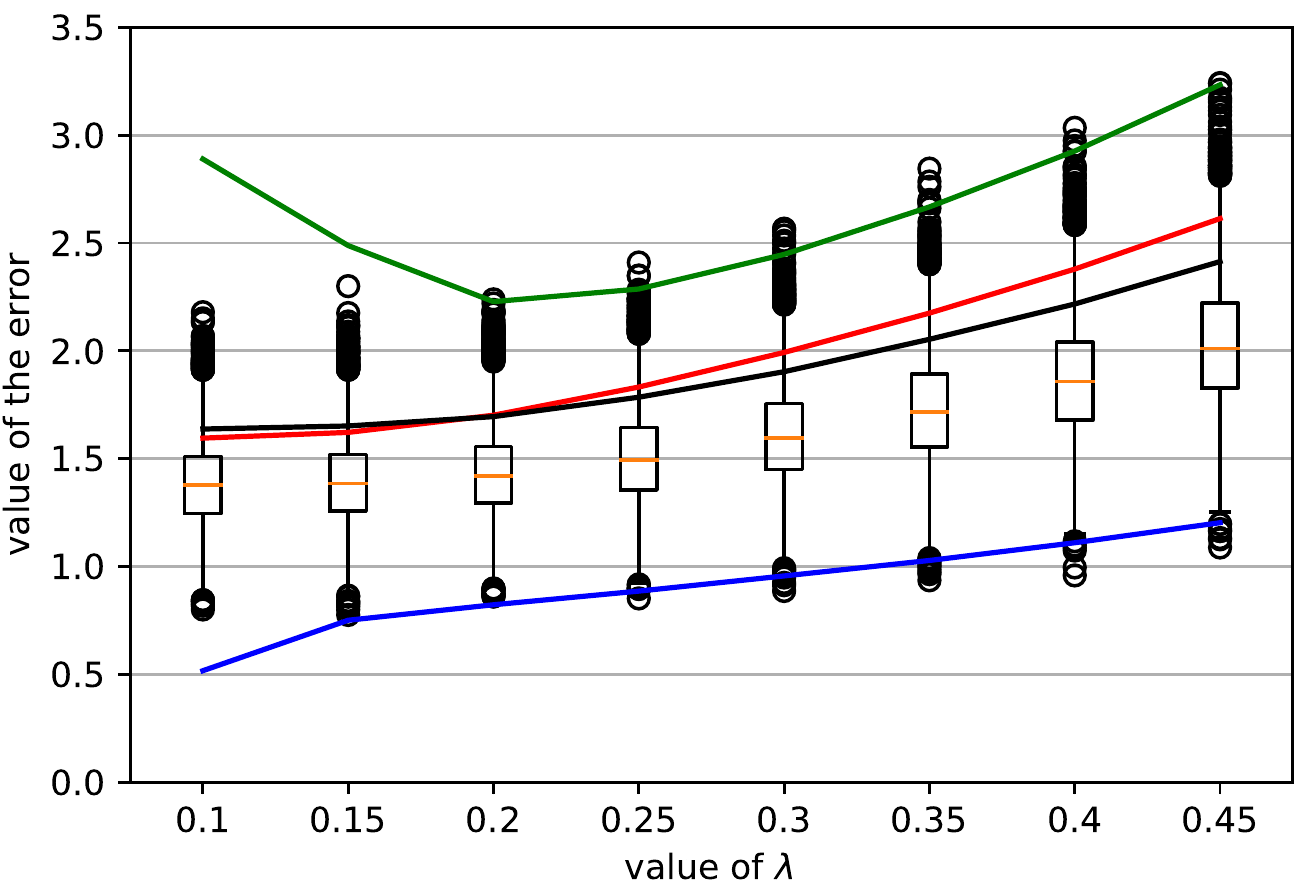}}
  \subfloat[\label{fig:sim32} $n/K = 100$]
  {\includegraphics[width=0.22\paperwidth]{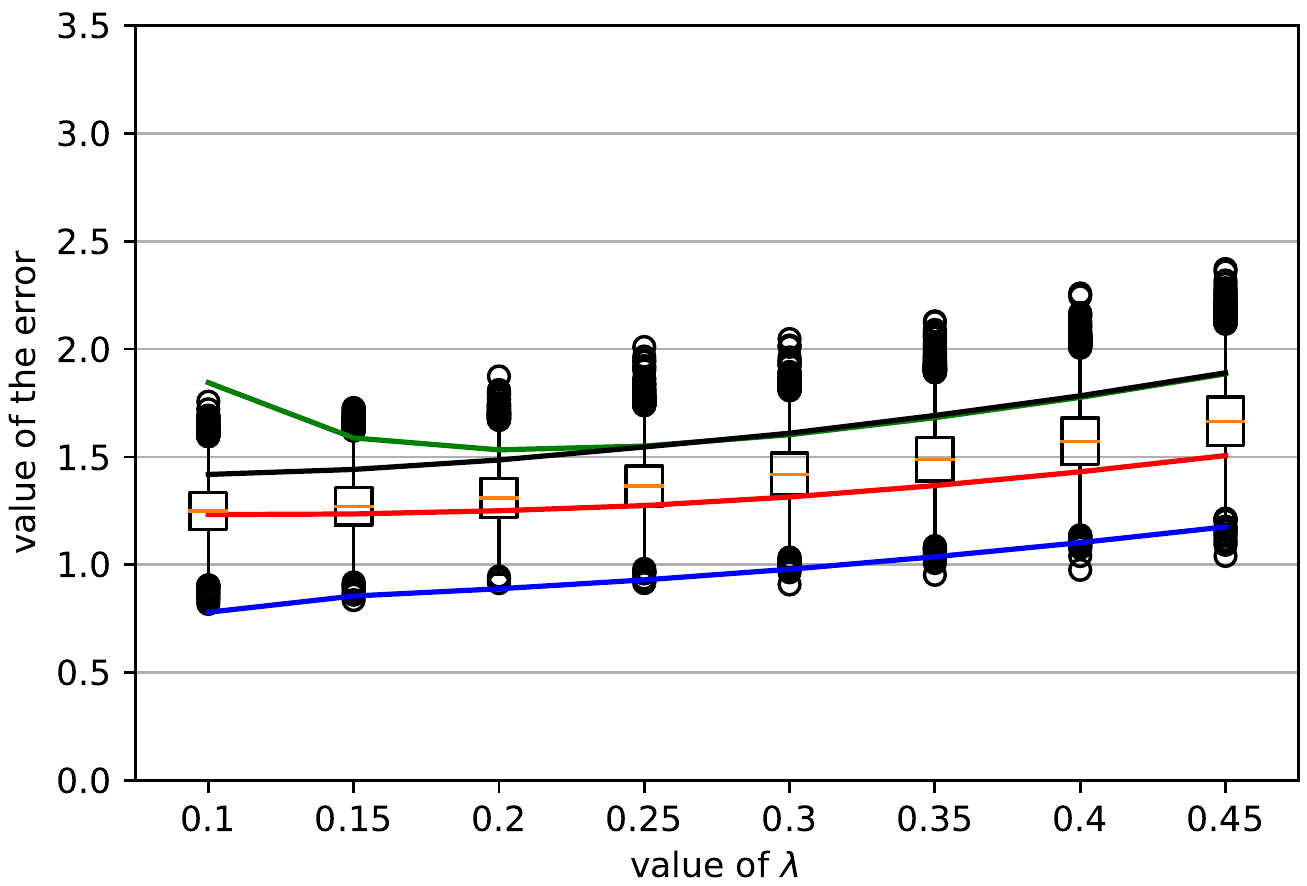}}
  \subfloat[\label{fig:sim33} $n/K = 200$]
  {\includegraphics[width=0.22\paperwidth]{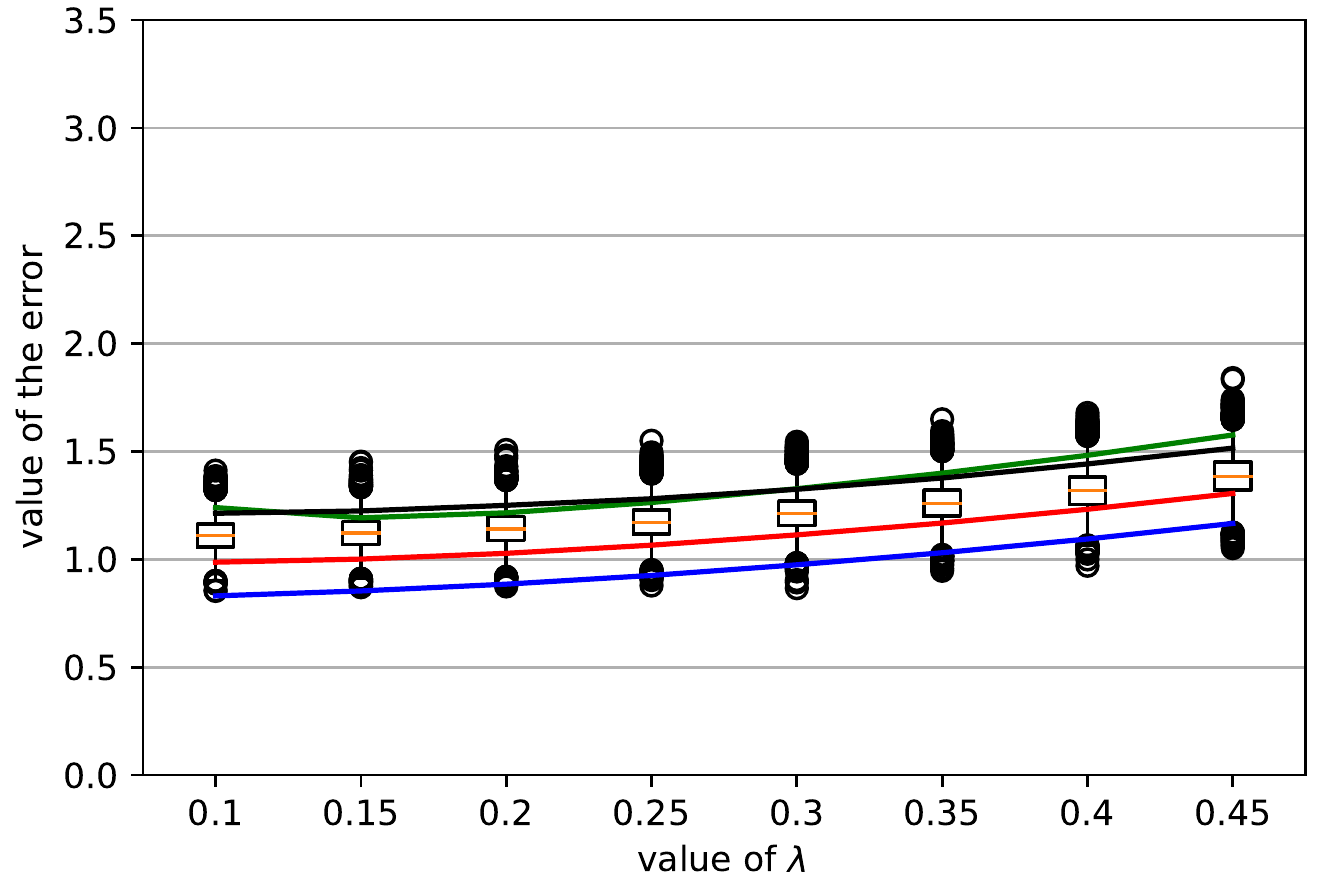}}

  {\includegraphics[width=0.7\paperwidth]{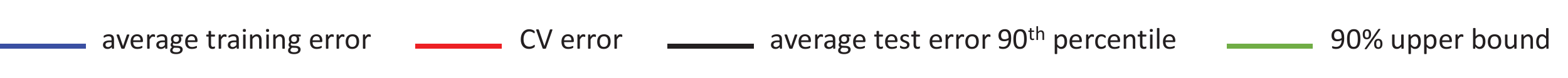}}

  \caption{Upper bound of the CV error for the lasso regression $(p = 100)$}
  \label{fig:sim3}
\end{figure}

Figure~\ref{fig:sim31} plots the simulation for $n/K = 50$. Since there are 50 data points in each fold, the lasso chooses at most 50 from the $p = 100$ variables. Since $n/K < p$ in this case, the degrees of freedom for the training error and the prediction error are low. Moreover, the lasso regression coefficients for each $\lambda$ are more unstable given the high-dimensional setting, resulting in training and CV errors with high variances. Owing to these two problems, the 90\% upper bound in Figure~\ref{fig:sim31} is not especially tight, although it does have approximately the same shape as the 90\textsuperscript{th} percentile across models. Figure~\ref{fig:sim32} plots the simulation for $n/K = 100$. In this case, the upper bound is much closer to the 90\textsuperscript{th} percentile compared with Figure~\ref{fig:sim31}. For $n/K = 200$ in Figure~\ref{fig:sim33}, the upper bound for a given $\lambda$ is more stable, returning an even tighter bound than Figure~\ref{fig:sim32}. Figure~\ref{fig:sim3} thus shows that the upper bounds and 90\textsuperscript{th} percentiles have similar shapes and that, with the exception of the high-dimensional case, the upper bound is close to or above the 90\textsuperscript{th} percentile. The convergence of the upper bound to the average test error as $n/K$ increases is clearly displayed in Figures~\ref{fig:sim31} to~\ref{fig:sim33}.

Overall, the simulations show that the upper bound is a good finite-sample approximation for the 90\textsuperscript{th} percentile of the average test error. Since the average test error is a reliable measure of the out-of-sample performance of the model, the upper bound is a reliable measure of the maximum-extent variation (with 90\% probability) of the out-of-sample performance of the model. Thus, provided that these simulations are robust and representative, the upper bound will be a reliable estimator of the 90\textsuperscript{th} percentile of the average test error.

\subsubsection{Simulation 2: robustness checks}

In the second set of simulations, we check the robustness of the upper bound to sampling randomness, referred to simply as \emph{robustness} below. Sampling randomness is implemented in the simulations through different random seeds.

\subsubsection*{Robustness of variable selection}

Firstly, we demonstrate the robustness of the upper-bound minimizer in variable selection. We use the same values as above for the parameters, repeating the simulations 120 times with different random seeds. We compare the performance of the CV-error minimizer ($\lambda_{CV}$) with the upper-bound minimizer ($\lambda_{CUB}$) in terms of the sparsity, stability and accuracy of variable selection. Sparsity is summarized empirically by the average number of the variables selected by $\lambda_{CV}$ (or $\lambda_{CUB}$). Stability is summarized empirically by the variance of the number of the variables selected by $\lambda_{CV}$ (or $\lambda_{CUB}$). To measure accuracy quantitatively, we use the following definition.
\begin{definition}[90\% accuracy of variable selection]
For a variable-selection algorithm, define \emph{90\% accuracy} to be when all the variables in $X_1$ (the variables with non-zero regression coefficients in the population) are selected with at least 90\% probability.
  \label{def:90-accurate}
\end{definition}

\noindent
Because we construct 90\% upper bounds, we expect the upper bound to perform well in 90\% of the simulations. Hence, we also choose 90\% to define empirical accuracy. Since we repeat each simulation 120 times, $\lambda_{CV}$ (or $\lambda_{CUB}$) will be empirically 90\% accurate if all 5 of the $X_1$ variables are selected in at least 108 repetitions. Table~\ref{table:simulation} compares the accuracy, sparsity and stability levels for $\lambda_{CV}$ and $\lambda_{CUB}$ across the different $n/K$ settings.

\begin{table}
\centering
  \caption{Accuracy, sparsity and stability of variable selection for the lasso}
  \label{table:simulation}
  \begin{tabular}{|r|c|r|r|c|r|r|}
     \hline
            & \multicolumn{3}{c|}{$\lambda_{CV}$} & \multicolumn{3}{c|}{$\lambda_{CUB}$} \\
     \cline{2-7}
            & & \multicolumn{2}{c|}{\footnotesize total number of}
            & & \multicolumn{2}{c|}{\footnotesize total number of} \\
            & \footnotesize 90\% & \multicolumn{2}{c|}{\footnotesize variables selected}
            & \footnotesize 90\% & \multicolumn{2}{c|}{\footnotesize variables selected} \\
     \cline{3-4} \cline{6-7}
     $n/K$  & \footnotesize accuracy
            & \footnotesize \ average & \footnotesize variance
            & \footnotesize accuracy
            & \footnotesize \ average & \footnotesize variance \\
     \hline\hline
      50 & Yes & 11.59 & 10.44 & Yes & 11.16 & 9.48 \\ \hline
     100 & Yes &  9.01 &  5.61 & Yes &  8.12 & 4.59 \\ \hline
     200 & Yes &  6.75 &  2.54 & Yes &  6.11 & 1.28 \\ \hline
  \end{tabular}
\end{table}

When $n/K = 100$, Table~\ref{table:simulation} shows that $\lambda_{CV}$ selects on average roughly one more variable compared with $\lambda_{CUB}$. Since $\lambda_{CV}$ and $\lambda_{CUB}$ are both 90\% accurate, the extra variable selected by $\lambda_{CV}$ is highly likely to be redundant. Put another way, with a high probability, $\lambda_{CUB}$ reduces the number of selected redundant variables by roughly 25\% (from 4 to 3). Moreover, $\lambda_{CUB}$ reduces the variance of the number of the selected variables by approximately 17\% compared with $\lambda_{CV}$. Thus, maintaining 90\% accuracy, $\lambda_{CUB}$ on average delivers more sparsity and superior stability in variable selection relative to $\lambda_{CV}$.

When $n/K = 200$, Table~\ref{table:simulation} shows that both $\lambda_{CV}$ and $\lambda_{CUB}$ satisfy 90\% accuracy. However, with a high probability, $\lambda_{CUB}$ reduces the number of selected redundant variables by 38\% (from 1.75 to 1.11) and the variance of the number of selected variables by almost 50\%. Thus, relative to $\lambda_{CV}$, $\lambda_{CUB}$ again delivers superior sparsity and stability along with a similar level of accuracy in variable selection.

When $n/K = 50$, the lasso is forced to select among 100 variables using just 50 observations. Due to the high-dimensional setting, both the average training error and the CV error have significantly higher variances compared with $n/K = 100$ or $200$, given $\lambda$. Despite being 90\% accurate, the sparsity of variable selection for both $\lambda_{CV}$ and $\lambda_{CUB}$ is reduced: both select roughly 11 variables, implying, with a high probability, that 6 redundant variables are selected by each method. However, compared with $\lambda_{CV}$, $\lambda_{CUB}$ reduces the variance of the number of variables selected by roughly 10\%. Thus, while maintaining 90\% accuracy, the stability of variable selection for $\lambda_{CUB}$ remains slightly superior to that of $\lambda_{CV}$.

\subsubsection*{$L^1$ shrinkage}

Another perspective on the robustness of variable selection is provided in Figure~\ref{fig:L1norm}, which shows the distributions of the $L^1$ norms of the lasso regression coefficients for $\lambda_{CV}$ ($\left\Vert \beta_{CV} \right\Vert_1$) and $\lambda_{CUB}$ ($\left\Vert \beta_{CUB} \right\Vert_1$). Figures~\ref{fig:hist1}-\ref{fig:hist3} use the same scale to illustrate the convergence in performance for $\lambda_{CV}$ and $\lambda_{CUB}$. Figure~\ref{fig:hist1} illustrates the high variance problem associated with high-dimensional space ($n/K = 50$): both $\left\Vert \beta_{CV} \right\Vert_1$ and $\left\Vert \beta_{CUB} \right\Vert_1$ are disperse, consistent with the large variances for $\lambda_{CUB}$ and $\lambda_{CV}$ in Table~\ref{table:simulation}. With $n/K = 100$ in Figure~\ref{fig:hist2}, the distributions of both $\left\Vert \beta_{CV} \right\Vert_1$ and $\left\Vert \beta_{CUB} \right\Vert_1$ become more compact, implying much smaller variances for $\left\Vert \beta_{CV} \right\Vert_1$ and $\left\Vert \beta_{CUB} \right\Vert_1$. The convergence is clearly shown with $n/K = 200$ in Figure~\ref{fig:hist3}, where the distributions of $\left\Vert \beta_{CUB} \right\Vert_1$ and $\left\Vert \beta_{CV} \right\Vert_1$ are closer to 25, the $L^1$ norm of $\beta_{nonzero}$ (the population regression coefficients).

\begin{figure}[ht]
  \centering

  \subfloat[\label{fig:hist1} $n/K = 50$]
  {\includegraphics[width=0.24\paperwidth]{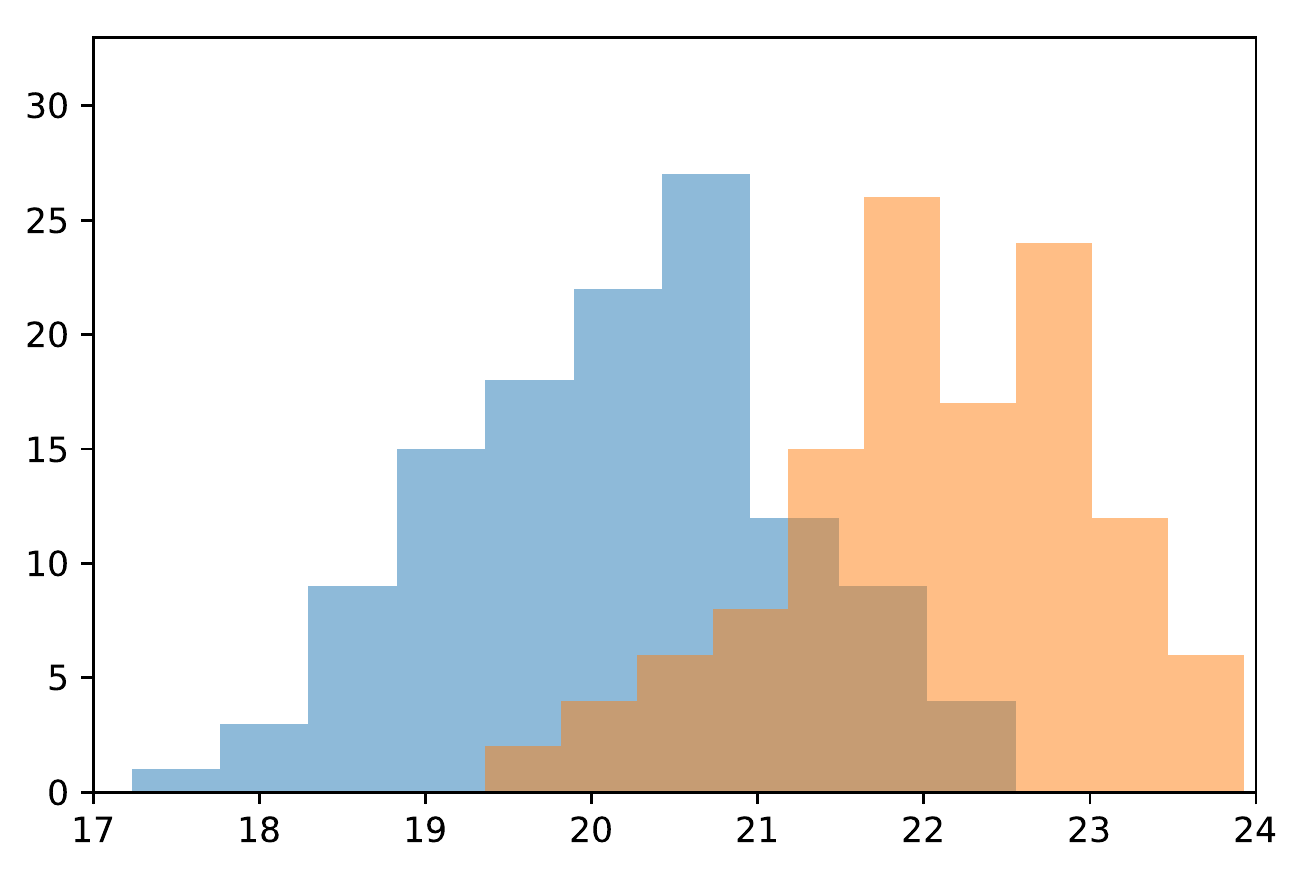}}
  \subfloat[\label{fig:hist2} $n/K = 100$]
  {\includegraphics[width=0.24\paperwidth]{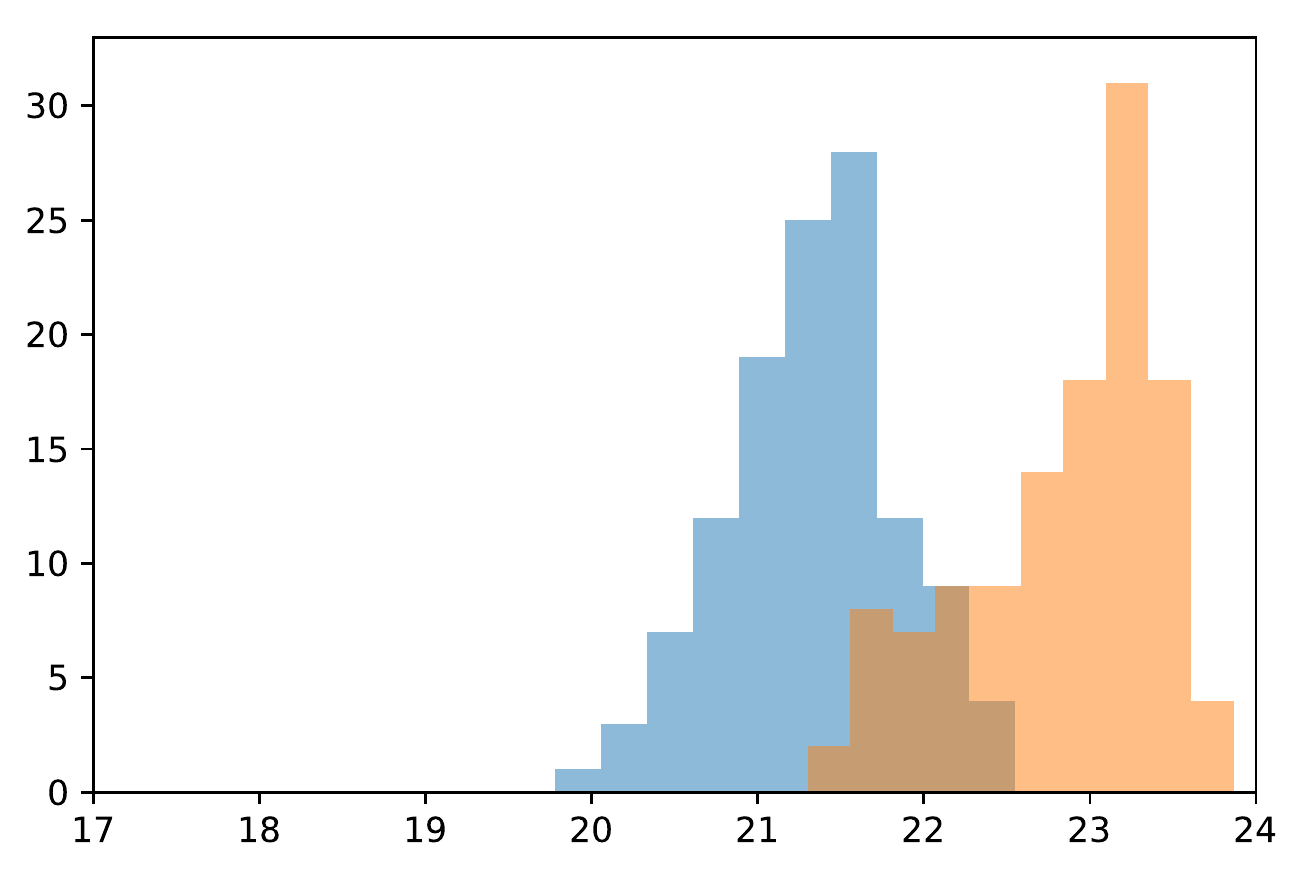}}
  \subfloat[\label{fig:hist3} $n/K = 200$]
  {\includegraphics[width=0.24\paperwidth]{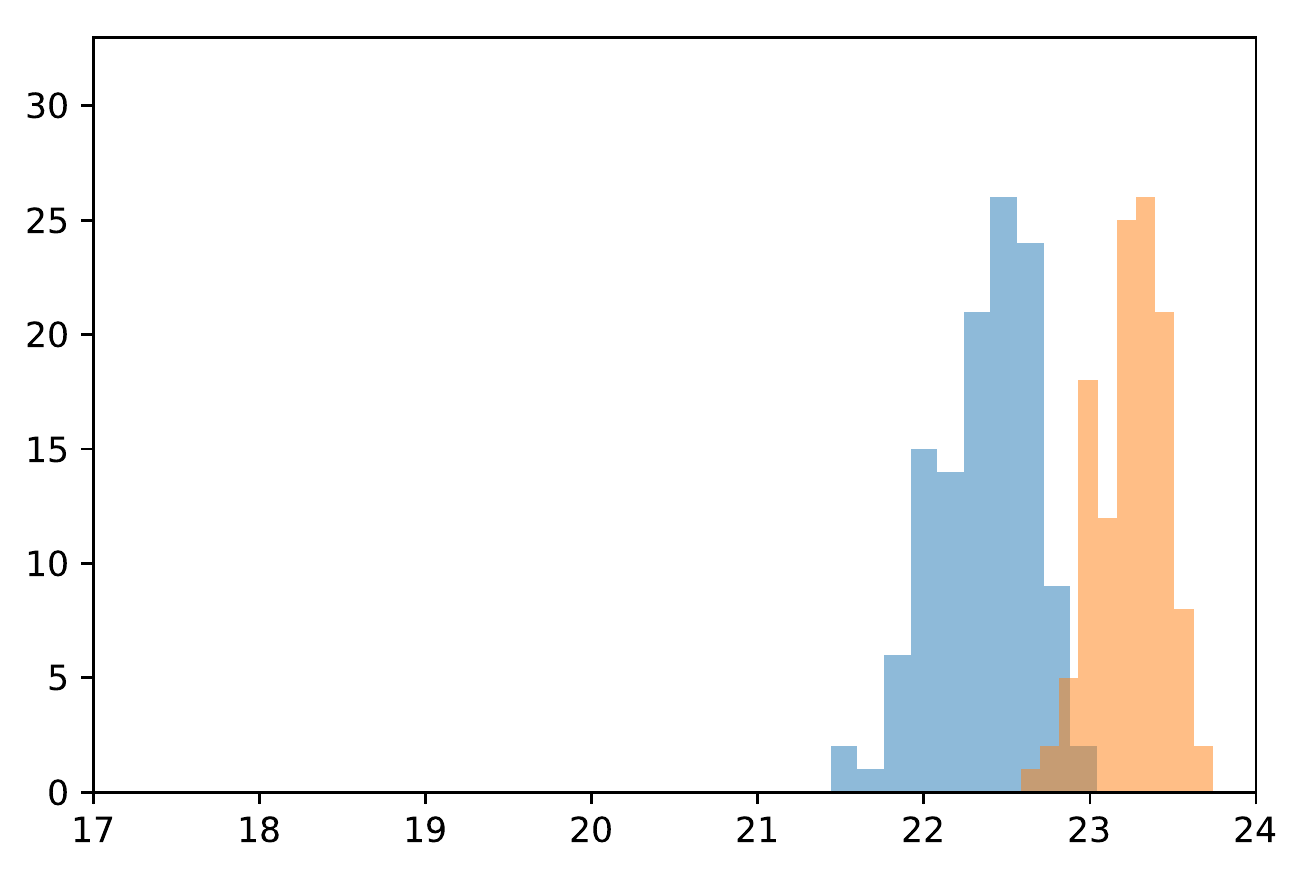}}

  {\includegraphics[width=0.2\paperwidth]{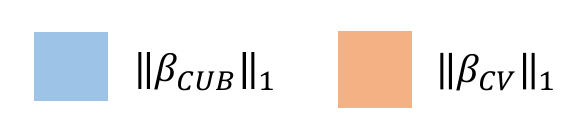}}

  \caption{$L^1$ norms of the lasso for $\lambda_{CUB}$ and $\lambda_{CV}$ given $n/K$}
  \label{fig:L1norm}

\end{figure}

All three panels in Figure~\ref{fig:L1norm} also show a distribution for $\left\Vert \beta_{CUB} \right\Vert_1$ that lies to the left of the distribution for $\left\Vert \beta_{CV} \right\Vert_1$, indicating that $\lambda_{CUB}$ shrinks the regression coefficients more aggressively than $\lambda_{CV}$ and, hence, that $\lambda_{CUB}$ offers more sparsity. Essentially, tuning the value of $\lambda$ represents a trade-off between the sparsity and the accuracy of variable selection. The larger $\lambda$, the sparser the variable selection and the larger the probability we eliminate variables with non-zero regression coefficients. The overall impression from Table~\ref{table:simulation} and Figure~\ref{fig:L1norm} is that $\lambda_{CUB}$ hits a sweet spot in terms of the trade-off between sparsity and accuracy: it increases the level of shrinkage without breaching the 90\% accuracy level.

\subsubsection*{Robustness of shape and location}

Now we focus on the robustness of the shape and location of the upper bounds. First, we consider the relative positions of the upper bound and the 90\% percentile of the average test error for given $\lambda$. Thus, we plot the curve of the 90\% upper bounds (CUB) and the curve of the 90\textsuperscript{th} percentiles of the average test errors (CAT). For comparison we also plot the CV error curve (CCV). Since the 90\% upper bound is also the upper bound of the 90\textsuperscript{th} percentile, we expect the CUB to be close to or above the CAT. Second, we consider the shapes of the CUB and CAT. If they have similar shapes, their properties (such as the locations of the minima) will also be similar. Third, we consider whether $\lambda_{CUB}$ improves the sparsity-accuracy trade-off relative to $\lambda_{CV}$, in a similar vein to heuristic decision rules for the lasso (such as the one-standard-error-rule).

Figures~\ref{fig:sim4}-\ref{fig:sim6} in Appendix~B plot for each setting the CUB, CAT and CCV from the lasso simulations repeated for 120 different random seeds. To keep the number of plots manageable, we report only the first 20~results for each repetition, corresponding to random seeds from 5 to 100. Generally speaking, the 20 plots in each of Figures~\ref{fig:sim4}-\ref{fig:sim6} are representative of the 120 repetitions in each setting. Note that while the vertical scale is the same across all plots for each figure, it changes slightly between figures.

The plots for $n/K = 100$ in Figure~\ref{fig:sim4} show that the location and shape of CCV fluctuates across random seeds. The vertical position of the CCV varies between $1.0$ and $2.0$ while the CUB varies between $1.5$ and $2.0$, revealing the inferior reliability of the CV error. The U-shape of the CUB also remains stable across the 20~plots. As expected (given we construct a 90\% upper bound), the CUB is close to or above the CAT in most of the Figure~\ref{fig:sim4} plots. The 4 exceptions (out of 20 plots) are when CUB is below CAT (Figures~\ref{fig:sim411} and~\ref{fig:sim413}) and some way above CAT (Figures~\ref{fig:sim48}, ~\ref{fig:sim412}).

Figure~\ref{fig:sim4} also shows that $\lambda_{CUB}$ is more robust than $\lambda_{CV}$ across different random seeds. In 18 of the 20 plots, $\lambda_{CUB} = 0.2$; by contrast, the value of $\lambda_{CV}$ fluctuates between $0.1$, $0.15$ and $0.2$ across the 20 plots. The robustness of $\lambda_{CUB}$ is explained intuitively as follows. The CUB measures (with 90\% probability) the \emph{maximum} level of underfitting or overfitting for a regression model on \emph{any} sample from the DGP whereas the CCV only measures (with bias and a high variance) the \emph{average} level of underfitting or overfitting for the same regression model on \emph{some} sample. As a result, any deviation from the $\lambda_{CUB}$ causes CUB to increase by at least the same magnitude as CCV. Thus, as shown in the Figure~\ref{fig:sim4} plots, CUB has a more pronounced U-shape than CCV, making $\lambda_{CUB}$ more stable. (This is similar to the relation between the minimum eigenvalue of the Gram matrix and the variance of the regression coefficients in OLS, in which the larger the minimum eigenvalue, the more convex the linear space and the lower the variance of the regression coefficients.) In most of the Figure~\ref{fig:sim4} plots, the CCV and CAT are reasonably flat between $0$ and $0.20$, implying a variable $\lambda_{CV}$ across random seeds. Reflecting the conclusion from Table~\ref{table:simulation}, Figure~\ref{fig:sim4} shows empirically that $\lambda_{CUB}$ improves the sparsity-accuracy trade-off, similar to the effect of the one-standard-error rule. Of course, in addition to being a useful empirical tool, $\lambda_{CUB}$ has a sound theoretical basis---supported by the use of Rademacher complexity and the Orlicz-$\Psi_v$ norm---in contrast with heuristic decision rules for the lasso.

Figure~\ref{fig:sim5} reports the repeated simulations when $n/K = 200$. Since $n/K > p$ in this case, CCV, CAT and CUB exhibit a tendency to converge: the shapes and locations of the three curves are stable and similar while the high variance problem associated with $K$-CV is reduced. Overall in Figure~\ref{fig:sim5}, while $\lambda_{CUB} = 0.15 > \lambda_{CV} = 0.10$, the accuracy, sparsity and stability of variable selection for $\lambda_{CUB}$ and $\lambda_{CV}$ are quite similar, reflecting the conclusion from Table~\ref{table:simulation}.

Figure~\ref{fig:sim6} reports the repeated simulations when $n/K = 50$. In this case, $n/K < p$, and the location and shape of the CCV are more volatile across different random seeds than the CUB. The CUB is also close to or above CAT in most repetitions. The 3 exceptions (out of 20 plots) are when CUB intersects CAT (Figures~\ref{fig:sim611} and \ref{fig:sim615}) and below CAT (Figure \ref{fig:sim616}). This implies that, even in high-dimensional spaces, it is still feasible to establish a reliable and stable upper bound for the CV error. It also important to note that the U-shape of the CUB is more pronounced than the (almost flat) CCV in Figure~\ref{fig:sim6} when the sample size is much smaller than the previous cases. As a result, $\lambda_{CV}$ is more unstable across repetitions than $\lambda_{CUB}$, as shown in the plots. Thus, while both $\lambda_{CV}$ and $\lambda_{CUB}$ fail to retain the sparsity of the $n/K \geqslant p$ cases, $\lambda_{CUB}$ outperforms $\lambda_{CV}$ in terms of sparsity and stability. In conclusion, the improvement in variable-selection stability due to $\lambda_{CUB}$ noted in Table~\ref{table:simulation} also apparent from Figure~\ref{fig:sim4} to Figure~\ref{fig:sim6}.

\section{Conclusion}

In this paper, we derive a new upper bound for the cross-validation errors from any model in $K$-fold cross-validation. Use of Orlicz-$\Psi_\nu$ space in the analysis means the upper bound applies for different types of error processes. As well as serving as a confidence interval, the bound may be used as a criterion for model selection. In simulations, we show that the 90\% confidence band is tight for the lasso regression in high-dimensional space. The new bounds also have a similar shape to the 90\textsuperscript{th} percentiles for the distributions of the CV errors for the lasso. Bounds for $K$-CV based on other complexity measures, such as covering numbers or Mallow's $C_p$, can also be retrieved from our framework. Since the theoretical results are derived in a general framework, potentially they may be applied to many different $K$-CV-based learning algorithms and estimators (such as decision trees, nearest neighbors, the lasso, etc.) for hyperparameter tuning and model selection.

Two caveats of our approach are worth mentioning. For the upper bounds of the prediction error in each round, the concentration inequality based on the Orlicz-$\Psi_\nu$ norm quantifies the exponential concentration tendency without requiring a Lebesgue $p$-norm. However, because the classical techniques of time series typically rely on a well-defined Lebesgue $2$-norm, we need to assume that $T_q$ has a finite Lebesgue $2$-norm, weakening the power of Orlicz-$\Psi$ space. Also, independent blocks require a well-defined envelope function for the random variable of interest, which confines our analysis for $\beta$-mixing data to bounded losses.

\newpage
\appendix
\section{Mathematical proofs}
%
%
\begin{proof}
\textbf{Theorem}~\ref{thm:one_round_VC_bound}

To prove Theorem~\ref{thm:one_round_VC_bound}, it is equivalent to quantify the following probability:
\begin{equation}
  \mathrm{ Pr } \left[ \sup_{b \in \Lambda_l} \left\vert \mathcal{R}_{n_s} \left( b, Y_s^q, X_s^q \right) - \mathcal{ R }_{ n_t } \left( b, Y_t^q, X_t^q \right) \right\vert > \epsilon \right],
  \mbox{ for given } \epsilon \in \mathbb{R}^+.
  \label{eq:proof_1_0}
\end{equation}
Eq.~(\ref{eq:proof_1_0}) could be rewritten as
\begin{align}
  \phantom{ = } & ~ \mathrm{ Pr } \left[ \sup_{ b \in \Lambda_l } \left\vert \mathcal{ R }_{ n_s } \left( b, Y_s^q, X_s^q \right) - \mathcal{ R }_{ n_t } \left( b, Y_t^q, X_t^q \right) \right\vert > \epsilon \right] \notag \\
  = & ~ \mathrm{Pr} \left[ \sup_{b \in \Lambda_l} \left\vert \mathcal{R}_{n_s} \left( b, Y_s^q, X_s^q \right) - \mathcal{R} \left( b, Y, X \right) + \mathcal{R} \left( b, Y, X \right) - \mathcal{R}_{n_t} \left( b, Y_t^q, X_t^q \right) \right\vert > \epsilon \right].
  \label{eq:proof_1_1}
\end{align}
Due to the convexity of the norm, eq.~(\ref{eq:proof_1_1}) implies that
\begin{align}
  \phantom{ = } & ~ \mathrm{ Pr } \left[ \sup_{ b \in \Lambda_l } \left\vert \mathcal{ R }_{ n_s } \left( b, Y_s^q, X_s^q \right) - \mathcal{ R } \left( b, Y, X \right) + \mathcal{ R } \left( b, Y, X \right) - \mathcal{ R }_{ n_t } \left( b, Y_t^q, X_t^q \right) \right\vert > \epsilon \right] \notag \\
  \leqslant & ~ \mathrm{ Pr } \left[ \sup_{ b \in \Lambda_l } \left\vert \mathcal{ R }_{ n_s } \left( b, Y_s^q, X_s^q \right) - \mathcal{ R } \left( b, Y, X \right) \right\vert + \sup_{ b \in \Lambda_l } \left\vert \mathcal{ R } \left( b, Y, X \right) - \mathcal{ R }_{ n_t } \left( b, Y_t^q, X_t^q \right) \right\vert > \epsilon \right].
  \label{eq:proof_1_2}
\end{align}

\noindent
If we further define
\[
  \Phi_{ n_t } \mbox{ as } \sup_{ b \in \Lambda_l } \left\vert \mathcal{ R } \left( b, Y, X \right) - \mathcal{ R }_{ n_t } \left( b, Y_t^q, X_t^q \right) \right\vert
\]
and
\[
  \Phi_{ n_s } \mbox{ as } \sup_{ b \in \Lambda_l } \left\vert \mathcal{ R }_{ n_s } \left( b, Y_s^q, X_s^q \right) - \mathcal{ R } \left( b, Y, X \right) \right\vert,
\]

\noindent
the following derivation holds for eq.~(\ref{eq:proof_1_2}),

\noindent
\begin{align}
  \phantom{ = } & ~ \mathrm{ Pr } \left[ \sup_{ b \in \Lambda_l } \left\vert \mathcal{ R }_{ n_s } \left( b, Y_s^q, X_s^q \right) - \mathcal{ R } \left( b, Y, X \right) \right\vert + \sup_{ b \in \Lambda_l } \left\vert \mathcal{ R } \left( b, Y, X \right) - \mathcal{ R }_{ n_t } \left( b, Y_t^q, X_t^q \right) \right\vert > \epsilon \right] \notag \\
  = & ~ \mathrm{ Pr } \left[ \Phi_{ n_s } + \Phi_{ n_t } > \epsilon \right].
\end{align}

\smallskip
\noindent
1. \emph{ $\left\Vert \rho_j \right\Vert_{\Psi_2}$ is finite}

\smallskip
In Theorem~\ref{thm:one_round_VC_bound}, we defined the performance of the model class in each fold as
\begin{equation}
  \rho_j = \sup_{ b \in \Lambda_l } \left\vert \frac{ 1 } { n/K } \sum_{ i = 1 }^{ n/K } Q \left( b, y^{ j, i }, \mathrm{ x }^{ j, i } \right) - \mathcal{ R } \left( b, Y, X \right) \right\vert.
\end{equation}
We reserve $\rho_K$ for the fold used as the validation data. The following relation between $\Phi_{ n_t }$ and $\rho_j$ holds,
\begin{align}
  \Phi_{n_t} & \leqslant  \frac{ 1 } { K - 1 } \sum_{ j = 1 }^{ K - 1 } \sup_{ b \in \Lambda_l } \left\vert \frac{ 1 } { n/K } \sum_{ i = 1 }^{ n/K } Q \left( b, y^{ j, i }, \mathrm{ x }^{ j, i } \right) - \mathcal{ R } \left( b, Y, X \right) \right\vert   \\
  & = \frac{ 1 } { K - 1 } \sum_{ q = 1 }^{ K - 1 } \rho_j \, .
\end{align}
If $\left\Vert \rho_j \right\Vert_{ \Psi_2 }$ is finite, $\rho_j$ is subgaussian. Denote the variance of $\rho_j$ as $\sigma^2$. Define $\thickbar{ \epsilon } = \epsilon - 2 \cdot \mathbb{ E } \left( \rho_j \right)$. We also denote the vector $\mathbf{ a } = \left[ 1 / \left( K - 1 \right), \ldots, 1 / \left( K - 1 \right), 1 \right] \in \mathbb{R}^{ 1 \times K }$ as the weight and $\left\Vert \cdot \right\Vert_2$ as the $L^2$ norm. As a result, the Chernoff bound for the weighted sum of subgaussian variables implies that
\begin{align}
  \mathrm{ Pr } \left[ \Phi_{ n_t } + \Phi_{ n_s } > \epsilon \right] & \leqslant \mathrm{ Pr } \left[ \frac{ 1 } { K - 1 } \sum_{ j = 1 }^{ K - 1 } \rho_j + \rho_K > \epsilon \right] \notag \\
  & = \mathrm{ Pr } \left[ \frac{ 1 } { K - 1 } \sum_{ j = 1 }^{ K - 1 } \rho_j + \rho_K - 2 \cdot \mathbb{ E } \left( \rho_j \right) > \epsilon - 2 \cdot \mathbb{ E } \left( \rho_j \right) \right] \\
  & \leqslant \exp \left\{ - \frac{ \thickbar{ \epsilon }^2 } { 2 \cdot \left\Vert \mathbf{ a } \right\Vert_2^2 \cdot \sigma^2 } \right\} \\
  & =  \exp \left\{ - \frac{ \thickbar{ \epsilon }^2 \cdot \left( K - 1 \right)} { 2 \cdot \sigma^2 \cdot K } \right\}.
\end{align}
Since $K \in \left[ 2, n \right]$,
\begin{align}
  \mathrm{ Pr } \left[ \sup_{ b \in \Lambda_l } \left\{ \mathcal{ R }_{ n_s } \left( b, Y_s^q, X_s^q \right) - \mathcal{ R }_{ n_t } \left( b, Y_t^q, X_t^q \right) \right\} > \epsilon \right]
  & \leqslant \exp \left\{ -\frac{ \left( \epsilon - 2 \cdot \mathbb{ E } \left[ \rho_j \right] \right)^2 \cdot \left( K - 1 \right) }{ 2 \cdot \sigma^2 \cdot K } \right\} \\
  & \leqslant \exp \left\{ -\frac{ \left( \epsilon - 2 \cdot \mathbb{ E } \left[ \rho_j \right] \right)^2 }{ 4 \cdot \sigma^2 } \right\}
\end{align}
Set $ \varpi = \exp \left\{ - \left( \epsilon - 2 \cdot \mathbb{ E } \left[ \rho_j \right] \right)^2 / \left( 4 \cdot \sigma^2 \right) \right\}$. Since $\mathbb{ E } \left[ \rho_j \right] \leqslant \mathrm{ RC }_{ n/K } \left( \Lambda_l \right)$ \; \citep{mohri2009rademacher},
\begin{align}
  \epsilon & = 2 \cdot \mathbb{ E } \left[ \rho_j \right] + 2 \cdot \sigma \cdot \sqrt{ \log \left( \frac{ 1 } { \varpi } \right) } \\
  \phantom{ \epsilon } & \leqslant 2 \cdot \mathrm{ RC }_{ n/K } \left( \Lambda_l \right) + 2 \cdot \sigma \cdot \sqrt{ \log \left( \frac{ 1 } { \varpi } \right) },
\end{align}
where $\mathrm{ RC }_{ n/K } \left( \Lambda_l \right) $ is the Rademacher complexity of $\Lambda_l$ on each fold. Hence, $ \forall b \in \Lambda_l $, the following inequality holds with probability at least $ 1 - \varpi \in \left( 0 , 1  \right] $,
\begin{equation}
  \mathcal{ R }_{ n_s } \left( b, Y_s^q, X_s^q \right) \leqslant \mathcal{ R }_{ n_t } \left( b, Y_t^q, X_t^q \right) + 2 \cdot \mathrm{ RC }_{ n/K } \left( \Lambda_l \right) + 2 \cdot \sigma \cdot \sqrt{ \log \left( \frac{ 1 } { \varpi } \right) }.
\end{equation}
The definition of $\rho_i$ implies that
\begin{equation}
  \sigma^2 \leqslant ~ \mathbb{ E } \left[
    \left( ~
      \sup_{ b \in \Lambda_l } \left\vert \frac{ 1 }{ n/K } \sum_{ i = 1 }^{ n/K } \left[ Q \left( b, y^{ j, i }, \mathbf{ x }^{ j, i } \right) - \mathcal{ R } \left( b, Y, X \right) \right] \right\vert
    ~ \right)^2
  \right].
  \label{keystep_1_theorem_1}
\end{equation}
Since, $\forall x \in \mathbb{R}$, $\left[ \sup \left\vert x \right\vert \right]^2 = \left[ \sup \left\vert x \right\vert^2 \right] = \left[ \sup \left( x \right)^2 \right]$, eq.~(\ref{keystep_1_theorem_1}) implies that

\begin{align}
  \phantom{=} & ~ \mathbb{ E } \left[ \left( ~ \sup_{ b \in \Lambda_l } \left\vert \frac{ 1 }{ n/K } \sum_{ i = 1 }^{ n/K } \left[ Q \left( b, y^{ j, i }, \mathbf{ x }^{ j, i } \right) - \mathcal{ R } \left( b, Y, X \right) \right] \right\vert ~ \right)^2 \right] \notag \\
  = & ~ \mathbb{ E } \left[ \sup_{ b \in \Lambda_l } \left( \frac{ 1 }{ n/K } \sum_{ i = 1 }^{ n/K } \left[ Q \left( b, y^{ j, i }, \mathbf{ x }^{ j, i } \right) - \mathcal{ R } \left( b, Y, X \right) \right] \right)^2 \right].
  \label{keystep_2_theorem_1}
\end{align}
\text{A5} implies that we can `partially' interchange $\mathbb{ E } \left[ \cdot \right] $ and $\sup_{ b \in \Lambda_l } \left( \cdot \right)$ by multiplying a factor $\theta$. As a result, eq.~(\ref{keystep_2_theorem_1}) implies that

\begin{align}
  \phantom{=} & ~ \mathbb{ E } \left[ \sup_{ b \in \Lambda_l } \left( \frac{ 1 }{ n/K } \sum_{ i = 1 }^{ n/K } \left[ Q \left( b, y^{ j, i }, \mathbf{ x }^{ j, i } \right) - \mathcal{ R } \left( b, Y, X \right) \right] \right)^2 \right] \notag \\
  \leqslant & ~ \theta \cdot \sup_{ b \in \Lambda_l } \left\{ \mathbb{ E } \left[ ~ \left( \frac{ 1 }{ n/K } \sum_{ i = 1 }^{ n/K } \left[ Q \left( b, y^{ j, i }, \mathbf{ x }^{ j, i } \right) - \mathcal{ R } \left( b, Y, X \right) \right] ~\right)^2 \right] \right\}  \\
  = & ~ \theta \cdot \frac{ 1 }{ n/K } \cdot \sup_{ b \in \Lambda_l } \left\{ \mathrm{ var } \left[ Q \left( b, y^{ j, i }, \mathbf{ x }^{ j, i } \right) - \mathcal{ R } \left( b, Y, X \right) \right] \right\}
  \label{eqn:last_step_bounded} \\
  \leqslant & ~  \theta \cdot \frac{ B^2 }{ n/K } .
\end{align}
As a result, given the probability $ 1 - \varpi $,
\begin{equation}
  \mathcal{ R }_{ n_s } \left( b, Y_s^q, X_s^q \right) \leqslant \mathcal{ R }_{ n_t } \left( b, Y_t^q, X_t^q \right) + 2 \cdot \mathrm{ RC }_{ n/K } \left( \Lambda_l \right) + 2 \cdot B \cdot \sqrt{ \theta \cdot \frac{ \log \left( 1 / \varpi \right) } { n / K } }.
\end{equation}

\bigskip
%
\noindent
2. \emph{$ Q \left( \cdot \right) $ is bounded}

If $ Q \left( \cdot \right) $ is bounded by $M$, so is $ \rho_j $. Hence, we know that $ \rho_j $ is also subgaussian. Defining the variance proxy of $ \sup_{ b \in \Lambda_l } \left\{ \mathrm{ var } \left[ Q \left( b, y^{ j, i }, \mathbf{ x }^{ j, i } \right) - \mathcal{ R } \left( b, Y, X \right) \right] \right\}$ as $\tilde{ \sigma }^2$,  we know that $ \tilde{ \sigma }^2 < M^2 $. As a result, eq.~(\ref{eqn:last_step_bounded}) implies that
\begin{equation}
  \mathcal{R}_{n_s} \left( b, Y_s^q, X_s^q \right) \leqslant \mathcal{R}_{n_t} \left( b, Y_t^q, X_t^q \right) + 2 \cdot \mathrm{RC}_{n/K} \left( \Lambda_l \right) + 2 \cdot M \cdot \sqrt{ \theta \cdot \frac{ \log \left( 1 / \varpi \right) } { n / K } },
  \forall 1-\varpi \in \left[ 0,1 \right)
\end{equation}

\bigskip
\noindent
3. \emph{$\left\Vert \rho_j \right\Vert_{ \Psi_1 }$ is finite}

If $\left\Vert \rho_j \right\Vert_{ \Psi_1 }$ is finite, $ \rho_j $ is subexponential. As a result, the Bernstein-type inequality \citep{Lecue09tool,talagrand1994supremum} holds as follows
\begin{align}
  \mathrm{ Pr } \left[ \Phi_{ n_t } + \Phi_{ n_s } > \epsilon \right] & \leqslant \mathrm{ Pr } \left[ \frac{ 1 } { K - 1 } \sum_{ j = 1 }^{ K - 1 } \rho_j + \rho_K > \epsilon \right] \\
  & = \mathrm{ Pr } \left[ \frac{ 1 } { K - 1 } \sum_{ j = 1 }^{ K - 1 } \rho_j + \rho_K - 2 \cdot \mathbb{ E } \left( \rho_j \right) > \epsilon - 2 \cdot \mathbb{ E } \left( \rho_j \right) \right] \\
  & \leqslant 2 \cdot \exp \left\{ - c \cdot \min \left(
    \frac{ \thickbar{ \epsilon }^2 } { \left\Vert \rho_j \right\Vert_{ \Psi_1 }^2 \cdot \left\Vert \mathbf{a} \right\Vert_2^2 },
    \frac{ \thickbar{ \epsilon } } { \left\Vert \rho_j \right\Vert_{ \Psi_1 } \cdot \left\Vert \mathbf{a} \right\Vert_\infty }
    \right) \right\}  \\
  & \leqslant 2 \cdot \exp \left\{ - c \cdot \min \left(
    \frac{ \left( \epsilon - 2 \cdot \mathbb{ E } \left[ \rho_j \right] \right)^2 } { 2 \cdot \left\Vert \rho_j \right\Vert_{ \Psi_1 }^2 },
    \frac{ \epsilon - 2 \cdot \mathbb{ E } \left[ \rho_j \right] } { \left\Vert \rho_j \right\Vert_{ \Psi_1 } }
    \right) \right\}
\end{align}
If we set $ \left( \epsilon -2 \cdot \mathbb{ E } \left[ \rho_j \right] \right) / \left( 2 \cdot \left\Vert \rho_j \right\Vert_{ \Psi_1 } \right) = \tau > 0$,
\begin{align}
  \min \left(
  \frac{ \left( \epsilon - 2 \cdot \mathbb{ E } \left[ \rho_j \right] \right)^2 } { 2 \cdot \left\Vert \rho_j \right\Vert_{ \Psi_1 }^2 },
  \frac{ \epsilon -2 \cdot \mathbb{ E } \left[ \rho_j \right] } { \left\Vert \rho_j \right\Vert_{ \Psi_1 } }
  \right)
  = \left\{
  \begin{array}{lr}
    \left( \epsilon - 2 \cdot \mathbb{ E } \left[ \rho_j \right] \right) / \left\Vert \rho_j \right\Vert_{ \Psi_1 } ,   & \text{if } \tau > 1  \\
    \left( \epsilon - 2 \cdot \mathbb{ E } \left[ \rho_j \right] \right)^2 / \left( \sqrt{ 2 } \cdot  \left\Vert \rho_j \right\Vert_{ \Psi_1 } \right)^2 ,   & \text{if } \tau \leqslant 1
  \end{array}\right.
\end{align}
Hence,
\begin{align}
  \mathrm{Pr} \left[ \sup_{b \in \Lambda_l} \left\vert \mathcal{R}_{n_s} \left( b, X_s^q, Y_s^q \right) - \mathcal{R}_{n_t} \left( b, X_s^q, Y_s^q \right) \right\vert > \epsilon \right]
  \leqslant  \left\{
    \begin{array}{lr}
      2 \exp \left\{ - c \cdot \frac{ \epsilon - 2 \cdot \mathbb{E} \left[ \rho_j \right] } { \left\Vert \rho_j \right\Vert_{\Psi_1} }\right\} ,   & \text{if } \tau > 1  \\
      2 \exp \left\{ - c \cdot \frac{ \left( \epsilon - 2 \cdot \mathbb{E} \left[ \rho_j \right] \right)^2 } { 2 \cdot \left\Vert \rho_j \right\Vert_{\Psi_1}^2 } \right\} ,   & \text{if } \tau \leqslant 1
    \end{array}\right.
  \label{eqn:proof_heavy_tail}
\end{align}
If we set $ \varpi $ as the RHS of eq.~(\ref{eqn:proof_heavy_tail}), i.e.
\begin{align}
  \varpi = \left\{
    \begin{array}{lr}
      2 \exp \left\{ - c \cdot \frac{ \epsilon - 2 \cdot \mathbb{ E } \left[ \rho_j \right] } { \left\Vert \rho_j \right\Vert_{\Psi_1} } \right\} ,   & \text{if} ~~ \epsilon - 2 \cdot \mathbb{ E } \left[ \rho_j \right] >  2 \cdot \left\Vert \rho_j \right\Vert_{ \Psi_1 } > 0 \\
      2 \exp \left\{ - c \cdot \frac{ \left( \epsilon - 2 \cdot \mathbb{ E } \left[ \rho_j \right] \right)^2 } { 2 \cdot \left\Vert \rho_j \right\Vert_{ \Psi_1 }^2 } \right\} ,   & \text{if} ~~ 0 < \epsilon - 2 \cdot \mathbb{ E } \left[ \rho_j \right] \leqslant  2 \cdot \left\Vert \rho_j \right\Vert_{ \Psi_1 },
    \end{array}\right.
\end{align}
$\epsilon$ may be expressed as a function of $\varpi$,
\begin{align}
  \epsilon= &
    \begin{cases}
      2 \mathbb{ E } \left[ \rho_j \right] + \left\Vert \rho_j \right\Vert_{ \Psi_1 } \cdot \log \sqrt[c] { \left( 2 / \varpi \right) } , & \,\mathrm{if} ~~ \epsilon - 2 \cdot \mathbb{ E } \left[ \rho_j \right] >  2 \cdot \left\Vert \rho_j \right\Vert_{ \Psi_1 } > 0 \\
      2 \mathbb{ E } \left[ \rho_j \right] + \left\Vert \rho_j \right\Vert_{ \Psi_1 } \cdot \left( 2 \cdot \log \sqrt[c] { \left( 2 / \varpi \right) } \right)^{ \frac{ 1 }{ 2 } } , & \mathrm{if} ~~ 0 < \epsilon - 2 \cdot \mathbb{ E } \left[ \rho_j \right] \leqslant  2 \cdot \left\Vert \rho_j \right\Vert_{ \Psi_1 }
    \end{cases}
\end{align}
Hence, $ \forall b \in \Lambda_l $, the following inequality holds with probability at least $ 1 - \varpi \in \left( 0 , 1 \right]$
\begin{align}
  \mathcal{ R }_{ n_s } \left( b, Y_s^q, X_s^q \right) \leqslant \mathcal{ R }_{ n_t } \left( b, Y_t^q, X_t^q \right)+ 2 \cdot \mathrm{ RC }_{ n/K } \left( \Lambda_l \right) +  \left\Vert \rho_j \right\Vert_{ \Psi_1 } \cdot \log \sqrt[c] { \left( 2 / \varpi \right) } , \notag \\
  \mathrm{if} ~~ 2 \exp \left\{ -2 c \right\} > \varpi  >  0 \\
  \mathcal{ R }_{ n_s } \left( b, Y_s^q, X_s^q \right) \leqslant \mathcal{ R }_{ n_t } \left( b, Y_t^q, X_t^q \right) + 2 \cdot \mathrm{ RC }_{ n/K } \left( \Lambda_l \right) + \left\Vert \rho_j \right\Vert_{ \Psi_1 } \cdot \left( 2 \cdot \log \sqrt[c] { \left( 2 / \varpi \right) } \right)^{ \frac{ 1 }{ 2 } } , \notag \\
  \mathrm{if} ~~ 2 \exp \left\{ -2 c \right\} \leqslant \varpi \leqslant 1
\end{align}
\smallskip
\end{proof}
\bigskip

%
\begin{proof}
\textbf{Lemma}~\ref{lem:Cheby_ineq}

The Chebyshev inequality shows that
\begin{equation}
  \mathrm{ Pr } \left( \vert \overline{ W } - \mathbb{ E } \left( W \right) \vert \leqslant \epsilon \right)
  \geqslant 1 - \frac{ \mathrm{ var } \left( \overline{ W } \right) }{ \epsilon^2 }.
\end{equation}
The variance of $\overline{ W }$ may be expressed as
\begin{eqnarray}
  \mathrm{ var } \left( \overline{ W } \right) & = & \mathbb{ E } \left( \overline{ W }^2 \right) - \mathbb{ E } \left( \overline{ W } \right)^2 \\
  & = & \frac{ 1 }{ n^2 } \, \mathbb{ E } \left[ \left( \sum_{ i = 1 }^{ n } W_i \right)^2 \right] - \mathbb{ E } \left( W \right)^2 \\
  & = & \frac{ 1 }{ n^2 } \sum_{ j = 1 }^{ n } \sum_{ i = 1 }^{ n } \mathbb{ E } \left[ W_i \cdot W_j \right] - \mathbb{E} \left( W \right)^2.
\end{eqnarray}
If we define the covariance $\gamma_{ i - j } := \mathrm{ cov } \left( W_i,\; W_j \right)$ and the variance as $ \gamma_0 := \mathrm{ var } \left( W \right)$, this may be expressed as

\begin{eqnarray}
  \frac{ 1 }{ n^2 } \sum_{ j = 1 }^{ n } \sum_{ i = 1 }^{ n } \mathbb{ E } \left[ W_i \cdot W_j \right] - \mathbb{ E } \left( W \right)^2
  & = & \frac{ 1 }{ n^2 } \sum_{ j = 1 }^{ n } \sum_{ i = 1 }^{ n } \gamma_{ i - j } + \mathbb{E} \left( W \right)^2 - \mathbb{E} \left( W \right)^2 \\
  & = & \frac{ 1 }{ n^2 } \sum_{ j = 1 }^{ n } \; \sum_{ l = 1 - j }^{ n - j } \gamma_{ l }\\
  & = & \frac{ \gamma_0 } { n } + \frac{ 2 }{ n^2 } \sum_{ j = 1 }^{ n - 1 }
        \sum_{ l = 1 }^{ n - j } \gamma_{ l } \\
  & \leqslant & \frac{ \gamma_0 }{ n } + \frac{ 2 }{ n^2 } \sum_{ j = 1 }^{ n - 1 }
    \sum_{ l = 1 }^{ n - j }\left\vert \gamma_{ l } \right\vert .
\end{eqnarray}
Define $\gamma_0 \cdot V_n \left[ W \right] := \sum_{ l = 1 }^{ n - 1 } \left\vert \gamma_l \right\vert $, then
\begin{eqnarray}
  \frac{ \gamma_0 }{ n } + \frac{ 2 }{ n^2 } \sum_{ j = 1 }^{ n } \sum_{ l = 1 }^{ n - j } \left\vert \gamma_{ l } \right\vert
    & \leqslant & \frac{ \gamma_0 }{ n } + \frac{ 2 V_n \left[ W \right] \cdot \left( n - 1 \right) \cdot \gamma_0 }{ n^2 } \\
  & \leqslant & \frac{ \gamma_0 }{ n } \cdot \left( 1 + 2V_n \left[ W \right] \right).
\end{eqnarray}
Hence the Chebyshev inequality may be generalized to a stationery stochastic process as
\begin{equation}
  \mathrm{ Pr } \left( \left\vert \overline{ W } - \mathbb{E} \left( W \right) \right\vert \leqslant \epsilon \right)
  \geqslant 1 - \frac{ \gamma_0 }{ \epsilon^2 n } \cdot \left( 1 + 2 V_n \left[ W \right] \right).
  \label{eqn:semi_cheby}
\end{equation}
\end{proof}
\bigskip


\begin{proof}
\textbf{Theorem}~\ref{thm:convoluted_VC_bound}

We are going to prove Theorem~\ref{thm:convoluted_VC_bound} for two cases

\noindent
1. $ \left\Vert \rho_j \right\Vert_{ \Psi_2 } \leqslant \infty $.

Since we have already defined
\begin{align}
  T_q  & = U_q - \mathbb{ E } \left[ U_q \right] \notag \\
  & = \sup_{ b \in \Lambda_l } \; \left\vert \mathcal{ R }_{ n_s } \left( b, Y_s^q, X_s^q \right) - \mathcal{ R }_{ n_t } \left( b, Y_t^q, X_t^q \right) \right\vert \notag \\
  & \phantom{ = } - \mathbb{ E } \left[ \sup_{ b \in \Lambda_l } \; \left\vert \mathcal{ R }_{ n_s } \left( b, Y_s^q, X_s^q \right) - \mathcal{ R }_{ n_t } \left( b, Y_t^q, X_t^q \right) \right\vert \right],
\end{align}
Lemma~\ref{lem:Cheby_ineq} implies that
\begin{align}
  \mathrm{Pr} \left\{ \frac{1}{K} \sum^{K}_{q=1} T_q \geqslant \varsigma \right\}
  \leqslant
  \frac{ \gamma_0 \left[ T_q \right] }{ \varsigma^2 K} \cdot \left( 1 + 2 V_K \left[ T_q \right] \right)
\end{align}
Based on Theorem~\ref{thm:one_round_VC_bound}, we let $\varsigma = 2 \cdot B \cdot \sqrt{ \theta \cdot \frac{ \log \left( 1/\varpi \right) } { n / K } }$ when $\left\Vert \rho_j \right\Vert_{\Psi_2}$ is well defined. The definition of $\gamma_0$ implies that
\begin{align}
  \gamma_0 \left[ T_q \right]
  & = \mathbb{E} \left[ \left( U_q \right)^2 \right] - \left( \mathbb{E} \left[ U_q \right] \right)^2 \notag \\
  & = \mathbb{E} \left[ ~ \left( \sup_{ b \in \Lambda_l} \; \left\vert
  \mathcal{R}_{n_s} \left( b, Y_s^q, X_s^q\right) -
  \mathcal{R}_{n_t} \left( b, Y_t^q, X_t^q \right) \right\vert ~ \right)^2 \right]
  \label{thm2_key_step_1} \\
  & \phantom{=} ~ - \left( \mathbb{E} \left[\sup_{ b \in \Lambda_l} \; \left\vert
  \mathcal{R}_{n_s}\left( b, Y_s^q, X_s^q\right) -
  \mathcal{R}_{n_t} \left( b, Y_t^q, X_t^q \right) \right\vert \right] \right)^2.
  \notag
\end{align}
Since, $\forall x \in \mathbb{R}$, $\left[ \sup \left\vert x \right\vert \right]^2 = \left[ \sup \left\vert x \right\vert^2 \right] = \left[ \sup \left( x \right)^2 \right]$, the following derivation holds for the first right-hand side term of eq.~(\ref{thm2_key_step_1}).
\begin{align}
  & ~ \phantom{=} \mathbb{E} \left[ \left( ~ \sup_{ b \in \Lambda_l} \; \left\vert
    \mathcal{R}_{n_s} \left( b, Y_s^q, X_s^q \right) -
    \mathcal{R}_{n_t} \left( b, Y_t^q, X_t^q \right) \right\vert ~ \right)^2 \right] \notag \\
  & = \mathbb{E} \left[ \left( ~ \sup_{ b \in \Lambda_l} \; \left\vert
    \mathcal{R}_{n_s} \left( b, Y_s^q, X_s^q \right) -
    \mathcal{R} \left( b, Y, X \right) +
    \mathcal{R} \left( b, Y, X \right) -
    \mathcal{R}_{n_t} \left( b, Y_t^q, X_t^q \right) \right\vert ~ \right)^2 \right] \\
  & \leqslant \mathbb{E} \left[ \left( ~ \sup_{ b \in \Lambda_l} \; \left\vert
    \mathcal{R}_{n_s} \left( b, Y_s^q, X_s^q \right) -
    \mathcal{R} \left( b, Y, X \right) \right\vert +
    \sup_{ b \in \Lambda_l} \left\vert  \mathcal{R} \left( b, Y, X \right) -
    \mathcal{R}_{n_t} \left( b, Y_t^q, X_t^q \right)\right\vert ~ \right)^2 \right] \\
  & = \mathbb{E} \left[ \left( ~ \sup_{ b \in \Lambda_l} \; \left\vert
    \mathcal{R}_{n_s} \left( b, Y_s^q, X_s^q \right) -
    \mathcal{R} \left( b, Y, X \right)\right\vert ~ \right)^2 \right] +
    \mathbb{E} \left[ \left( ~ \sup_{ b \in \Lambda_l} \left\vert \mathcal{R} \left( b, Y, X \right) - \mathcal{R}_{n_t} \left( b, Y_t^q, X_t^q \right) \right\vert ~ \right)^2 \right]
    \label{thm2_key_step_2} \\
  & \phantom{=} ~ + 2 \cdot \mathbb{E} \left[\sup_{ b \in \Lambda_l} \; \left\vert
    \mathcal{R}_{n_s} \left( b, Y_s^q, X_s^q \right) -
    \mathcal{R} \left( b, Y, X \right) \right\vert \cdot
    \sup_{ b \in \Lambda_l} \left\vert \mathcal{R} \left( b, Y, X \right) -
    \mathcal{R}_{n_t} \left( b, Y_t^q, X_t^q \right) \right\vert \right]
    \notag
\end{align}
Consider first the third right-hand side term of eq.~(\ref{thm2_key_step_2}). Denoting
\[
  \sup_{ b \in \Lambda_l} \; \left\vert \mathcal{R}_{n_s} \left( b, Y_s^q, X_s^q \right) - \mathcal{R} \left( b, Y, X \right) \right\vert \mbox{ as } f_1
\]
and
\[
  \sup_{ b \in \Lambda_l} \;  \left\vert \mathcal{R}_{n_t} \left( b, Y_s^q, X_s^q \right) - \mathcal{R} \left( b, Y, X \right) \right\vert \mbox{ as } f_2,
\]
the Holder inequality implies $ \left\Vert \langle f_1, f_2 \rangle \right\Vert_1 \leqslant \left\Vert f_1 \right\Vert_2 \cdot \left\Vert f_2 \right\Vert_2$, i.e.,
\begin{align}
  \phantom{\leqslant} & \mathbb{E} \left[\sup_{ b \in \Lambda_l}
  \left\vert \mathcal{R}_{n_s} \left( b, Y_s^q, X_s^q \right) -
  \mathcal{R} \left( b, Y, X \right)\right\vert \cdot
  \sup_{ b \in \Lambda_l} \left\vert \mathcal{R} \left( b, Y, X \right) -
  \mathcal{R}_{n_t} \left( b, Y_t^q, X_t^q \right)\right\vert \right] \notag \\
  \leqslant & \left( \mathbb{E} \left[ \left( ~ \sup_{ b \in \Lambda_l}
  \left\vert \mathcal{R}_{n_s} \left( b, Y_s^q, X_s^q \right) -
  \mathcal{R} \left( b, Y, X \right) \right\vert ~ \right)^2 \right] \right)^{1/2}
  \left( \mathbb{E} \left[ \left( ~ \sup_{ b \in \Lambda_l} \left\vert \mathcal{R} \left( b, Y, X \right) - \mathcal{R}_{n_t} \left( b, Y_t^q, X_t^q \right) \right\vert ~ \right)^2 \right] \right)^{1/2}
  \label{thm2_key_step_3}
\end{align}
Since, $\forall x \in \mathbb{R}$, $\left[ \sup \left\vert x \right\vert \right]^2 = \left[ \sup \left\vert x \right\vert^2 \right] = \left[ \sup \left( x \right)^2 \right]$,
\begin{align}
  \phantom{\leqslant} & ~ \left( \mathbb{E} \left[ \left( ~ \sup_{ b \in \Lambda_l} \;
  \left\vert\mathcal{R}_{n_s} \left( b, Y_s^q, X_s^q \right) -
  \mathcal{R} \left( b, Y, X \right) \right\vert \right)^2 \right] \right)^{1/2}
  \left( \mathbb{E} \left[ \left( ~ \sup_{ b \in \Lambda_l} \left\vert \mathcal{R} \left( b, Y, X \right) - \mathcal{R}_{n_t} \left( b, Y_t^q, X_t^q \right) \right\vert \right)^2 \right] \right)^{1/2} \notag \\
  = & ~ \left(\mathbb{E} \left[\sup_{ b \in \Lambda_l} \;
  \left(\mathcal{R}_{n_s} \left( b, Y_s^q, X_s^q \right) -
  \mathcal{R} \left( b, Y, X \right)\right)^2 \right] \right)^{1/2}
  \left( \mathbb{E} \left[\sup_{ b \in \Lambda_l} \left( \mathcal{R} \left( b, Y, X \right) - \mathcal{R}_{n_t} \left( b, Y_t^q, X_t^q \right)\right)^2 \right] \right)^{1/2}
  \label{thm2_key_step_4}
\end{align}
Based on \textbf{A5}, we can interchange the order of $\mathbb{ E } \left[ \cdot \right] $ and $\sup_{ b \in \Lambda_l } \left\{ \cdot \right\}$ in eq.~(\ref{thm2_key_step_4}), implying that
\begin{align}
  \phantom{\leqslant} & ~ \left( \mathbb{E} \left[ \sup_{ b \in \Lambda_l} \;
  \left( \mathcal{R}_{n_s} \left( b, Y_s^q, X_s^q \right) -
  \mathcal{R} \left( b, Y, X \right) \right)^2 \right] \right)^{1/2}
  \left( \mathbb{E} \left[ \sup_{ b \in \Lambda_l} \left( \mathcal{R} \left( b, Y, X \right) - \mathcal{R}_{n_t} \left( b, Y_t^q, X_t^q \right) \right)^2 \right] \right)^{1/2} \notag \\
  \leqslant & ~ \theta \cdot \left( \sup_{ b \in \Lambda_l} \left\{\mathbb{E}
  \left[ \left( \mathcal{R}_{n_s} \left( b, Y_s^q, X_s^q \right) -
  \mathcal{R} \left( b, Y, X \right) \right)^2 \right] \right\} \cdot
  \sup_{ b \in \Lambda_l} \left\{ \mathbb{E} \left[ \left( \mathcal{R}_{n_t} \left( b, Y_s^q, X_s^q \right) -
  \mathcal{R} \left( b, Y, X \right) \right)^2 \right] \right\} \right)^{1/2} \\
  = & ~ \theta \cdot \left( \sup_{ b \in \Lambda_l} \left\{\mathrm{var}
  \left[\mathcal{R}_{n_s} \left( b, Y_s^q, X_s^q \right) -
  \mathcal{R} \left( b, Y, X \right)\right] \right\} \cdot
  \sup_{ b \in \Lambda_l} \left\{\mathrm{var} \left[\mathcal{R}_{n_t} \left( b, Y_s^q, X_s^q \right) -
  \mathcal{R} \left( b, Y, X \right)\right] \right\} \right)^{1/2}
  \label{final_key_step_thm_2}
\end{align}
Since, given any $b \in \Lambda_l$, $\mathrm{var} \left[ \mathcal{R}_{n_s} \left( b, Y_s^q, X_s^q \right) - \mathcal{R} \left( b, Y, X \right) \right]$ does not increase when sample size increases, eq.~(\ref{final_key_step_thm_2}) implies that
\begin{align}
  \phantom{=} & ~ \theta \cdot \left( \sup_{ b \in \Lambda_l} \left\{\mathrm{var}
  \left[\mathcal{R}_{n_s} \left( b, Y_s^q, X_s^q \right) -
  \mathcal{R} \left( b, Y, X \right)\right] \right\} \cdot
  \sup_{ b \in \Lambda_l} \left\{\mathrm{var} \left[\mathcal{R}_{n_t} \left( b, Y_s^q, X_s^q \right) -
  \mathcal{R} \left( b, Y, X \right)\right] \right\} \right)^{1/2} \\
  \leqslant & ~ \theta \cdot \sup_{ b \in \Lambda_l} \left\{\mathrm{var}
  \left[\mathcal{R}_{n_s} \left( b, Y_s^q, X_s^q \right) -
  \mathcal{R} \left( b, Y, X \right)\right] \right\} \\
  = & ~ \frac{\theta}{ n/K } \cdot
  \sup_{ b \in \Lambda_l} \left\{ \mathrm{var} \left[ Q \left( b, y^{i,j}, \mathbf{x}^{i,j} \right) - \mathcal{R} \left( b, Y, X \right) \right] \right\}
  \label{final_key_step_thm_5}
\end{align}
Similarly, $\mathbb{E} \left[ \cdot \right]$ and $\sup_{ b \in \Lambda_l} \left\{ \right\}$ in the first two right-hand side terms of eq.~(\ref{thm2_key_step_2}) can also be interchanged, implying
\begin{align}
  \phantom{\leqslant} & \mathbb{E} \left[ \left( ~ \sup_{ b \in \Lambda_l}
  \left\vert \mathcal{R}_{n_s} \left( b, Y_s^q, X_s^q \right) -
  \mathcal{R} \left( b, Y, X \right) \right\vert \right)^2 \right] +
  \mathbb{E} \left[ \left( ~ \sup_{ b \in \Lambda_l} \left\vert \mathcal{R} \left( b, Y, X \right) - \mathcal{R}_{n_t} \left( b, Y_t^q, X_t^q \right)\right\vert \right)^2 \right] \notag \\
  \leqslant & ~ \theta \cdot \sup_{ b \in \Lambda_l}
  \left\{ \mathrm{var} \left[\mathcal{R}_{n_s} \left( b, Y_s^q, X_s^q \right) -
  \mathcal{R} \left( b, Y, X \right)\right] \right\} +
  \theta \cdot \sup_{ b \in \Lambda_l} \left\{ \mathrm{var} \left[ \mathcal{R} \left( b, Y, X \right) -
  \mathcal{R}_{n_t} \left( b, Y_t^q, X_t^q \right)\right] \right\}
\end{align}
As a result, the following relation holds for the first right-hand side term of eq.~(\ref{thm2_key_step_1})
\begin{align}
  \phantom{\leqslant} & ~ \mathbb{E} \left[ \left( ~ \sup_{ b \in \Lambda_l}
  \left\vert\mathcal{R}_{n_s} \left( b, Y_s^q, X_s^q \right) -
  \mathcal{R}_{n_t} \left( b, Y_t^q, X_t^q \right) \right\vert ~ \right)^2 \right] \notag \\
  \leqslant & ~ \theta \cdot \sup_{ b \in \Lambda_l}
  \left\{ \mathrm{var} \left[\mathcal{R}_{n_s} \left( b, Y_s^q, X_s^q \right) -
  \mathcal{R} \left( b, Y, X \right)\right] \right\} +
  \theta \cdot \sup_{ b \in \Lambda_l} \left\{ \mathrm{var} \left[ \mathcal{R} \left( b, Y, X \right) - \mathcal{R}_{n_t} \left( b, Y_t^q, X_t^q \right)\right] \right\}  \\
  \phantom{\leqslant} & ~ + \frac{2 \cdot \theta}{ n/K } \cdot \sup_{ b \in \Lambda_l}
  \left\{ \mathrm{var} \left[ Q \left( b, y^{i,j}, \mathbf{x}^{i,j} \right) -
  \mathcal{R} \left( b, Y, X \right) \right] \right\} \notag \\
  = & ~ \theta \cdot \left( \frac{1}{n/K} + \frac{K}{n(K-1)} \right) \sup_{ b \in \Lambda_l}
  \left\{\mathrm{var} \left[ Q \left( b, y^{i,j}, \mathbf{x}^{i,j} \right) -
  \mathcal{R} \left( b, Y, X \right)\right] \right\} \label{key_result_thm2_inline3} \\
  \phantom{\leqslant} & ~ + \frac{2 \cdot \theta}{ n/K } \cdot \sup_{ b \in \Lambda_l}
  \left\{ \mathrm{var} \left[ Q \left( b, y^{i,j}, \mathbf{x}^{i,j} \right) -
  \mathcal{R} \left( b, Y, X \right) \right] \right\} \notag
\end{align}
Now focus on the second right-hand side term of eq.~(\ref{thm2_key_step_1}). Jensen inequality implies
\begin{equation}
  \mathbb{E} \left[\sup_{ b \in \Lambda_l}
  \left\vert\mathcal{R}_{n_s} \left( b, Y_s^q, X_s^q \right) -
  \mathcal{R}_{n_t} \left( b, Y_t^q, X_t^q \right)\right\vert  \right]
  \geqslant \sup_{ b \in \Lambda_l}
  \left\{ \mathbb{E} \left[  \left\vert \mathcal{R}_{n_s} \left( b, Y_s^q, X_s^q \right) - \mathcal{R}_{n_t} \left( b, Y_t^q, X_t^q \right)\right\vert \right] \right\}
\end{equation}
which further implies
\begin{align}
  \phantom{=} & ~ - \left( \mathbb{E} \left[\sup_{ b \in \Lambda_l}
  \left\vert\mathcal{R}_{n_s} \left( b, Y_s^q, X_s^q \right) -
  \mathcal{R}_{n_t} \left( b, Y_t^q, X_t^q \right)\right\vert  \right] \right)^2 \\
  \leqslant & ~ - \left( \sup_{ b \in \Lambda_l}
  \left\{ \mathbb{E} \left[  \left\vert \mathcal{R}_{n_s} \left( b, Y_s^q, X_s^q \right) - \mathcal{R}_{n_t} \left( b, Y_t^q, X_t^q \right)\right\vert \right] \right\} \right)^2 \notag
\end{align}
Since the $L^2$ norm is less than or equal to the $L^1$ norm,
\begin{align}
  \phantom{=} & ~ - \left( \sup_{ b \in \Lambda_l}
  \left\{ \mathbb{E} \left[ \left\vert\mathcal{R}_{n_s} \left( b, Y_s^q, X_s^q \right) -
  \mathcal{R}_{n_t} \left( b, Y_t^q, X_t^q \right)\right\vert \right]  \right\} \right)^2 \notag \\
  \leqslant & ~ - \left( \sup_{ b \in \Lambda_l}
  \left\{ \mathbb{E} \left[ \left( ~ \mathcal{R}_{n_s} \left( b, Y_s^q, X_s^q \right) -
  \mathcal{R}_{n_t} \left( b, Y_t^q, X_t^q \right) ~ \right)^2 \right] \right\}^{1/2} \right)^2 \\
  = & ~ - \sup_{ b \in \Lambda_l}
  \left\{ \mathbb{E} \left[ \left(\mathcal{R}_{n_s} \left( b, Y_s^q, X_s^q \right) -
  \mathcal{R}_{n_t} \left( b, Y_t^q, X_t^q \right)\right)^2 \right] \right\} \\
  = & ~ - \sup_{ b \in \Lambda_l}
  \left\{ \mathrm{var} \left(\mathcal{R}_{n_s} \left( b, Y_s^q, X_s^q \right) -
  \mathcal{R}_{n_t} \left( b, Y_t^q, X_t^q \right)\right) \right\} \\
  = & ~ - \sup_{ b \in \Lambda_l}
  \left\{ \mathrm{var} \left(\mathcal{R}_{n_s} \left( b, Y_s^q, X_s^q \right) -
  \mathcal{R} \left( b, Y, X \right) +
  \mathcal{R} \left( b, Y, X \right) -
  \mathcal{R}_{n_t} \left( b, Y_t^q, X_t^q \right)\right) \right\}
\end{align}
Since the data are i.i.d., $\mathcal{R}_{n_s} \left( b, Y_s^q, X_s^q \right)$ is independent from $\mathcal{R}_{n_t} \left( b, Y_s^q, X_s^q \right)$, implying that
\begin{align}
  \phantom{=} & ~ - \sup_{ b \in \Lambda_l}
  \left\{ \mathrm{var} \left(\mathcal{R}_{n_s} \left( b, Y_s^q, X_s^q \right) -
  \mathcal{R} \left( b, Y, X \right) +
  \mathcal{R} \left( b, Y, X \right) -
  \mathcal{R}_{n_t} \left( b, Y_t^q, X_t^q \right)\right) \right\} \notag \\
  = & ~ - \sup_{ b \in \Lambda_l}
  \left\{ \mathrm{var} \left(\mathcal{R}_{n_s} \left( b, Y_s^q, X_s^q \right) -
  \mathcal{R} \left( b, Y, X \right)\right) +
    \mathrm{var} \left(\mathcal{R} \left( b, Y, X \right) -
    \mathcal{R}_{n_t} \left( b, Y_t^q, X_t^q \right)\right) \right\} \\
  = & ~ - \left( \frac{1}{n/K} + \frac{K}{n(K-1)} \right)\sup_{ b \in \Lambda_l}
  \left\{ \mathrm{var} \left(Q \left( b, y^{i,j}, \mathbf{x}^{i,j} \right) -
  \mathcal{R} \left( b, Y, X \right)\right) \right\}
  \label{key_result_thm2_inline2}
\end{align}
As a result, eq.~(\ref{key_result_thm2_inline2}) and~(\ref{key_result_thm2_inline3}) imply that
\begin{align}
  \gamma_0 \left[ T_q \right] & \leqslant \left( \theta - 1 \right) \cdot
  \left( \frac{1}{n/K} + \frac{K}{n(K-1)} \right) \sup_{ b \in \Lambda_l}
  \left\{ \mathrm{var} \left[ Q \left( b, y^{i,j}, \mathbf{x}^{i,j} \right) -
  \mathcal{R} \left( b, Y, X \right)\right] \right\} \label{key_result_thm2_inline} \\
  & \phantom{\leqslant} ~ + \frac{2 \cdot \theta}{ n/K } \cdot
  \sup_{ b \in \Lambda_l} \left\{ \mathrm{var} \left[ Q \left( b, y^{i,j}, \mathbf{x}^{i,j} \right) - \mathcal{R} \left( b, Y, X \right) \right] \right\}
  \notag
\end{align}

\noindent
Based on the definition of $\varsigma$ in Theorem~\ref{thm:one_round_VC_bound} and eq.~(\ref{key_result_thm2_inline}),
\begin{align}
  \phantom{=} & ~ \frac{ \gamma_0 \left[ T_q \right] }{ \varsigma^2 K} \cdot \left( 1 + 2 V_K \left[ T_q \right] \right) \notag \\
  \leqslant & ~ \frac{ \left(\theta - 1 \right) \cdot \left( ~ 1 / \left( n/K \right) + K / \left[ n \cdot \left( K - 1 \right) \right] ~ \right) + 2 \cdot \theta / \left( n / K \right) }
  { \left( 4 \cdot \theta \cdot K \right) / \left( n/K \right)}  \\
  \phantom{\leqslant} & \cdot \frac{ \sup_{ b \in \Lambda_l} \left\{ \mathrm{var} \left[ Q \left( b, y^{i,j}, \mathbf{x}^{i,j} \right) - \mathcal{R} \left( b, Y, X \right) \right] \right\} \cdot
  \left( 1 + 2 V_K \left[ T_q \right] \right) }
  { \sup_{ b \in \Lambda_l} \left\{ \mathrm{var} \left[ Q \left( b, y^{i,j}, \mathbf{x}^{i,j} \right) - \mathcal{R} \left( b, Y, X \right) \right] \right\} \cdot \log \left( 1 / \varpi \right)} \notag \\
  \leqslant & ~ \frac{ \left(\theta - 1 \right) \cdot \left( ~ 1 + 1 / \left( K - 1 \right) ~ \right) + 2 \cdot \theta }
  { \left( 4 \cdot \theta \cdot K \right) \cdot  \log \left( 1 / \varpi \right)} \cdot \left( 1 + 2 V_K \left[ T_q \right] \right)
\end{align}
Since $1 + 1 / \left( K - 1 \right)$ decreases as $K$ increaes and $K \geqslant 2$,
\begin{align}
  \frac{ \gamma_0 \left[ T_q \right] }{ \varsigma^2 K} \cdot \left( 1 + 2 V_K \left[ T_q \right] \right)
  \leqslant & ~ \frac{ \left(\theta - 1 \right) \cdot \left( ~ 1 + 1 / \left( K - 1 \right) ~ \right) + 2 \cdot \theta }
  { \left( 4 \cdot \theta \cdot K \right) \cdot  \log \left( 1 / \varpi \right)} \cdot \left( 1 + 2 V_K \left[ T_q \right] \right) \notag \\
  \leqslant & ~ \frac{ \left[ 2 \cdot \left(\theta - 1 \right) + 2 \cdot \theta \right] }
  { \left( 4 \cdot \theta \cdot K \right) \cdot  \log \left( 1 / \varpi \right)} \cdot \left( 1 + 2 V_K \left[ T_q \right] \right) \\
  \leqslant & ~ \frac{ \left(\theta - 1 \right)/\theta + 1  }
  { 2 \cdot K \cdot  \log \left( 1 / \varpi \right)} \cdot \left( 1 + 2 V_K \left[ T_q \right] \right)
\end{align}
\noindent
Since $\mathbb{E} \left[ U_q \right] = \mathbb{E} \left[ \sup_{b \in \Lambda_l} \left\vert \mathcal{R}_{n_s} \left( b, Y_s^q, X_s^q \right) - \mathcal{R}_{n_t} \left( b, Y_t^q, X_t^q \right)\right\vert \right] \leqslant 2 \cdot \mathrm{RC} \left( \Lambda_l, n, K \right)$, $\forall b \in \Lambda_l$ the following holds,
\begin{align}
  \phantom{=} & \mathrm{Pr} \left\{ \frac{1}{K} \sum^{K}_{q=1} T_q \leqslant \varsigma \right\} \notag \\
  \geqslant & ~ \mathrm{Pr}\left\{ \frac{1}{K} \sum^{K}_{q = 1}
    \mathcal{R}_{n_s} \left( b, Y_s^q, X_s^q \right)
    \leqslant \frac{ 1 } { K } \sum^{ K }_{ q = 1 }\mathcal{R}_{n_t} \left( b, Y_t^q, X_t^q \right) +
    \mathbb{E} \left[ U_q \right] +
    \varsigma \right\} \\
  \geqslant & ~ \mathrm{Pr}\left\{ \frac{1}{K} \sum^{K}_{q = 1}
    \mathcal{R}_{n_s} \left( b, Y_s^q, X_s^q \right)
    \leqslant \frac{ 1 } { K } \sum^{ K }_{ q = 1 }\mathcal{R}_{n_t} \left( b, Y_t^q, X_t^q \right) +
    2 \cdot \mathrm{RC} \left( \Lambda_l, n, K \right) +
    \varsigma \right\} \\
  \geqslant & \left( 1 - \frac{ \left(\theta - 1 \right)/\theta + 1  }
    { \left( 2 \cdot K \right) \cdot  \log \left( 1 / \varpi \right)}
    \cdot \left( 1 + 2 V_K \left[ T_q \right] \right) \right)^+.
\end{align}
This completes the proof 1.
\bigskip


\noindent
2. $ \left\Vert \rho_j \right\Vert_{\Psi_1} \leqslant \infty $.

We assume $ \left\Vert \rho_j \right\Vert_{\Psi_1}$ and $\mathrm{var} \left[ T_q \right]$ are well defined. If $\varpi \in \left[ 2 \exp \left\{ -2c \right\}, 1 \right]$, based on the definition of $\varsigma$ and the fact that
\begin{align}
  \frac{ \gamma_0 \left[ T_q \right] }{ \varsigma^2 K} \cdot \left( 1 + 2 V_K \left[ T_q \right] \right)
  = & ~
  \frac{ 1 + 2 V_K \left[ T_q \right] } { K \cdot 2 \cdot \log \sqrt[ c ] { 2 / \varpi } } \cdot
    \frac{\gamma_0 \left[ T_q \right]}{\left\Vert \rho_j \right\Vert_{\Psi_1}^2}
\label{eq:p72}
\end{align}
\noindent
Since $\gamma_{0} \left[ T_q \right] / \left\Vert \rho_j \right\Vert_{\Psi_1}^2 \leqslant 8 \cdot \left[ \frac { \theta - 1 } { \theta } + 1 \right]$,\footnote{See Lemma~\ref{lemma:a1}, which is presented after this proof, for detail.}
\begin{align}
  \frac{ \gamma_0 \left[ T_q \right] }{ \varsigma^2 K} \cdot \left( 1 + 2 V_K \left[ T_q \right] \right)
  \leqslant & ~
  \frac{ 1 + 2 V_K \left[ T_q \right] } { K \cdot 2 \cdot \log \sqrt[ c ] { 2 / \varpi } }
  \cdot 8 \cdot \left[ \frac { \theta - 1 } { \theta } + 1 \right] \\
  \leqslant & ~
  \frac{ 4 \cdot \left[ \frac { \theta - 1 } { \theta } + 1 \right] }
  { K \cdot \log \sqrt[ c ] { 2 / \varpi } }
  \cdot \left( 1 + 2 V_K \left[ T_q \right] \right) .
\end{align}
Likewise, if $\varpi \in \left(0 , 2 \exp \left\{ -2c \right\} \right)$,
\begin{align}
  \frac{ \gamma_0 \left[ T_q \right] }{ \varsigma^2 K} \cdot \left( 1 + 2 V_K \left[ T_q \right] \right)
  \leqslant & ~
  \frac{ 8 \cdot \left[ \frac { \theta - 1 } { \theta } + 1 \right] }
  { K \cdot \left( \log \sqrt[ c ] { 2 / \varpi } \right)^2 }
  \cdot \left( 1 + 2 V_K \left[ T_q \right] \right) .
\end{align}
\noindent
As a result, if we set
\begin{equation}
    \kappa =
      \left\{
        \begin{array}{ll}
          \frac{ 4 \cdot \left[ \frac { \theta - 1 } { \theta } + 1 \right] }
          { K \cdot \log \sqrt[ c ] { 2 / \varpi } }
          \cdot \left( 1 + 2 V_K \left[ T_q \right] \right) ,
          & \mbox{ if } \varpi \in \left[ 2 \exp \left\{ -2c \right\}, 1 \right]. \\
          \frac{ 8 \cdot \left[ \frac { \theta - 1 } { \theta } + 1 \right] }
          { K \cdot \left( \log \sqrt[ c ] { 2 / \varpi } \right)^2 }
          \cdot \left( 1 + 2 V_K \left[ T_q \right] \right) ,
          & \mbox{ if } \varpi \in \left(0 , 2 \exp \left\{ -2c \right\} \right).
        \end{array}\right.
\end{equation}
\begin{align}
  \mathrm{Pr} \left\{ \frac{ 1 }{ K } \sum^{ K }_{ q = 1 } \mathcal{R}_{n_s} \left( b, Y_s^q, X_s^q \right)
     \leqslant \frac{ 1 } { K } \sum^{ K }_{ q = 1 }
    \mathcal{R}_{n_t} \left( b, Y_t^q, X_t^q \right) + 2 \cdot \mathrm{RC} \left( \Lambda_l, n, K \right)
    + \varsigma \right\} \notag \\
    \geqslant \left( 1 - \kappa \right)^+,
\end{align}
which completes the proof for 2.
\end{proof}

\begin{lemma}
  Denote $\gamma_0 \left[ \cdot \right]$ as the variance operator for some random variable and $\left\Vert \cdot \right\Vert_{\Psi_1}$ as the Orlicz-$\Psi_1$ norm for some random variable. Under \textbf{A1} to \textbf{A5}, the following statement holds:
  \begin{equation}
  \frac{ \gamma_0 \left[ T_q \right] } { \left\Vert \rho_j \right\Vert _{\Psi_1}^2 }
  \leqslant
  8 \cdot \left[ \frac { \theta - 1 } { \theta } + 1 \right]
  \end{equation}
\label{lemma:a1}
\end{lemma}

%
%
\begin{proof}
\textbf{Lemma}~\ref{lemma:a1}
Based on eq.~(\ref{key_result_thm2_inline}),
\begin{align}
  \frac{ \gamma_0 \left[ T_q \right] } { \left\Vert \rho_j \right\Vert _{\Psi_1}^2 }
    \leqslant & ~ \frac{1} { \left\Vert \rho_j \right\Vert_{\Psi_1}^2 }
              \cdot \left[ \left( \theta - 1 \right) \cdot
              \left( \frac{1}{n/K} + \frac{K}{n(K-1)} \right) + \frac{2 \cdot \theta}{ n/K }
              \right] \label{eq:proof_a1_2} \\
              & ~ \cdot \sup_{ b \in \Lambda_l}
              \left\{ \mathrm{var} \left[ Q \left( b, y^{i,j}, \mathbf{x}^{i,j} \right) -
              \mathcal{R} \left( b, Y, X \right)\right] \right\} \notag
\end{align}
Using the result that the Lebesgue-$p$ norm is less than or equal to the Orlicz-$\Psi_1$ norm multiplied by $p!$, \citep{vanweak}
\begin{align}
  \frac{ \gamma_0 \left[ T_q \right] } { \left\Vert \rho_j \right\Vert _{\Psi_1}^2 }
  \leqslant & ~ \frac{4} { \left\Vert \rho_j \right\Vert_{2}^2 }
              \cdot \left[ \left( \theta - 1 \right) \cdot
              \left( \frac{1}{n/K} + \frac{K}{n(K-1)} \right) + \frac{2 \cdot \theta}{ n/K }
              \right] \\
            & ~ \cdot \sup_{ b \in \Lambda_l}
              \left\{ \mathrm{var} \left[ Q \left( b, y^{i,j}, \mathbf{x}^{i,j} \right) -
              \mathcal{R} \left( b, Y, X \right)\right] \right\} \notag
  \label{lemma_1_inline_1}
\end{align}
Since $\forall x \in \mathbb{R}$, $\left[ \sup \left\vert x \right\vert \right]^2 = \left[ \sup \left\vert x \right\vert^2 \right] = \left[ \sup \left( x \right)^2 \right]$, the definition of $\rho_j$ implies that
\begin{align}
  \left\Vert \rho_j \right\Vert_{2}^2
  = & ~
  \mathbb{E} \left[ \left( ~ \sup_{ b \in \Lambda_l } \left\vert \frac{ 1 } { n/K } \sum_{ i = 1 }^{ n/K } Q \left( b, y^{ j, i }, \mathrm{ x }^{ j, i } \right) - \mathcal{ R } \left( b, Y, X \right) \right\vert
  \right)^2 ~ \right] \\
  = & ~
  \mathbb{E} \left[ \sup_{ b \in \Lambda_l } \left( ~ \frac{ 1 } { n/K } \sum_{ i = 1 }^{ n/K } Q \left( b, y^{ j, i }, \mathrm{ x }^{ j, i } \right) - \mathcal{ R } \left( b, Y, X \right) \right)^2 ~ \right]
  \label{lemma_1_inline_2}
\end{align}
Based on \textbf{A5}, Eq.~(\ref{lemma_1_inline_2}) implies that
\begin{align}
  \left\Vert \rho_j \right\Vert_{2}^2
  = & ~ \theta \cdot
  \sup_{ b \in \Lambda_l } \left\{ ~ \mathbb{E} \left[ \frac{ 1 } { n/K } \sum_{ i = 1 }^{ n/K } Q \left( b, y^{ j, i }, \mathrm{ x }^{ j, i } \right) - \mathcal{ R } \left( b, Y, X \right) \right]^2 ~\right\}  \\
  = & ~
  \frac{ \theta } { n/K } \cdot \sup_{ b \in \Lambda_l } \left\{ ~ \mathrm{var} \left[ Q \left( b, y^{ j, i }, \mathrm{ x }^{ j, i } \right) - \mathcal{ R } \left( b, Y, X \right) \right] ~\right\}
\end{align}
As a result, alongside with the fact that $K \geqslant 2$, we can simplify eq.~(\ref{lemma_1_inline_1}) as follows,
\begin{align}
  \frac{ \gamma_0 \left[ T_q \right] } { \left\Vert \rho_j \right\Vert _{\Psi_1}^2 }
  \leqslant & ~ \frac{ 4 \cdot \left[ \left( \theta - 1 \right) \cdot
              \left( \frac{ 1 }{ n / K } + \frac{ K }{ n \left( K - 1 \right) } \right)
              + \frac{2 \cdot \theta}{ n / K } \right] }
              { \theta / \left( n / K \right) \cdot \sup_{ b \in \Lambda_l}
              \left\{ \mathrm{var}
              \left[ Q \left( b, y^{i,j}, \mathbf{x}^{i,j} \right) -
              \mathcal{R} \left( b, Y, X \right)\right] \right\}  } \\
            & ~ \cdot \sup_{ b \in \Lambda_l}
              \left\{ \mathrm{var}
              \left[ Q \left( b, y^{i,j}, \mathbf{x}^{i,j} \right) -
              \mathcal{R} \left( b, Y, X \right)\right] \right\} \notag \\
          = & ~ 4 \cdot \left[ \frac { \theta - 1 } { \theta } \cdot
              \left( 1 + \frac{1}{K-1} \right) + 2 \right] \\
  \leqslant & ~ 8 \cdot \left[ \frac { \theta - 1 } { \theta } + 1 \right]
\end{align}
\end{proof}
\bigskip

%
\begin{proof}
\textbf{Theorem}~\ref{thm:one_round_RC_bound}

To prove Theorem~\ref{thm:one_round_RC_bound}, we need to quantify the following probability:
\[
  \mathrm{Pr} \left(
    \sup_{ b \in \Lambda_l } \left\vert
    \mathcal{R}_{n_t} \left( b, X_t, Y_t \right) - \mathcal{R}_{n_s} \left( b, X_s, Y_s \right)
    \right\vert
    \geqslant \epsilon
    \right).
\]
Based on the independent blocks, we denote the empirical errors of $b$ on $S^T_0$ and $S^T_1$ as $\mathcal{R}_{ n_t / 2} \left( b, S^T_0 \right)$ and $\mathcal{R}_{ n_t / 2} \left( b, S^T_1 \right)$ respectively. The empirical errors of $b$ on $S^S_0$ and $S^S_1$ are also respectively denoted as $\mathcal{R}_{ n_s / 2} \left( b, S^S_0 \right)$ and $\mathcal{R}_{ n_s / 2} \left( b, S^S_1 \right)$. As a result,
\begin{align}
  \phantom{=} & ~ \mathrm{Pr} \left( \sup_{ b \in \Lambda_l } \left\vert
    \mathcal{R}_{n_t} \left( b, X_t, Y_t \right) - \mathcal{R}_{n_s} \left( b, X_s, Y_s \right)
    \right\vert \geqslant \epsilon \right) \notag \\
  = & ~ \mathrm{Pr} \left( \sup_{ b \in \Lambda_l } \left\vert
    \frac{1}{2} \left[ \mathcal{R}_{n_t/2} \left( b, S^T_0 \right) + \mathcal{R}_{n_t/2} \left( b, S^T_1 \right) \right]
    - \frac{1}{2} \left[ \mathcal{R}_{n_s/2} \left( b, S^S_0 \right) + \mathcal{R}_{n_s/2} \left( b, S^S_1 \right) \right]
    \right\vert \geqslant \epsilon \right) \\
  = & ~ \mathrm{Pr} \left( \sup_{ b \in \Lambda_l } \left\vert
    \left[ \mathcal{R}_{n_t/2} \left( b, S^T_0 \right) - \mathcal{R}_{n_s/2} \left( b, S^S_0 \right) \right]
    + \left[ \mathcal{R}_{n_t/2} \left( b, S^T_1 \right) - \mathcal{R}_{n_s/2} \left( b, S^S_1 \right) \right]
    \right\vert
    \geqslant 2\epsilon \right).
    \label{eq:ind_block_inline_1}
\end{align}
Due to the convexity of the norm and union bound, eq.~(\ref{eq:ind_block_inline_1}) implies that
\begin{align}
  \phantom{=} & ~ \mathrm{Pr} \left( \sup_{ b \in \Lambda_l } \left\vert
    \mathcal{R}_{n_t} \left( b, X_t, Y_t \right) - \mathcal{R}_{n_s} \left( b, X_s, Y_s \right)
    \right\vert \geqslant \epsilon \right) \notag \\
  \leqslant & ~ \mathrm{Pr} \left( \sup_{ b \in \Lambda_l } \left\vert
    \mathcal{R}_{n_t/2} \left( b, S^T_0 \right) - \mathcal{R}_{n_s/2} \left( b, S^S_0 \right)
    \right\vert +
    \sup_{ b \in \Lambda_l } \left\vert
    \mathcal{R}_{n_t/2} \left( b, S^T_1 \right) - \mathcal{R}_{n_s/2} \left( b, S^S_1 \right)
    \right\vert
    \geqslant 2\epsilon \right) \\
  \leqslant & ~ \mathrm{Pr} \left( \sup_{ b \in \Lambda_l }
    \left\vert \mathcal{R}_{n_t/2} \left( b, S^T_0 \right) - \mathcal{R}_{n_s/2} \left( b, S^S_0 \right) \right\vert
    \geqslant \epsilon \right)
    + \mathrm{Pr} \left( \sup_{ b \in \Lambda_l }
    \left\vert \mathcal{R}_{n_t/2} \left( b, S^T_1 \right) - \mathcal{R}_{n_s/2} \left( b, S^S_1 \right) \right\vert
    \geqslant \epsilon \right).
    \label{eqn:iid}
\end{align}
\noindent
Since $\left( S^S_0, S^T_0 \right)$ and $\left( S^S_1, S^T_1 \right)$ are identically distributed due to stationarity,
\begin{equation}
  \mathrm{Pr} \left( \sup_{ b \in \Lambda_l } \left\vert
    \mathcal{R}_{n_t/2} \left( b, S^T_0 \right) - \mathcal{R}_{n_s/2} \left( b, S^S_0 \right)
    \right\vert \geqslant \epsilon \right)
  =
  \mathrm{Pr} \left( \sup_{ b \in \Lambda_l } \left\vert
    \mathcal{R}_{n_t/2} \left( b, S^T_1 \right) - \mathcal{R}_{n_s/2} \left( b, S^S_1 \right)
    \right\vert \geqslant \epsilon \right).
\end{equation}
\noindent
This implies, alongside eq.~(\ref{eqn:iid}), that
\begin{align}
  \mathrm{Pr} \left( \sup_{ b \in \Lambda_l } \left\vert \
    \mathcal{R}_{n_t} \left( b, X_t, Y_t \right) - \mathcal{R}_{n_s} \left( b, X_s, Y_s \right)
    \right\vert \geqslant \epsilon \right)
  \leqslant
  2 \cdot \mathrm{Pr} \left( \sup_{ b \in \Lambda_l } \left\vert
    \mathcal{R}_{n_t/2} \left( b, S^T_0 \right) - \mathcal{R}_{n_s/2} \left( b, S^S_0 \right)
    \right\vert \geqslant \epsilon \right)
\end{align}
If we define
\begin{align}
  \Phi \left( S^S_0 \right) & = \sup_{ b \in \Lambda_l} \left\vert \mathcal{R} \left(b, Y, X \right) - \mathcal{R}_{n_s} \left(b,\;S^S_0 \right) \right\vert,  \notag \\
  \Phi \left( S^T_0 \right) & = \sup_{ b \in \Lambda_l} \left\vert \mathcal{R} \left(b, Y, X \right) - \mathcal{R}_{n_t} \left( b,\; S^T_0 \right)\right\vert, \notag
\end{align}
we obtain the following result,
\begin{align}
  \mathrm{Pr} \left( \sup_{ b \in \Lambda_l } \left\vert
    \mathcal{R}_{n_t/2} \left( b, \; S^T_0 \right) - \mathcal{R}_{n_s/2} \left( b, \; S^S_0 \right) \right\vert \geqslant \epsilon \right)
  \leqslant
  \mathrm{Pr} \left( \Phi \left( S^T_0 \right) + \Phi \left( S^S_0 \right) \geqslant \epsilon \right),
\end{align}

To compute the probability of $\mathrm{Pr} \left( \Phi \left( S^T_0 \right) + \Phi \left( S^S_0 \right) \geqslant \epsilon \right) $, we define $\epsilon_1 = \epsilon / 2 - \mathbb{E} \left[ \Phi \left( \widetilde{S}^T_0 \right) \right]$ and $\epsilon_2 = \epsilon / 2 - \mathbb{E} \left[ \Phi \left( \widetilde{S}^S_0 \right) \right]$. Hence, $\forall \epsilon / 2 > \max \left( \mathbb{E} \left[ \Phi \left( \widetilde{S}^T_0 \right) \right], \mathbb{E} \left[ \Phi \left( \widetilde{S}^S_0 \right) \right] \right)$,
\begin{align}
  \mathrm{Pr} \left( \Phi \left( S^T_0 \right) + \Phi \left( S^S_0 \right) \geqslant \epsilon \right)
  & \leqslant \mathrm{Pr} \left( \Phi \left( S^T_0 \right) \geqslant \epsilon / 2 \right) + \mathrm{Pr} \left( \Phi \left( S^S_0 \right) \geqslant \epsilon / 2 \right)  \\
  & = \mathrm{Pr} \left( \Phi \left( S^T_0 \right) - \mathbb{E} \left[ \Phi \left( \widetilde{S}^T_0 \right) \right] \geqslant \epsilon_1 \right) \\
  & \phantom{=} \, + \mathrm{Pr} \left( \Phi \left( S^S_0 \right) - \mathbb{E} \left[ \Phi \left( \widetilde{S}^S_0 \right) \right] \geqslant \epsilon_2 \right). \notag
\end{align}
\noindent
Since the probability may be considered to be the expectation of some binary indicator function, we can apply Theorem~\ref{thm:Yu94}. Thus,
\begin{align}
  \mathrm{Pr} \left( \Phi \left( S^T_0 \right) - \mathbb{E} \left[ \Phi \left( \widetilde{S}^T_0 \right) \right] \geqslant \epsilon_1 \right) & \leqslant \mathrm{Pr} \left( \Phi \left( \widetilde{S}^T_0 \right) - \mathbb{E} \left[ \Phi \left( \widetilde{S}^T_0 \right) \right] \geqslant \epsilon_1 \right)
  \label{eqn:IB_app_t} \\
  & \phantom{=} + \left( \mu - 1 \right) \beta_{a_t}, \notag \\
  \smallskip
  \mathrm{Pr} \left( \Phi \left( S^S_0 \right) - \mathbb{E} \left[ \Phi \left( \widetilde{S}^S_0 \right) \right] \geqslant \epsilon_2 \right) & \leqslant \mathrm{Pr} \left( \Phi \left( \widetilde{S}^S_0 \right) - \mathbb{E} \left[ \Phi \left( \widetilde{S}^S_0 \right) \right] \geqslant \epsilon_2 \right)
  \label{eqn:IB_app_s} \\
  & \phantom{=} + \left( \mu - 1 \right) \beta_{a_s}. \notag
\end{align}

By applying the McDiarmid inequality to the RHS of eqs.~(\ref{eqn:IB_app_t}) and~(\ref{eqn:IB_app_s}), we get the exponential inequalities for the LHS of eqs.~(\ref{eqn:IB_app_t}) and~(\ref{eqn:IB_app_s}), which yields
\begin{align}
  \mathrm{Pr} \left( \Phi \left( S^T_0 \right) - \mathbb{E} \left[ \Phi \left( \widetilde{S}^T_0 \right) \right] \geqslant \epsilon_1 \right) & \leqslant \exp \left( - \frac{ 2 \mu \left( \epsilon_1 \right)^2 }{ M^2 } \right) + \left( \mu - 1 \right) \beta_{a_t}, \\
  \mathrm{Pr} \left( \Phi \left( S^S_0 \right) - \mathbb{E} \left[ \Phi \left( \widetilde{S}^S_0 \right) \right] \geqslant \epsilon_2 \right) & \leqslant \exp \left( - \frac{ 2 \mu \left( \epsilon_2 \right)^2 }{ M^2 } \right) + \left( \mu - 1 \right) \beta_{a_s}.
\end{align}

\noindent
By denoting $\widetilde{\epsilon} = \min \left( \epsilon_1, \epsilon_2 \right) =  \epsilon / 2 - \max \left( \mathbb{E} \left[ \Phi \left( \widetilde{S}^S_0 \right) \right], \mathbb{E} \left[ \Phi \left( \widetilde{S}^T_0 \right) \right] \right)$,
\begin{align}
  \mathrm{Pr} \left( \sup_{ b \in \Lambda_l } \left\vert \mathcal{R}_{n_t} \left( b, X_t, Y_t \right) - \mathcal{R}_{n_s} \left( b, X_s, Y_s \right) \right\vert \geqslant \epsilon \right)
  \leqslant 4 \exp \left( - \frac{ 2 \mu \left(\widetilde{\epsilon} \right)^2 }{ M^2 } \right) + 2 \left( \mu - 1 \right) \left[ \beta_{a_t} + \beta_{a_s} \right],
\end{align}
\noindent
which yields
\begin{align}
  \mathrm{Pr} \left( \sup_{ b \in \Lambda_l } \left\vert \mathcal{R}_{n_t} \left( b, X_t, Y_t \right) - \mathcal{R}_{n_s} \left( b, X_s, Y_s \right) \right\vert \leqslant \epsilon \right)
  \geqslant 1 - \left[ 4 \exp \left( - \frac{ 2 \mu \left( \widetilde{\epsilon} \right)^2}{M^2} \right)
    + 2 \left( \mu - 1 \right) \left[ \beta_{a_t} + \beta_{a_s} \right] \right].
\label{eqn:quarter-result_one_round}
\end{align}

\noindent
If we set $ \varpi= 4 \exp\left( - \frac{ 2 \mu \left( \widetilde{\epsilon} \right)^2 }{ M^2 } \right) + 2 \left( \mu - 1 \right) \left[ \beta_{a_t} + \beta_{a_s} \right]$, \, $\varpi' = \varpi - \left( \mu - 1 \right) \left[ \beta_{a_t} + \beta_{a_s} \right]$ and assume $\varpi' > 0$,
\begin{align}
  \epsilon & = M \cdot \sqrt{ \frac{ \log \left( 4 / \varpi' \right) } { 2\mu } }  + 2 \cdot \max \left( \mathbb{E} \left[ \Phi \left( \widetilde{S}^T_0 \right) \right], \mathbb{E} \left[ \Phi \left( \widetilde{S}^S_0 \right) \right] \right) \\
  & \leqslant M \cdot \sqrt{ \frac{ \log \left( 4 / \varpi' \right) } { 2\mu } }  + 2 \cdot \max \left( \mathrm{RC}_{S^T_0} \left( \Lambda_l\right) , \mathrm{RC}_{S^S_0} \left( \Lambda_l\right) \right) \\
  & = M \cdot \sqrt{ \frac{ \log \left( 4 / \varpi' \right) } { 2\mu } }  + 2 \cdot \mathrm{RC}_{S^S_0} \left( \Lambda_l \right).
\end{align}

\noindent
As a result, eq.~(\ref{eqn:quarter-result_one_round}) may be respecified as, $\forall b \in \Lambda_l$,
\begin{equation}
  \mathrm{Pr} \left( \mathcal{R}_{n_s} \left( b, X_s, Y_s \right) \leqslant \mathcal{R}_{n_t} \left( b, X_t, Y_t \right) +  M \cdot \sqrt{ \frac{ \log \left( 4 / \varpi' \right) } { 2\mu } }  + 2 \cdot \mathrm{RC}_{\widetilde{S}^S_0} \left( \Lambda_l \right) \right)
  \geqslant 1 - \varpi
\label{eqn:semi-result_one_round}
\end{equation}
\end{proof}
\bigskip

\begin{proof}
\textbf{Theorem}~\ref{thm:RC_bound_CV}

Since we have already defined
\begin{align}
  T_q  = &\sup_{b \in \Lambda_l} \; \left\vert \mathcal{R}_{n_s} \left( b, Y_{s}^q, X_{s}^q \right) - \mathcal{R}_{n_t} \left( b, Y_{t}^q, X_{t}^q \right) \right\vert  \\
  & - \mathbb{E} \left[ \sup_{b \in \Lambda_l} \; \left\vert \mathcal{R}_{n_s} \left( b, Y_{s}^q, X_{s}^q \right) - \mathcal{R}_{n_t} \left( b, Y_{t}^q, X_{t}^q \right) \right\vert \right], \notag
\end{align}
Lemma~\ref{lem:Cheby_ineq} implies that
\begin{align}
  \mathrm{Pr} \left\{ \frac{1}{K} \sum^{K}_{q=1} T_q \geqslant \varsigma \right\}
  \leqslant
  \frac{ \gamma_0 \left[ T_q \right] }{ \varsigma^2 K} \cdot \left( 1 + 2 V_K \left[ T_q \right] \right)
\end{align}
Based on Theorem~\ref{thm:one_round_RC_bound}, we let $\varsigma = M \cdot \sqrt{ \frac{ \log \left( 4 / \varpi' \right) } { 2\mu } }$, which implies that
\begin{align}
  \phantom{=} & ~
    \frac{ \gamma_0 \left[ T_q \right] }{ \varsigma^2 K} \cdot \left( 1 + 2 V_K \left[ T_q \right] \right)
  \notag \\
  = & ~
  \frac{ \gamma_0 \left[ T_q \right] } { M^2 \cdot K \cdot \log \left( 4 / \varpi' \right) / \left( 2\mu \right) } \cdot \left( 1 + 2 V_K \left[ T_q \right] \right) \\
  \leqslant & ~
  \frac{ 2 \left( 1 + 2 V_K \left[ T_q \right] \right) } { \log \left( 4 / \varpi' \right) \cdot K / \mu }.
\end{align}
As a result,
\begin{align}
  \mathrm{Pr} \left\{ \frac{1}{K} \sum_{q = 1}^{K} \mathcal{R}_{n_s} \left( b, X_s^q, Y_s^q \right)
  \leqslant
  \frac{1}{K} \sum_{q = 1}^{K} \mathcal{R}_{n_t} \left( b, X_t^q, Y_t^q \right)
  + 2 \cdot \mathrm{RC}_{S^S_0} \left( \Lambda_l \right)
  + M \cdot \sqrt{ \frac{ \log \left( 4 / \varpi' \right) } { 2\mu } } \right\} \\
  \geqslant 1 - \frac{ 2 \left( 1 + 2 V_K \left[ T_q \right] \right) } { \log \left( 4 / \varpi' \right) \cdot K / \mu } \notag
\end{align}
To ensure $1 - 2\left( 1 + 2 V_K \left[ T_q \right] \right) / \left( \log \left( 4 / \varpi' \right) \cdot K / \mu \right)$ is between $0$ and $1$, we need
\begin{align}
  \frac{ 2 \left( 1 + 2 V_K \left[ T_q \right] \right) }
    { \log \left( 4 / \varpi' \right) \cdot K / \mu }
  \leqslant 1,
\end{align}
which implies that
\begin{align}
  \varpi' \in
   \left( 0, 4 \exp \left\{ - \frac{ 2 \left( 1 + 2 V_K \left[ T_q \right] \right) } { K / \mu } \right\} \right].
\end{align}
\end{proof}


\begin{landscape}
\begin{figure}
  \begin{flushleft}
  \section{Repeated lasso simulations}
  \end{flushleft}


    \subfloat[\label{fig:sim41}  $\mbox{random seed} = 5 $]
    {\includegraphics[width=0.18\paperwidth]{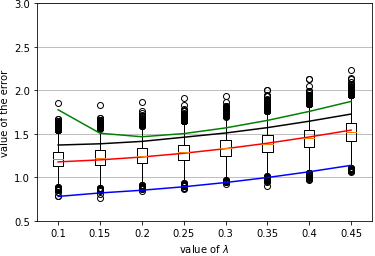}}
    \subfloat[\label{fig:sim42}  $\mbox{random seed} = 10$]
    {\includegraphics[width=0.18\paperwidth]{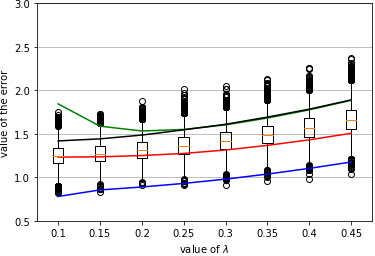}}
    \subfloat[\label{fig:sim43}  $\mbox{random seed} = 15$]
    {\includegraphics[width=0.18\paperwidth]{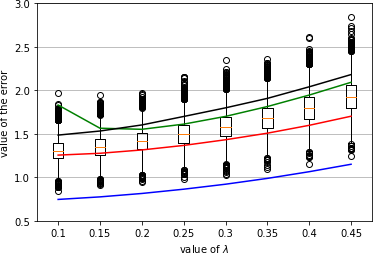}}
    \subfloat[\label{fig:sim44}  $\mbox{random seed} = 20$]
    {\includegraphics[width=0.18\paperwidth]{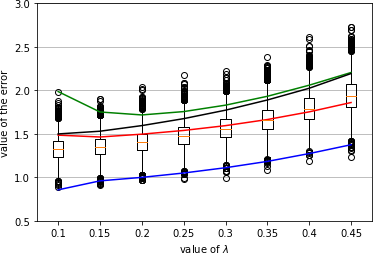}}
    \subfloat[\label{fig:sim45}  $\mbox{random seed} = 25$]
    {\includegraphics[width=0.18\paperwidth]{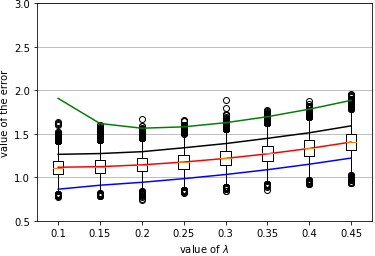}}

    \subfloat[\label{fig:sim46}  $\mbox{random seed} = 30$]
    {\includegraphics[width=0.18\paperwidth]{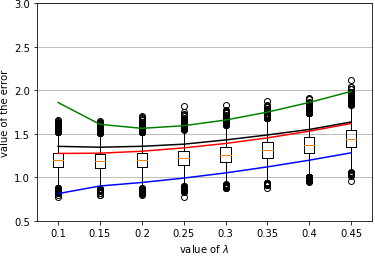}}
    \subfloat[\label{fig:sim47}  $\mbox{random seed} = 35$]
    {\includegraphics[width=0.18\paperwidth]{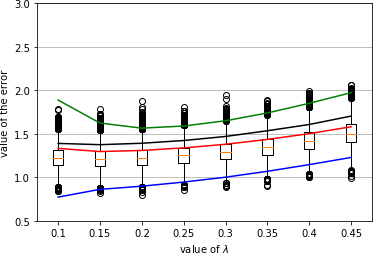}}
    \subfloat[\label{fig:sim48}  $\mbox{random seed} = 40$]
    {\includegraphics[width=0.18\paperwidth]{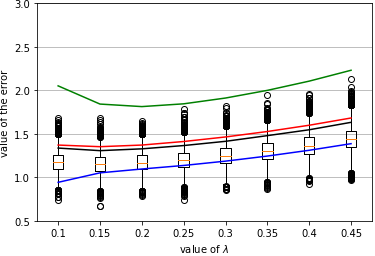}}
    \subfloat[\label{fig:sim49}  $\mbox{random seed} = 45$]
    {\includegraphics[width=0.18\paperwidth]{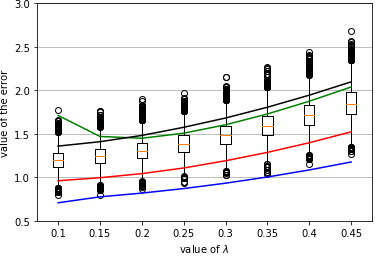}}
    \subfloat[\label{fig:sim410} $\mbox{random seed} = 50$]
    {\includegraphics[width=0.18\paperwidth]{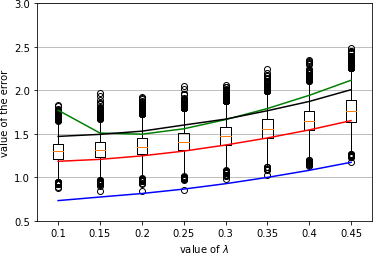}}

    \subfloat[\label{fig:sim411} $\mbox{random seed} = 55$]
    {\includegraphics[width=0.18\paperwidth]{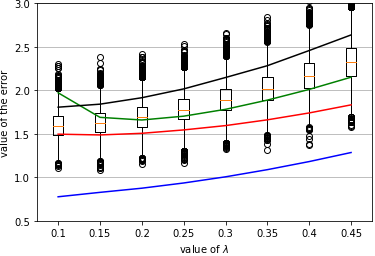}}
    \subfloat[\label{fig:sim412} $\mbox{random seed} = 60$]
    {\includegraphics[width=0.18\paperwidth]{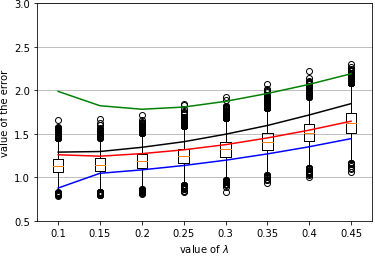}}
    \subfloat[\label{fig:sim413} $\mbox{random seed} = 65$]
    {\includegraphics[width=0.18\paperwidth]{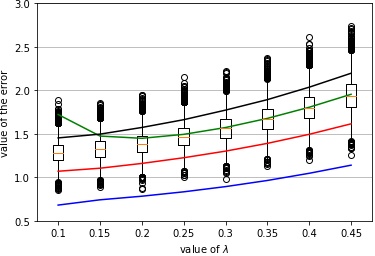}}
    \subfloat[\label{fig:sim414} $\mbox{random seed} = 70$]
    {\includegraphics[width=0.18\paperwidth]{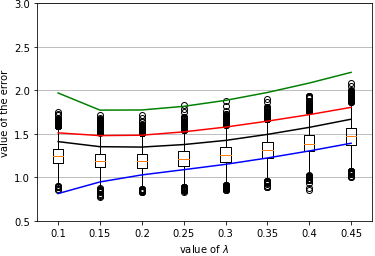}}
    \subfloat[\label{fig:sim415} $\mbox{random seed} = 75$]
    {\includegraphics[width=0.18\paperwidth]{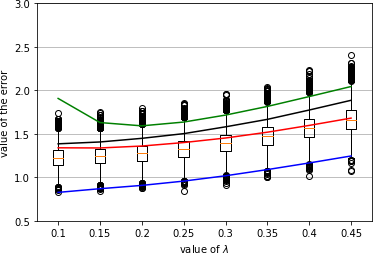}}

    \subfloat[\label{fig:sim416} $\mbox{random seed} = 80$]
    {\includegraphics[width=0.18\paperwidth]{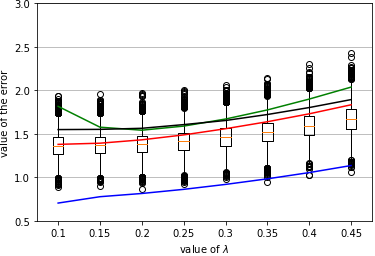}}
    \subfloat[\label{fig:sim417} $\mbox{random seed} = 85$]
    {\includegraphics[width=0.18\paperwidth]{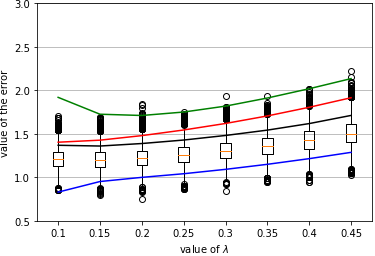}}
    \subfloat[\label{fig:sim418} $\mbox{random seed} = 90$]
    {\includegraphics[width=0.18\paperwidth]{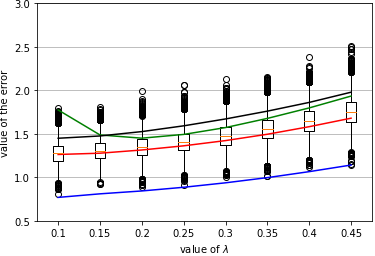}}
    \subfloat[\label{fig:sim419} $\mbox{random seed} = 95$]
    {\includegraphics[width=0.18\paperwidth]{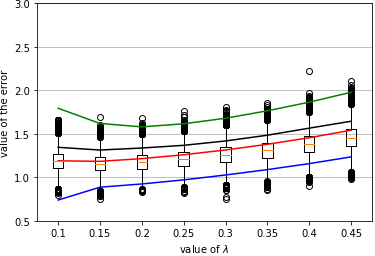}}
    \subfloat[\label{fig:sim420} $\mbox{random seed} =100$]
    {\includegraphics[width=0.18\paperwidth]{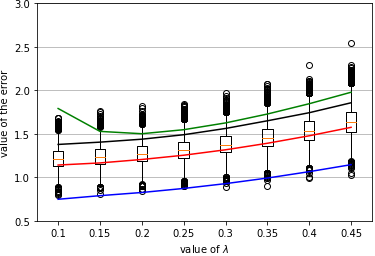}}

  {\includegraphics[width=0.6\paperwidth]{Fig3key.pdf}}

  \caption{Results from repeated lasso simulations with $n/K = 100, \; p = 100$}
  \label{fig:sim4}

\end{figure}
\end{landscape}

\begin{landscape}
\begin{figure}
  \centering

    \subfloat[\label{fig:sim51}  $\mbox{random seed} = 5 $]
    {\includegraphics[width=0.18\paperwidth]{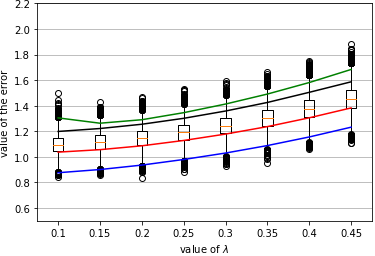}}
    \subfloat[\label{fig:sim52}  $\mbox{random seed} = 10$]
    {\includegraphics[width=0.18\paperwidth]{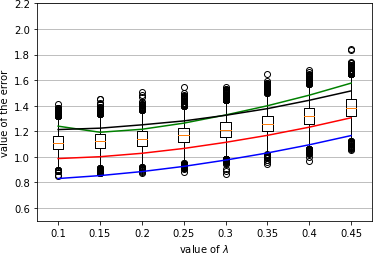}}
    \subfloat[\label{fig:sim53}  $\mbox{random seed} = 15$]
    {\includegraphics[width=0.18\paperwidth]{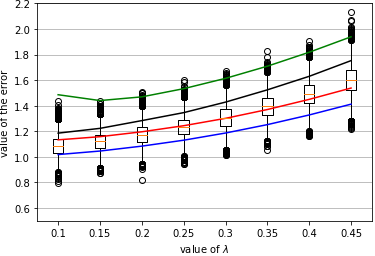}}
    \subfloat[\label{fig:sim54}  $\mbox{random seed} = 20$]
    {\includegraphics[width=0.18\paperwidth]{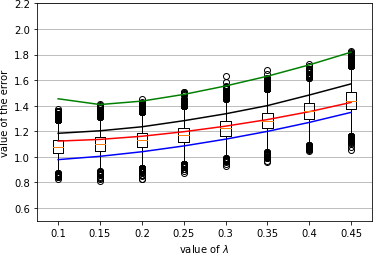}}
    \subfloat[\label{fig:sim55}  $\mbox{random seed} = 25$]
    {\includegraphics[width=0.18\paperwidth]{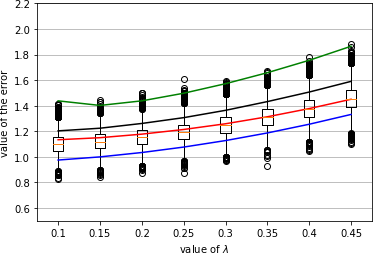}}

    \subfloat[\label{fig:sim56}  $\mbox{random seed} = 30$]
    {\includegraphics[width=0.18\paperwidth]{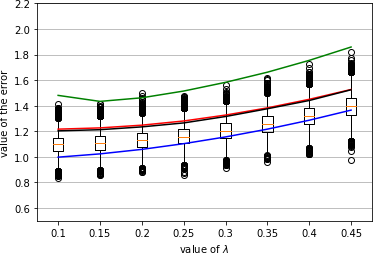}}
    \subfloat[\label{fig:sim57}  $\mbox{random seed} = 35$]
    {\includegraphics[width=0.18\paperwidth]{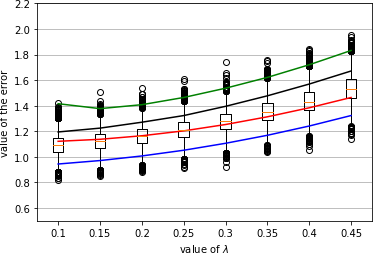}}
    \subfloat[\label{fig:sim58}  $\mbox{random seed} = 40$]
    {\includegraphics[width=0.18\paperwidth]{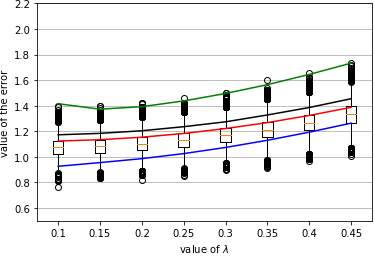}}
    \subfloat[\label{fig:sim59}  $\mbox{random seed} = 45$]
    {\includegraphics[width=0.18\paperwidth]{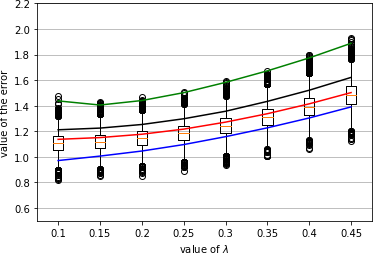}}
    \subfloat[\label{fig:sim510} $\mbox{random seed} = 50$]
    {\includegraphics[width=0.18\paperwidth]{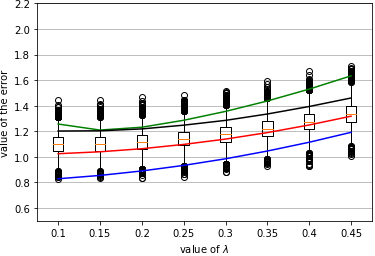}}

    \subfloat[\label{fig:sim511} $\mbox{random seed} = 55$]
    {\includegraphics[width=0.18\paperwidth]{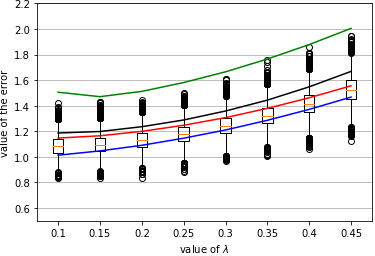}}
    \subfloat[\label{fig:sim512} $\mbox{random seed} = 60$]
    {\includegraphics[width=0.18\paperwidth]{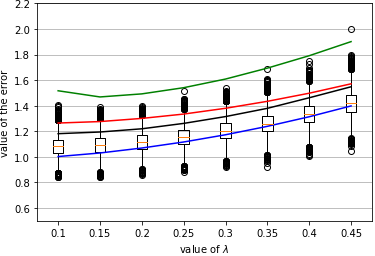}}
    \subfloat[\label{fig:sim513} $\mbox{random seed} = 65$]
    {\includegraphics[width=0.18\paperwidth]{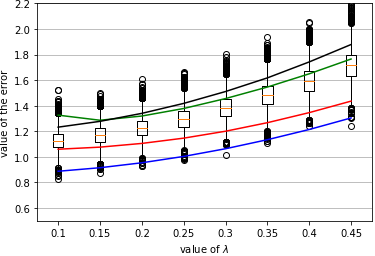}}
    \subfloat[\label{fig:sim514} $\mbox{random seed} = 70$]
    {\includegraphics[width=0.18\paperwidth]{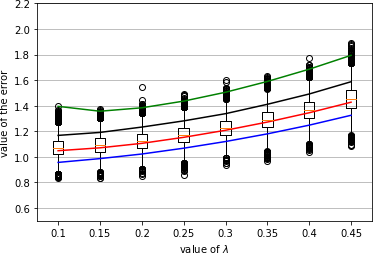}}
    \subfloat[\label{fig:sim515} $\mbox{random seed} = 75$]
    {\includegraphics[width=0.18\paperwidth]{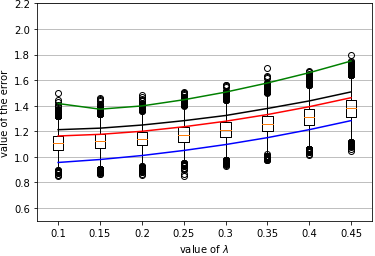}}

    \subfloat[\label{fig:sim516} $\mbox{random seed} = 80$]
    {\includegraphics[width=0.18\paperwidth]{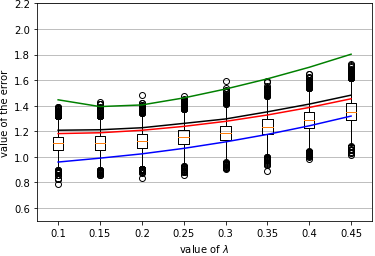}}
    \subfloat[\label{fig:sim517} $\mbox{random seed} = 85$]
    {\includegraphics[width=0.18\paperwidth]{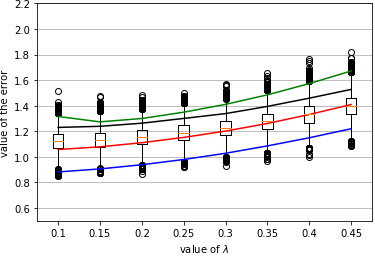}}
    \subfloat[\label{fig:sim518} $\mbox{random seed} = 90$]
    {\includegraphics[width=0.18\paperwidth]{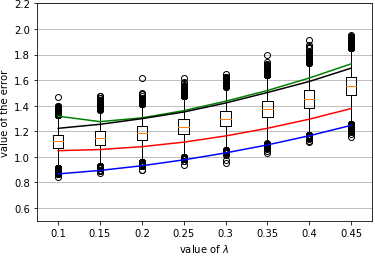}}
    \subfloat[\label{fig:sim519} $\mbox{random seed} = 95$]
    {\includegraphics[width=0.18\paperwidth]{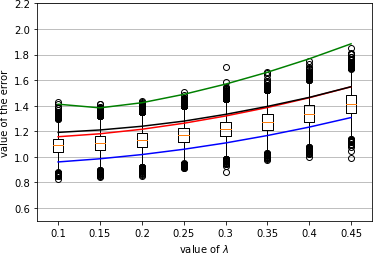}}
    \subfloat[\label{fig:sim520} $\mbox{random seed} =100$]
    {\includegraphics[width=0.18\paperwidth]{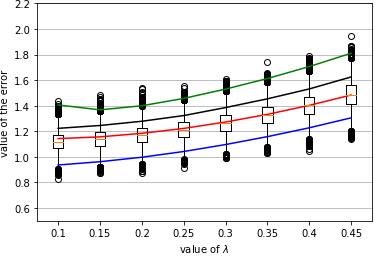}}

  {\includegraphics[width=0.6\paperwidth]{Fig3key.pdf}}

  \caption{Results from repeated lasso simulations with $n/K = 200, \; p = 100$}
  \label{fig:sim5}

\end{figure}
\end{landscape}

\begin{landscape}
\begin{figure}
  \centering

  \subfloat[\label{fig:sim61} $\mbox{random seed} = 5 $]
    {\includegraphics[width=0.18\paperwidth]{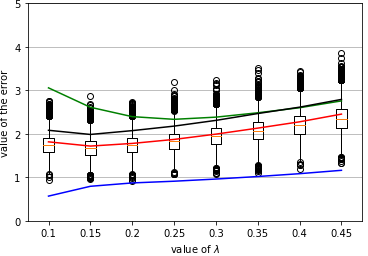}}
    \subfloat[\label{fig:sim62} $\mbox{random seed} = 10$]
    {\includegraphics[width=0.18\paperwidth]{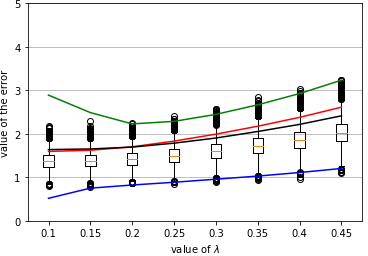}}
    \subfloat[\label{fig:sim63} $\mbox{random seed} = 15$]
    {\includegraphics[width=0.18\paperwidth]{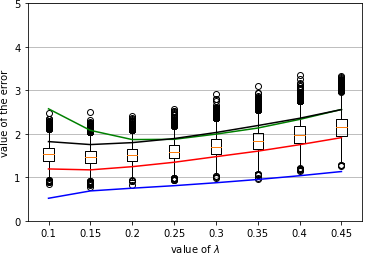}}
    \subfloat[\label{fig:sim64} $\mbox{random seed} = 20$]
    {\includegraphics[width=0.18\paperwidth]{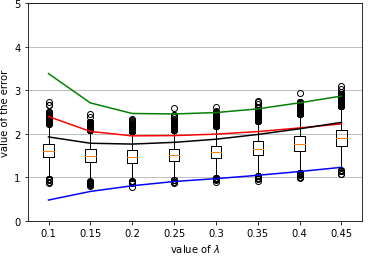}}
    \subfloat[\label{fig:sim65} $\mbox{random seed} = 25$]
    {\includegraphics[width=0.18\paperwidth]{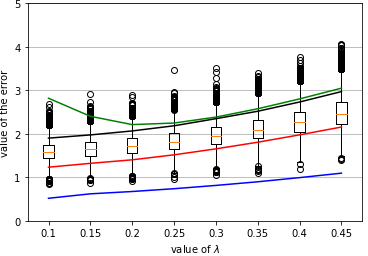}}

    \subfloat[\label{fig:sim66} $\mbox{random seed} = 30$]
    {\includegraphics[width=0.18\paperwidth]{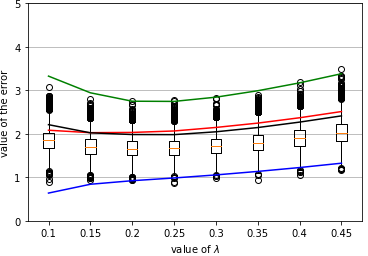}}
    \subfloat[\label{fig:sim67} $\mbox{random seed} = 35$]
    {\includegraphics[width=0.18\paperwidth]{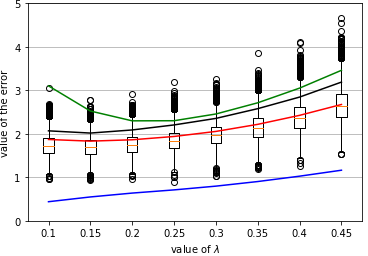}}
    \subfloat[\label{fig:sim68} $\mbox{random seed} = 40$]
    {\includegraphics[width=0.18\paperwidth]{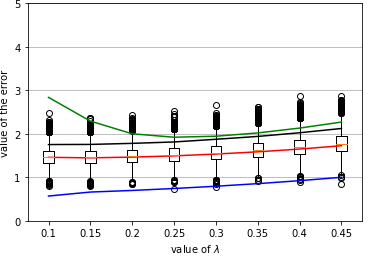}}
    \subfloat[\label{fig:sim69} $\mbox{random seed} = 45$]
    {\includegraphics[width=0.18\paperwidth]{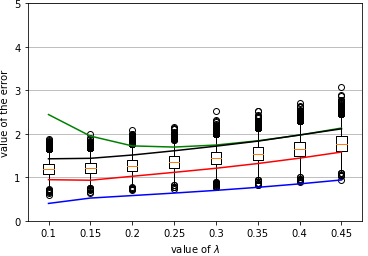}}
    \subfloat[\label{fig:sim610} $\mbox{random seed} = 50$]
    {\includegraphics[width=0.18\paperwidth]{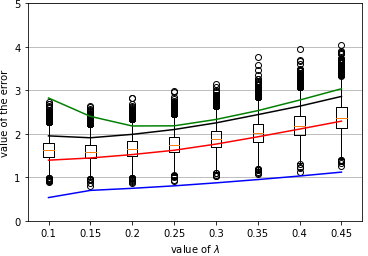}}

    \subfloat[\label{fig:sim611} $\mbox{random seed} = 55$]
    {\includegraphics[width=0.18\paperwidth]{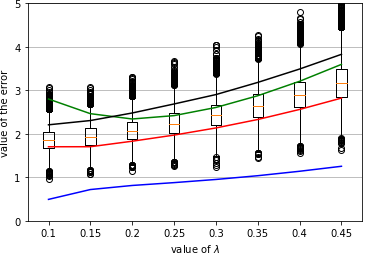}}
    \subfloat[\label{fig:sim612} $\mbox{random seed} = 60$]
    {\includegraphics[width=0.18\paperwidth]{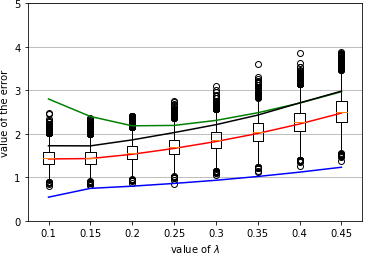}}
    \subfloat[\label{fig:sim613} $\mbox{random seed} = 65$]
    {\includegraphics[width=0.18\paperwidth]{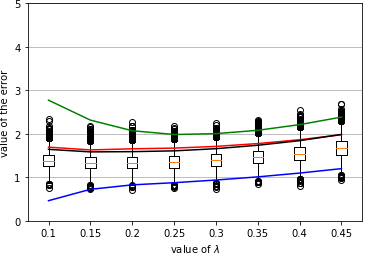}}
    \subfloat[\label{fig:sim614} $\mbox{random seed} = 70$]
    {\includegraphics[width=0.18\paperwidth]{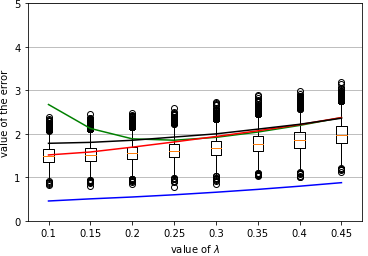}}
    \subfloat[\label{fig:sim615} $\mbox{random seed} = 75$]
    {\includegraphics[width=0.18\paperwidth]{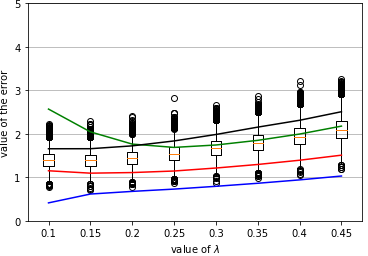}}

    \subfloat[\label{fig:sim616} $\mbox{random seed} = 80 $]
    {\includegraphics[width=0.18\paperwidth]{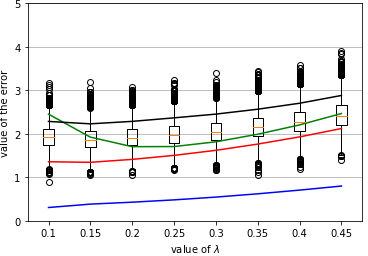}}
    \subfloat[\label{fig:sim617} $\mbox{random seed} = 85 $]
    {\includegraphics[width=0.18\paperwidth]{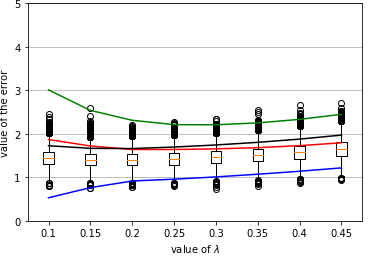}}
    \subfloat[\label{fig:sim618} $\mbox{random seed} = 90 $]
    {\includegraphics[width=0.18\paperwidth]{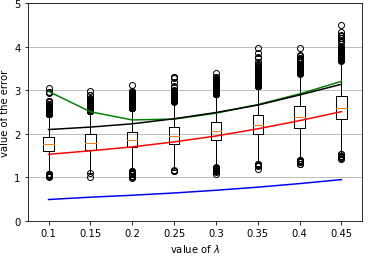}}
    \subfloat[\label{fig:sim619} $\mbox{random seed} = 95 $]
    {\includegraphics[width=0.18\paperwidth]{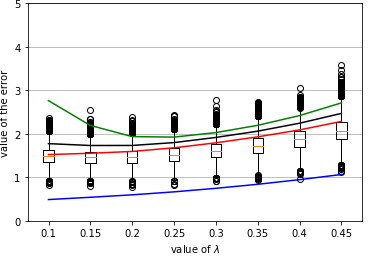}}
    \subfloat[\label{fig:sim620} $\mbox{random seed} = 100$]
    {\includegraphics[width=0.18\paperwidth]{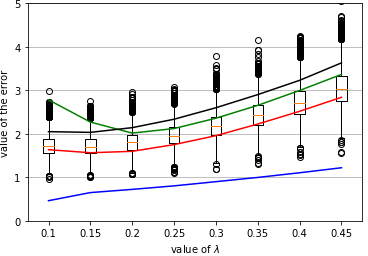}}

  {\includegraphics[width=0.6\paperwidth]{Fig3key.pdf}}

  \caption{Results from repeated lasso simulations with $n/K = 50, \; p = 100$}
  \label{fig:sim6}

\end{figure}
\end{landscape}

\bibliographystyle{elsarticle-harv}
\bibliography{CVrefs}

\end{document}